\documentclass[journal]{IEEEtran}
\usepackage{amsmath,amssymb,amsthm}
\usepackage{mathtools}
\usepackage[dvipsnames]{xcolor}
\usepackage[ruled,vlined,linesnumbered]{algorithm2e}
\usepackage{geometry}
\usepackage{hyperref}
\usepackage{booktabs}
\usepackage{natbib}
\usepackage{bm}
\usepackage{graphicx}
\newenvironment{revised}{\color{black}}{}
\newcommand{\nrev}[1]{{\color{black}#1}}
\newenvironment{nrevised}{\color{black}}{}

\providecommand{\res}[1]{\textcolor{black}{#1}}
\providecommand{\todo}[1]{\textcolor{black}{[#1]}}

\newtheorem{theorem}{Theorem}[section]
\newtheorem{proposition}[theorem]{Proposition}
\newtheorem{lemma}[theorem]{Lemma}
\newtheorem{corollary}[theorem]{Corollary}
\newtheorem{assumption}{Assumption}
\newtheorem{remark}{Remark}[section]

\newtheorem{example}{Example}[section]

\newcommand{\R}{\mathbb{R}}

\newcommand{\Spp}{\mathcal{S}_{++}^{p}}
\newcommand{\Dp}{\mathcal{D}^{p}_{++}}
\newcommand{\Or}{\mathcal{O}_{r}}
\newcommand{\Bpr}{\mathcal{B}^{p,r}}
\newcommand{\Mpr}{\mathcal{M}^{p,r}}
\newcommand{\Rpr}{\mathcal{R}_{p,r}}
\newcommand{\tr}{\operatorname{tr}}
\newcommand{\ddiag}{\operatorname{ddiag}}
\newcommand{\sym}{\operatorname{sym}}
\renewcommand{\skew}{\operatorname{skew}}

\newcommand{\grad}{\operatorname{grad}}
\newcommand{\Retr}{\operatorname{Retr}}
\newcommand{\PH}{\mathcal{P}^{\mathcal{H}}}
\newcommand{\dS}{d_{\mathcal{S}_{++}^p}}
\newcommand{\norm}[1]{\left\|#1\right\|}
\newcommand{\abs}[1]{\left|#1\right|}
\DeclareMathOperator{\diag}{diag}
\DeclareMathOperator{\Log}{Log}
\newcommand{\op}{\mathrm{op}}
\newcommand{\LT}{\mathbf{L}_T}
\newcommand{\dtime}{d_{\mathrm{time}}}
\newcommand{\dr}{d_r}
\newcommand{\PT}{\mathcal{P}}
\newcommand{\Dsp}{\mathcal{D}^p}
\newcommand{\Nb}{\mathcal{N}}
\newcommand{\E}{\mathbb{E}}

\begin{document}

\title{Dynamic Elliptical Graph Factor Models via Riemannian Optimization with Geodesic Temporal Regularization}

\author{Chuansen Peng and Xiaojing Shen
\thanks{The work was supported by the Sichuan Provincial Natural Science Foundation Project (Innovative Research Group) 2026NSFSCZY0056 and in part by the National Natural Science
 Foundation of China (NSFC) under Grant U2133208, 62203313. \textit{(Corresponding author: Xiaojing Shen)}\\
 
 Chuansen Peng and Xiaojing Shen are with School of Mathematics, Sichuan University, Chengdu, Sichuan, 610064, China. (e-mail: \href{mailto:pengchuansen@stu.scu.edu.cn}{pengchuansen@stu.scu.edu.cn}; \href{mailto:shenxj@scu.edu.cn}{shenxj@scu.edu.cn}).}
}

\maketitle

\begin{abstract}
\begin{revised}
Inferring time-varying graphs from high-dimensional nodal observations
is difficult when each observation window contains fewer samples than
nodes, when the observations are heavy-tailed, and when the dependence
structure drifts multiplicatively rather than additively over time.
We propose the dynamic elliptical graph factor model (\textsc{Degfm}),
which combines three ingredients: an elliptical likelihood that
down-weights outlying samples; a low-rank-plus-diagonal (LRaD)
parameterization of each precision matrix, which we show to be the posterior precision of
a field conditioned on a few aggregate linear measurements; and a
temporal regularizer equal to the squared affine-invariant distance
between consecutive precision matrices, which is the energy of a
multiplicative evolution model and is invariant to node rescaling and to
the choice between precision and covariance.
Estimation is posed on a product of quotient manifolds and solved by a safeguarded Riemannian conjugate-gradient method whose
temporal gradients are evaluated in closed form through generalized
eigendecompositions.
We prove that (i) the LRaD parameterization is locally identifiable
modulo orthogonal transformations under a tangent-space transversality
(coherence) condition; (ii) \nrev{the solver, equipped with a rank-preserving Armijo
backtracking, is well defined and converges to critical points, with
$\min_{k<K}\|\grad F(\theta^{(k)})\|^2=\mathcal{O}(1/K)$};
(iii) a localized estimator obeys a finite-sample oracle inequality in
which the stochastic error scales with the intrinsic dimension
$pr+p-r(r-1)/2$ times an \emph{effective temporal dimension} that
decreases from $1$ to $1/T$ as the coupling strength grows, which
quantifies the variance reduction and the bias created by temporal
coupling, and which extends to non-LRaD (misspecified) models;
\nrev{for Student-$t$ observations, the local restricted strong
convexity underlying (iii) is proved rather than assumed, with a
curvature constant whose order $\nu/(p+\nu)$ is shown to be
unimprovable;} and
(iv) the partial-correlation graphs are recovered exactly by
thresholding.
Experiments on synthetic data, S\&P~500 returns, and HCP fMRI
recordings illustrate the approach.
\end{revised}
\end{abstract}


\section{Introduction}
\label{sec:intro}
 
Graphs provide a natural and powerful language for representing the
relational structure of complex systems: nodes encode entities of
interest, while edges capture pairwise interactions or statistical
dependencies among them~\cite{newman2003structure, watts1998collective, barabasi1999emergence, kolaczyk2014statistical}.
This representational flexibility has driven widespread adoption
across disciplines as diverse as social network analysis, brain
connectivity modelling, protein interaction mapping, power-grid
monitoring, and financial market analysis~\cite{shuman2013emerging, ortega2018graph, dong2019learning, mateos2019connecting,hamilton2017inductive}.
The recent confluence of graph signal processing~(GSP)
and geometric deep learning~\cite{defferrard2016convolutional, kipf2017semi, bruna2014spectral, bronstein2021geometric} has further
intensified interest in structured data representations, establishing
graphs as a first-class citizen in modern machine learning
pipelines~\cite{wu2022graph, ma2021deep}.
Yet real-world systems are rarely static: functional brain networks
reconfigure across cognitive states~\cite{monti2014estimating},
financial correlation structures shift dramatically around market
stress events~\cite{mantegna1999hierarchical, onnela2003dynamics},
gene regulatory programmes are rewired during cellular
differentiation, and climate teleconnections evolve with seasonal
and decadal forcing.
Capturing such temporal dynamics requires moving beyond fixed graph
models toward \emph{time-varying} graph representations that encode
evolving relational structure as a function of
time~\cite{kolar2010estimating, zhou2010time, ahmed2009recovering, hallac2017network,xuan2007modeling,kazemi2020representation, skarding2021foundations}.

In most practical settings, however, the graph topology is
\emph{latent}: one observes nodal signals or multivariate time
series, but the connectivity structure that generated them is unknown
and must be inferred from data.
This graph topology inference problem, variously termed graph
learning, network reconstruction, or structural equation
modelling, has emerged as a central challenge at the intersection of
signal processing, high-dimensional statistics, and
machine learning~\cite{kolaczyk2009statistical, lauritzen1996graphical, dong2016learning, mateos2019connecting, egilmez2017graph, pasdeloup2018characterization, navarro2024joint, zhu2021survey}.
The difficulty is compounded in the time-varying setting, where one
must simultaneously recover the graph at each instant and ensure
that the inferred sequence of graphs varies smoothly or according to
some structured prior, all from observations whose count per window
may be far smaller than the number of nodes.
 
The static counterpart of this problem has a rich history.
From a probabilistic graphical model perspective, the conditional
independence structure of a multivariate Gaussian distribution is
encoded in the sparsity pattern of its precision (inverse covariance)
matrix, and inferring the graph reduces to sparse precision matrix
estimation~\cite{dempster1972covariance, lauritzen1996graphical}.
This insight underpins the celebrated graphical
lasso~(\textsc{GLasso})~\cite{friedman2008sparse, banerjee2008model},
which solves an $\ell_1$-penalized log-likelihood problem and has
become a canonical tool for high-dimensional covariance
selection~\cite{meinshausen2006high, yuan2007model, tibshirani1996regression}.
Extensions to latent variable settings, where observed variables are
confounded by unobserved factors, were developed by
\cite{chandrasekaran2012latent}, who showed that the precision matrix
decomposes as a sparse component plus a low-rank perturbation,
recoverable via a convex nuclear-norm-plus-$\ell_1$ program.
Parallel developments in the GSP community recast graph learning as
a structured Laplacian estimation problem, exploiting the assumption
that graph signals are smooth over the inferred
topology~\cite{dong2016learning, kalofolias2016learn,
egilmez2017graph, segarra2017network}.
More recent work has leveraged spectral
templates~\cite{segarra2017network},
Laplacian-constrained optimization~\cite{egilmez2017graph,
kumar2020unified, ying2022does},
and diffusion-process models~\cite{pasdeloup2018characterization}
to impose physically motivated structural constraints on the learned
graph.
The growing intersection of graph learning and deep learning has
produced data-driven topology discovery methods that are
increasingly scalable~\cite{pu2021learning, fatemi2021slaps,
saboksayr2021accelerated},
and recent work has demonstrated that
combining structural constraints with spectral inductive
biases substantially improves sample
efficiency~\cite{cardoso2021nonconvex, saboksayr2021online,
berger2020graph}.
Estimation under non-Gaussian or heavy-tailed observations
has also received renewed attention, motivated by financial and
neuroimaging applications where the Gaussian assumption is
routinely violated~\cite{sun2022robust, ying2020nonconvex,shafipour2021network}.
Despite their diversity, all of these approaches produce a single,
\emph{time-invariant} graph estimate and are therefore unsuitable to
systems whose connectivity evolves over time.
 
The time-varying graph inference problem has attracted growing
attention over the past fifteen years.
Early statistical approaches employed kernel
smoothing~\cite{kolar2010estimating, zhou2010time} or local
likelihood~\cite{talih2005structural} to track slowly drifting
network structure, while \cite{ahmed2009recovering} framed the
problem as a sequence of penalized regressions linked by a
smoothness prior.
\cite{xuan2007modeling} proposed a Bayesian treatment based on
changepoint detection, enabling abrupt structural transitions.
A landmark contribution was the time-varying graphical
lasso~(\textsc{TVGL}) of \cite{hallac2017network}, which augments
the \textsc{GLasso} objective with \nrev{Euclidean penalties on successive differences, e.g., $\sum_t \|\Theta_t - \Theta_{t-1}\|_F^2$ or its $\ell_1$ counterpart,} that encourage successive
precision matrices to remain close in \nrev{an entrywise (Euclidean) sense}.
Piecewise-constant models, in which the graph is assumed to change
only at unknown breakpoints, were studied by
\cite{gibberd2017regularized} via the group-fused graphical lasso.
In neuroscience, \cite{monti2014estimating} adapted sliding-window
and $\ell_1$-penalized estimators to track time-varying
functional connectivity from fMRI data.
More recently, factor-model-based formulations have been explored to
exploit low-dimensional latent structure in large-scale
systems~\cite{lam2012factor, stock2002forecasting, bai2002determining,
tipping1999probabilistic}, and Riemannian optimization on the SPD
manifold has been applied to static covariance
estimation~\cite{absil2008optimization, bonnabel2013stochastic,
sra2015conic, boumal2014manopt, boumal2023introduction}.
The past few years have witnessed a surge of activity in this area.
\cite{natali2022learning} developed an online algorithm for tracking
graph topology from streaming smooth signals, while
\cite{shafipour2020online} extended this line of work to
non-stationary environments.
High-dimensional changepoint detection in graphical models has been
addressed by \cite{safikhani2022joint} and \cite{padilla2022optimal},
who established optimal detection rates under sparsity constraints.
Tensor-based representations that simultaneously exploit
cross-sectional and temporal structure have been proposed
by \cite{lin2023tensor}, and deep learning architectures for
time-varying graph inference have been explored
in~\cite{cao2021spectral, deng2021graph, li2021dynamic}.
The interplay between dynamic graphical models and macroeconomic
factor structures has been studied by
\cite{barigozzi2015dynamic} and \cite{zheng2011estimation},
who highlighted the importance of low-rank components in capturing
systemic risk during financial crises.
Online Riemannian methods for tracking SPD-valued parameters have
been developed by \cite{jeuris2022survey} and
\cite{huang2022riemannian}, providing algorithmic foundations that
the present work builds upon. 

\begin{revised}
\emph{Three coupled difficulties.}
Despite this progress, three features that are simultaneously present
in the applications motivating this work are rarely addressed jointly.
\emph{(i) High dimensionality.} In the quarterly S\&P~500 data of
Section~\ref{sec:experiments}, a window contains $n\approx 63$ returns
of $p=400$ stocks ($n/p\approx 0.16$); the sample covariance is
singular, and any estimator must rely on low-dimensional structure.
\emph{(ii) Heavy tails.} Financial returns and BOLD signals contain
outlying samples, for which Gaussian likelihoods, and hence
\textsc{GLasso} and \textsc{TVGL}, are not robust.
\emph{(iii) Multiplicative evolution.} Volatility regimes and global
signal fluctuations rescale entire eigen-directions of the dependence
structure rather than adding a fixed perturbation to it. A Frobenius
penalty $\|\Theta_{t+1}-\Theta_t\|_F^2$ measures \emph{absolute}
changes, depends on the measurement unit of each node, and treats a
precision trajectory differently from the corresponding covariance
trajectory; it is therefore poorly matched to multiplicative dynamics.
Moreover, a temporal penalty is useful only if it reduces estimation
error, which requires a statistical analysis showing that coupling
windows pools information rather than merely adding smoothing bias.

\emph{Closest works and positioning.}
Several recent works address subsets of these difficulties
(Table~\ref{tab:positioning}).
Javaheri \emph{et al.}~\cite{javaheri2025time} learn time-varying
Laplacian-constrained graphs under Student-$t$ models in a semi-online
framework, without low-rank structure.
Graphical factor models~\cite{hippert2023learning} combine elliptical
likelihoods with a low-rank-plus-diagonal \emph{covariance} and
Riemannian optimization, building on Riemannian tools for low-rank
structured elliptical models~\cite{bouchard2021riemannian};
Phi \emph{et al.}~\cite{phi2026leveraging} parameterize a low-rank
conditional-correlation precision through a quotient of the oblique
manifold; and~\cite{zhang2025graph} studies the LRaD \emph{precision}
model under a Gaussian likelihood. These methods are static.
The geometric ingredients used here, namely the affine-invariant
metric on $\Spp$~\cite{pennec2006riemannian, bhatia2009positive} and
quotient geometries for low-rank factors~\cite{absil2008optimization,
massart2020quotient, zheng2025riemannian}, are classical, and we do not
claim them as contributions. What is missing is a statistically
grounded account of (a)~when the LRaD precision parameterization is
identifiable, (b)~how an intrinsic temporal coupling of structured
precision matrices reduces estimation error, and (c)~what is estimated
when the LRaD hypothesis fails.
\end{revised}

\begin{table}[!t]
\color{black}
\centering
\footnotesize
\caption{Positioning of \textsc{Degfm} with respect to the closest graph-learning methods.}
\label{tab:positioning}
\setlength{\tabcolsep}{1.5pt}
\resizebox{\columnwidth}{!}{%
\begin{tabular}{lcccc}
\toprule
Method & Dynamic & Heavy tails & Low-rank structure & Temporal coupling\\
\midrule
GLasso~\cite{friedman2008sparse} & -- & -- & -- & --\\
TVGL~\cite{hallac2017network} & \checkmark & -- & -- & Euclidean\\
Javaheri'25~\cite{javaheri2025time} & \checkmark & Student-$t$ & -- (Laplacian) & semi-online\\
GFM~\cite{hippert2023learning} & -- & elliptical & covariance & --\\
Phi'26~\cite{phi2026leveraging} & -- & --$^{a}$ & cond.\ correlation & --\\
IndGGFM~\cite{zhang2025graph} & -- & -- & LRaD precision & --\\
\textsc{Degfm} & \checkmark & elliptical & LRaD precision & affine-invariant\\
\bottomrule
\end{tabular}}\\[1mm]
\raggedright{\scriptsize $^{a}$Generic objectives; instantiated with a Gaussian likelihood.}
\end{table}

\begin{revised}
\emph{Contributions.} We propose Dynamic Elliptical Graph Factor Models
(\textsc{Degfm}) and make the following contributions.
\begin{itemize}
\item \emph{Model.} \textsc{Degfm} couples an elliptical likelihood, an
LRaD precision $\Theta_t=\mathbf{Y}_t\mathbf{Y}_t^\top+\mathbf{D}_t$,
and the temporal energy $\mu\sum_t\dS^2(\Theta_t,\Theta_{t+1})$, where
$\dS$ is the affine-invariant distance on $\Spp$. We show that the LRaD
precision is exactly the posterior precision of a field conditioned on
$r$ noisy aggregate measurements,
that its graph is the \emph{partial-correlation} graph (which coincides
with the conditional-independence graph only in the Gaussian case), and
that the temporal energy equals $\sum_t\|\mathbf{H}_t\|_F^2$ for the
multiplicative evolution
$\Theta_{t+1}=\Theta_t^{1/2}\exp(\mathbf{H}_t)\Theta_t^{1/2}$ and is
invariant to node rescaling and to inversion.
\item \emph{Inference.} We pose estimation on a product of quotient
manifolds, derive closed-form Euclidean gradients of the squared
affine-invariant distance together with an exact evaluation through one
generalized eigendecomposition per pair of windows, and design a safeguarded Riemannian
conjugate-gradient (RCG) solver with an explicit per-iteration cost of
$\mathcal{O}(T(n_{\max}pr+p^3))$.
\item \emph{Theory.} We prove (a) local identifiability of the LRaD
parameterization modulo $\Or$ under a tangent-space transversality
condition, which holds whenever the leverage scores of the factor
subspace are below $1/2$;
(b) \nrev{well-posedness of the solver under a rank-preserving Armijo
backtracking, which is needed because the factor retraction is only
locally defined, and its convergence to critical points with an
$\mathcal{O}(1/K)$ bound on the squared gradient norm}; (c) a finite-sample oracle inequality
whose stochastic term scales as
$\dr\log(pT)\,\dtime(\mu)/n_{\min}$, with $\dr=pr+p-r(r-1)/2$ and an
effective temporal dimension $\dtime(\mu)\in[1/T,1]$, and whose bias
term is linear in $\mu$\nrev{, whose curvature condition is
proved for Student-$t$ observations with an unimprovable dependence on
$\nu/(p+\nu)$}; (d) an oracle inequality under LRaD
misspecification that formalizes the rank trade-off; and (e) exact recovery of the
partial-correlation graphs by thresholding, under a condition that
makes the role of $\mu\delta$ explicit.
\item \emph{Experiments.} Experiments on synthetic benchmarks and on
S\&P~500 and HCP fMRI data illustrate the behaviour of the proposed
framework.
\end{itemize}

The remainder of the paper is organized as follows.
Section~\ref{sec:background} reviews Gaussian graphical models,
elliptical distributions, the LRaD parameterization and its generative
interpretation, and basic notions of Riemannian optimization.
Section~\ref{sec:problem} formulates the \textsc{Degfm} objective and
analyzes the affine-invariant temporal energy.
Section~\ref{sec:geometry} develops the quotient geometry, the
identifiability result, the exact gradients, the solver, and its
complexity.
Section~\ref{sec:theory} establishes the optimization and statistical
guarantees. Section~\ref{sec:experiments} reports empirical results, and
Section~\ref{sec:conclusion} concludes. All proofs are deferred to the
appendices.
\end{revised}

\paragraph{Notation.}
$\Spp$ ($\mathcal{S}_+^{p,r}$) denotes $p\times p$ symmetric positive
definite (positive semidefinite of rank $r$) matrices\nrev{; $\mathcal{S}^p$ denotes $p\times p$ symmetric matrices}.
$\Dp$ is the set of $p\times p$ diagonal positive definite matrices\nrev{, and $\Dsp$ the space of $p\times p$ diagonal matrices}.
$\Rpr:=\{\mathbf{Y}\in\R^{p\times r}:\det(\mathbf{Y}^\top \mathbf{Y})\neq 0\}$ is the set of full-rank
$p\times r$ matrices.
$\Or$ is the orthogonal group in dimension $r$.
For a matrix $\mathbf{A}$: $\sym(\mathbf{A})=\tfrac{1}{2}(\mathbf{A}+\mathbf{A}^\top)$,
$\skew(\mathbf{A})=\tfrac{1}{2}(\mathbf{A}-\mathbf{A}^\top)$, $\ddiag(\mathbf{A})$ sets all off-diagonal entries to zero. $\mathbb{R}_+$ denotes the set of all positive real numbers. 
$\lambda_j(\mathbf{A})$ denotes the $j$-th eigenvalue of the matrix $\mathbf{A}$, arranged in non-increasing order. Finally, $\norm{\cdot}_{0,\mathrm{off}}$ counts the number of nonzero off-diagonal entries.
\nrev{We write $\|\cdot\|_F$, $\|\cdot\|_{\op}$, $\|\cdot\|_*$, and $\|\cdot\|_{\max}$ for the Frobenius, spectral, nuclear, and entrywise maximum norms; $\langle\mathbf{A},\mathbf{B}\rangle=\tr(\mathbf{A}^\top\mathbf{B})$. For a linear subspace $\mathcal{T}\subseteq\mathcal{S}^p$, $\PT_{\mathcal{T}}$ is the Frobenius-orthogonal projector onto $\mathcal{T}$ and $\PT_{\mathcal{T}}^\perp=\mathcal{I}-\PT_{\mathcal{T}}$. $\LT\in\R^{T\times T}$ is the Laplacian of the path graph on $\{1,\ldots,T\}$, so that $\sum_{t=1}^{T-1}\|\mathbf{A}_{t+1}-\mathbf{A}_t\|_F^2=\sum_{s,t}[\LT]_{st}\langle\mathbf{A}_s,\mathbf{A}_t\rangle$. For $\mathbf{Y}\in\Rpr$, $\mathbf{P}_{\mathbf{Y}}=\mathbf{Y}(\mathbf{Y}^\top\mathbf{Y})^{-1}\mathbf{Y}^\top$ and $\sigma_r(\mathbf{Y})$ is its smallest singular value.}

\section{Preliminaries}
\label{sec:background}
This section collects the background needed for the proposed dynamic elliptical graph factor model. We first review the standard Gaussian graphical model formulation and then introduce elliptical distributions as a robust generalization of Gaussian observations. We next describe the low-rank-plus-diagonal (LRaD) parameterization on precision matrices, which is the structural core of our model\nrev{, together with a generative interpretation and its implications for graph structure}. Finally, we recall the basic notions of Riemannian optimization that will be used to solve the resulting constrained problem on matrix manifolds.
\subsection{Gaussian Graphical Models and Graph Learning}

In a Gaussian graphical model (GGM), an undirected graph $\mathcal{G}=(V,E)$ with $|V|=p$ vertices encodes conditional independence relations through the precision matrix $\Theta=\Sigma^{-1}$. Specifically, $(i,j)\in E$ if and only if $\Theta_{ij}\neq 0$. Let $\mathbf{x}_i\in\mathbb{R}^p$ denote the $i$-th zero-mean graph signal defined on the $p$ vertices of $\mathcal{G}$, where each entry corresponds to the signal value associated with one node. Given $n$ i.i.d.\ samples, graph learning is commonly formulated as a penalized maximum likelihood problem \cite{dong2016learning}:
\[
\min_{\Theta\in\mathcal{S}_{++}^p}
\Bigl[\operatorname{tr}(\mathbf{S}\Theta)-\log\det(\Theta)+\lambda\,h(\Theta)\Bigr],
\]
where $\mathbf{S}=\frac{1}{n}\sum_{i=1}^n \mathbf{x}_i\mathbf{x}_i^\top$ is the sample covariance matrix, $\lambda\in\mathbb{R}_+$ is a regularization parameter, and $h(\Theta)=\sum_{i\neq j}|\Theta_{ij}|$ is the $\ell_1$ penalty used in graphical Lasso-type formulations. This convex formulation is attractive, but it does not exploit latent low-dimensional structure when present in the precision matrix \cite{zhang2025graph}. Moreover, the Gaussian assumption renders it sensitive to outliers and heavy-tailed observations, a limitation addressed by the distributional generalization described next.

\subsection{Elliptical Distributions}

Let $\mathbf{x}\in\mathbb{R}^p$ be a zero-mean random vector. We say that $\mathbf{x}$ follows an elliptically symmetric distribution with scatter matrix $\Sigma\in\mathcal{S}_{++}^p$ and density generator $g:(0,\infty)\to(0,\infty)$, written as $\mathbf{x}\sim \mathcal{ES}(0,\Sigma,g)$, if its density can be expressed as
\[
f(\mathbf{x};\Sigma)\propto \det(\Sigma)^{-1/2}\,g\!\left(\mathbf{x}^\top \Sigma^{-1}\mathbf{x}\right).
\]
This family strictly generalizes the Gaussian model. In particular, choosing $g(s)=e^{-s/2}$ recovers the multivariate normal distribution, while the multivariate $t$ distribution with $\nu\nrev{>0}$ degrees of freedom corresponds to $g(s)=(1+s/\nu)^{-(\nu+p)/2}$. Throughout the paper, we write $\rho=-\log g$ and $u(s)=-2g'(s)/g(s)=2\rho'(s)$ for the associated \nrev{weight} function\nrev{, and $\psi(s):=s\,u(s)$ for the weighted radius}. Under the Gaussian model, $u\equiv 1$; for the $t$ distribution, $u(s)=(\nu+p)/(\nu+s)$, which downweights large quadratic forms and yields robustness to heavy-tailed observations\nrev{; in that case $\psi(s)=(\nu+p)s/(\nu+s)\le \nu+p$ is bounded, whereas $\psi(s)=s$ is unbounded in the Gaussian case}.
\nrev{Every $\mathbf{x}\sim\mathcal{ES}(0,\Sigma,g)$ admits the stochastic representation $\mathbf{x}=R\,\Sigma^{1/2}\mathbf{s}$, where $\mathbf{s}$ is uniform on the unit sphere and $R\ge 0$ is independent of $\mathbf{s}$~\cite{fang2018symmetric}. For $\nu\le 2$ the $t$ distribution has no second moments; all statements below are therefore phrased in terms of the scatter matrix $\Sigma$ and its inverse, which exist for every elliptical law.}

While the elliptical model broadens the distributional scope of GGMs, the high-dimensional regime introduces a separate challenge: the precision matrix has $\mathcal{O}(p^2)$ free parameters. The following subsection introduces a structured low-rank-plus-diagonal parameterization that reduces the effective dimensionality to $\mathcal{O}(pr)$ ($r\ll p$) while preserving the positive-definiteness constraint.

\subsection{Low-Rank-Plus-Diagonal Structure on Precision Matrices}
\label{sec:lrad}

We assume that the precision matrix admits a \emph{low-rank-plus-diagonal}(LRaD) decomposition
\[
\Theta=\mathbf{L}+\mathbf{D},\qquad \mathbf{D}\in\mathcal{D}_{++}^p,\quad \mathbf{L}\in\mathcal{S}_{+}^{p,r},
\]
where $\mathbf{L}$ is positive semidefinite with rank $r\ll p$. This is the low-rank-plus-diagonal precision structure \cite{zhang2025graph}. Since every rank-$r$ positive semidefinite matrix can be factorized as $\mathbf{L}=\mathbf{Y}\mathbf{Y}^\top$ with $\mathbf{Y}\in\mathbb{R}^{p\times r}$ full column rank, we may equivalently parameterize
\[
\Theta=\mathbf{Y}\mathbf{Y}^\top+\mathbf{D}.
\]
This representation \nrev{enforces positive definiteness by construction and reduces the number of free parameters from $p(p+1)/2$ to
\begin{equation}
  \dr:=pr-\tfrac{r(r-1)}{2}+p,
  \label{eq:dr}
\end{equation}
where the term $r(r-1)/2=\dim\Or$ accounts for the invariance $\mathbf{Y}\mapsto\mathbf{Y}\mathbf{O}$, $\mathbf{O}\in\Or$. That $\dr$ is indeed the local dimension of the model, i.e., that no further redundancy arises from exchanging mass between $\mathbf{L}$ and $\mathbf{D}$, is not automatic; it is established in Theorem~\ref{thm:identifiability} under a transversality condition.}

\begin{revised}
\emph{Generative interpretation.}
The classical observation model $\mathbf{x}=\mathbf{B}\mathbf{z}+\boldsymbol{\epsilon}$
produces an LRaD \emph{covariance} $\Sigma=\mathbf{B}\mathbf{B}^\top+\Psi$,
whose precision $\Psi^{-1}-\Psi^{-1}\mathbf{B}(\mathbf{I}_r+\mathbf{B}^\top\Psi^{-1}\mathbf{B})^{-1}\mathbf{B}^\top\Psi^{-1}$
is diagonal \emph{minus} a low-rank matrix. The \emph{plus} sign in the
LRaD precision arises instead from conditioning, as follows.

\begin{proposition}[LRaD precision as conditioning on aggregate measurements]
\label{prop:conditioning}
Let $\mathbf{x}\sim\mathcal{N}(\mathbf{0},\mathbf{D}^{-1})$ with
$\mathbf{D}\in\Dp$, and let $\mathbf{z}=\mathbf{Y}^\top\mathbf{x}+\boldsymbol{\varepsilon}$
with $\boldsymbol{\varepsilon}\sim\mathcal{N}(\mathbf{0},\mathbf{I}_r)$
independent of $\mathbf{x}$ and $\mathbf{Y}\in\R^{p\times r}$. Then
$\mathbf{x}\mid\mathbf{z}\sim\mathcal{N}\bigl(\Theta^{-1}\mathbf{Y}\mathbf{z},\,\Theta^{-1}\bigr)$
with $\Theta=\mathbf{Y}\mathbf{Y}^\top+\mathbf{D}$.
\end{proposition}
\begin{proof}
By Bayes' rule,
$p(\mathbf{x}\mid\mathbf{z})\propto\exp\{-\tfrac12\mathbf{x}^\top\mathbf{D}\mathbf{x}-\tfrac12\|\mathbf{z}-\mathbf{Y}^\top\mathbf{x}\|^2\}
\propto\exp\{-\tfrac12\mathbf{x}^\top(\mathbf{D}+\mathbf{Y}\mathbf{Y}^\top)\mathbf{x}+\mathbf{x}^\top\mathbf{Y}\mathbf{z}\}$;
completing the square gives the claim.
\end{proof}

Thus $\Theta=\mathbf{Y}\mathbf{Y}^\top+\mathbf{D}$ describes nodes that
are a priori conditionally uncorrelated (diagonal $\mathbf{D}$) and become
coupled through $r$ noisy aggregate quantities $\mathbf{Y}^\top\mathbf{x}$,
such as sector indices, market-wide constraints, or network-level
signals. Each column of $\mathbf{Y}$ is the loading profile of one
aggregate.

\emph{Graph structure induced by LRaD.} For $i\neq j$,
$\Theta_{ij}=\mathbf{y}_i^\top\mathbf{y}_j$, where $\mathbf{y}_i^\top$ is the
$i$-th row of $\mathbf{Y}$. Hence $(i,j)$ is an edge if and only if
$\mathbf{y}_i$ and $\mathbf{y}_j$ are not orthogonal. For generic dense
$\mathbf{Y}$, the graph is complete; sparsity requires many exact row
orthogonalities. The prototypical sparse LRaD family arises when each node
loads on a subset $S_i\subseteq\{1,\ldots,r\}$ of aggregates with generic
weights: then $(i,j)$ is an edge if and only if $S_i\cap S_j\neq\emptyset$,
i.e., the graph is a union of $r$ (possibly overlapping) cliques, one per
aggregate. LRaD is therefore a specific \emph{community-type} graph
family, appropriate for sector- or network-organized systems, rather
than a model for arbitrary sparse graphs. Graphs outside this family are
handled by the misspecification analysis of Section~\ref{sec:theory}.
\end{revised}

\begin{remark}[Relation to \nrev{other} factor models]
    Earlier factor-based graph learning methods \cite{hippert2023learning} typically impose low-rank structure on the covariance matrix and recover the precision matrix only after inversion\nrev{, which yields a diagonal-minus-low-rank precision; the low-rank conditional-correlation model of~\cite{phi2026leveraging} constrains the normalized precision instead. These model classes are not nested, and none dominates the others; the LRaD precision is the natural choice when the dependence arises from conditioning on aggregates (Proposition~\ref{prop:conditioning}). Parameterizing the precision directly also gives direct access to the graph and avoids an explicit inversion.}
\end{remark}

The LRaD parameterization naturally defines a manifold of feasible precision matrices.
Optimizing over this manifold requires tools from Riemannian geometry, which we now briefly introduce.
 
\begin{figure}[!htbp]
    \centering
    \includegraphics[width=0.8\linewidth]{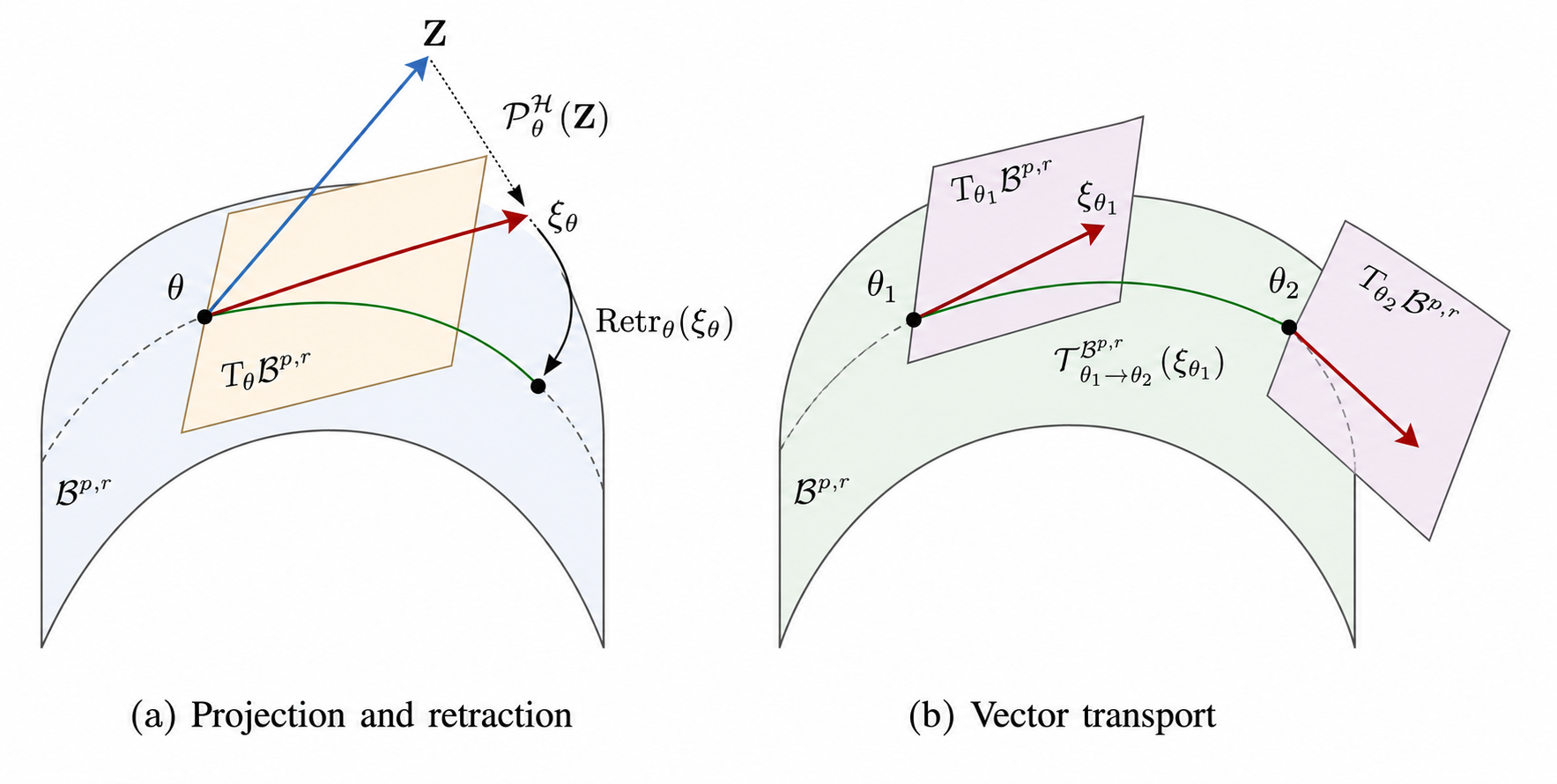}
    \caption{Illustration of the key geometric operations in Riemannian optimization on $\mathcal{B}^{p,r}$. The illustration follows the construction in \cite{zhang2025graph}.}
    \label{fig:riemannian_geometry}
\end{figure}
\subsection{Riemannian Optimization: Key Concepts}

Let $\mathcal{M}$ be a smooth Riemannian manifold equipped with a \emph{Riemannian metric} $\langle\cdot,\cdot\rangle_x$ on each \emph{tangent space} $T_x\mathcal{M}$. For a smooth function $f:\mathcal{M}\to\mathbb{R}$, the \emph{Riemannian gradient} $\operatorname{grad} f(x)\in T_x\mathcal{M}$ is the unique tangent vector satisfying
\[
\langle \operatorname{grad} f(x),\xi\rangle_x = Df(x)[\xi],
\qquad \forall\,\xi\in T_x\mathcal{M}.
\]
A \emph{retraction} $\mathrm{Retr}_x:T_x\mathcal{M}\to\mathcal{M}$ maps a tangent vector back to the manifold and agrees with the exponential map to first order at the origin. A \emph{vector transport} $\mathcal{T}_{x\to y}:T_x\mathcal{M}\to T_y\mathcal{M}$ provides a linear mechanism for moving tangent vectors between nearby points on the manifold. These constructions allow classical optimization methods to be extended to nonlinear matrix manifolds in a principled way \cite{absil2008optimization}. Figure~\ref{fig:riemannian_geometry} provides an intuitive illustration of projection, retraction, and vector transport on the manifold $\mathcal{B}^{p,r}$.

\section{Problem Formulation}
\label{sec:problem}
Building on the statistical and geometric foundations of Section~\ref{sec:background},
this section constructs the DEGFM optimization problem in three steps. We first specify the data model and the penalized per-window objective. We then introduce the geodesic temporal regularizer that couples consecutive precision matrices in a geometry-aware manner\nrev{, and we characterize it as the energy of a multiplicative evolution model}. Finally, we combine both components into a single optimization problem on a product quotient manifold.

\subsection{Data Model}

Let $T$ denote the number of time windows. In window $t\in\{1,\ldots,T\}$, we observe $n_t$ i.i.d.\ samples
$\{\mathbf{x}_{t,i}\}_{i=1}^{n_t}\subset\R^p$ generated from
\begin{equation}
  \mathbf{x}_{t,i}\sim\mathcal{ES}\!\left(0,\,\Theta_t^{-1},\,g_t\right),
  \qquad \Theta_t\in\Mpr,
  \label{eq:data_model}
\end{equation}
where $\Theta_t$ is the unknown precision matrix at time $t$\nrev{, and samples from different windows are independent}. The model class $
\Mpr := \left\{\Theta\in\Spp:\Theta=\mathbf{Y}\mathbf{Y}^\top+\mathbf{D},
\mathbf{Y}\in\Rpr, \mathbf{D}\in\Dp\right\}$ collects all low-rank-plus-diagonal (LRaD) precision matrices, and $ \Rpr:=\{\mathbf{Y}\in\R^{p\times r}:\det(\mathbf{Y}^\top \mathbf{Y})\neq 0\}$
denotes the set of full-column-rank factor matrices. Each $\Theta_t$ is parameterized by
$\theta_t=(\mathbf{Y}_t,\mathbf{D}_t)\in\Bpr:=\Rpr\times\Dp$ through the surjection
\begin{equation}
  \varphi:\Bpr\to\Mpr,\qquad \varphi(\mathbf{Y},\mathbf{D})=\mathbf{Y}\mathbf{Y}^\top+\mathbf{D}.
  \label{eq:phi}
\end{equation}
This parameterization \nrev{has $\dr$ effective parameters, see~\eqref{eq:dr}}, and guarantees positive definiteness by construction.
\nrev{The membership $\Theta_t\in\Mpr$ is a modeling hypothesis; the theory of Section~\ref{sec:theory} also covers the case in which it fails.}

\begin{revised}
\emph{Which graph is estimated.} We define the graph at time $t$ as the
\emph{partial-correlation graph}
\begin{equation}
\begin{aligned}
  &E_t:=\{(i,j):i\neq j,\ \varrho_{t,ij}\neq 0\},\\
  &\varrho_{t,ij}:=-\frac{[\Theta_t]_{ij}}{\sqrt{[\Theta_t]_{ii}[\Theta_t]_{jj}}}.
\end{aligned}
  \label{eq:pcorr_graph}
\end{equation}
For Gaussian data, $\varrho_{t,ij}=0$ if and only if $x_i$ and $x_j$ are
conditionally independent given the remaining variables. For a general
elliptical law this equivalence fails: e.g., for a multivariate $\mathbf{t}$
vector with diagonal scatter, the components are uncorrelated but
dependent through the common radius $R$. A zero entry of the inverse
scatter then encodes zero partial correlation (conditional
uncorrelatedness) rather than conditional independence, and
\eqref{eq:pcorr_graph} is the object targeted throughout this paper.
\end{revised}

In the sequel, we regard $\Bpr$ as a smooth product manifold and perform the estimation of $\{\theta_t\}_{t=1}^T$ directly on $\Pi_{t=1}^T\Bpr$. This geometric formulation yields an unconstrained optimization problem on a manifold, with the associated Riemannian machinery deferred to Section~\ref{sec:geometry}.

\subsection{Per-Window Robust Likelihood}

For each time window $t$, we define the penalized negative log-likelihood
\begin{equation}
  f_t(\theta_t)
  :=
  \mathcal{L}_t(\varphi(\theta_t))
  + \lambda\, h(\varphi(\theta_t)),
  \label{eq:ft}
\end{equation}
where $\lambda>0$ is the sparsity regularization parameter. The robust elliptical negative log-likelihood is
\begin{equation}
  \mathcal{L}_t(\Theta_t)
  =
  -\frac{1}{2}\log\det(\Theta_t)
  + \frac{1}{n_t}\sum_{i=1}^{n_t}
    \rho_t\!\left(\mathbf{x}_{t,i}^\top \Theta_t \mathbf{x}_{t,i}\right),
  \label{eq:likelihood}
\end{equation}
with $\rho_t(s)=-\log g_t(s)$. This reduces to the Gaussian likelihood when
$g_t(s)=e^{-s/2}$.

To promote sparsity in the off-diagonal entries while retaining smoothness for optimization, we use the differentiable surrogate \cite{hippert2023learning}
\begin{equation}
  h(\Theta_t)
  =
  \sum_{q\neq \ell}\psi_\varepsilon\!\left([\Theta_t]_{q\ell}\right),
  \qquad
  \psi_\varepsilon(s)=\varepsilon\log\cosh(s/\varepsilon),
  \label{eq:penalty}
\end{equation}
where $\varepsilon>0$ is a small smoothing parameter. \nrev{Its properties are collected in the following lemma (proof in Appendix~\ref{app:logcosh}).}

\begin{revised}
\begin{lemma}[Properties of the log-cosh surrogate]
\label{lem:logcosh}
For every $\varepsilon>0$ and $s\in\R$:
(i) $\psi_\varepsilon$ is even, convex, and $C^\infty$;
(ii) $0\le |s|-\psi_\varepsilon(s)\le\varepsilon\log 2$;
(iii) $\psi_\varepsilon'(s)=\tanh(s/\varepsilon)\in(-1,1)$ and
$\psi_\varepsilon''(s)=\varepsilon^{-1}\operatorname{sech}^2(s/\varepsilon)\in(0,\varepsilon^{-1}]$.
Consequently, $h$ is convex on $\mathcal{S}^p$, $\nabla h$ is
$\varepsilon^{-1}$-Lipschitz in the Frobenius norm,
$0\le\|\Theta\|_{1,\mathrm{off}}-h(\Theta)\le p(p-1)\varepsilon\log 2$,
and $[\nabla h(\Theta)]_{q\ell}=0$ whenever $\Theta_{q\ell}=0$; in
particular $\|\nabla h(\Theta)\|_F\le\|\Theta\|_{0,\mathrm{off}}^{1/2}$.
\end{lemma}

Lemma~\ref{lem:logcosh} makes the role of $\varepsilon$ explicit: it
controls both the uniform approximation error of the $\ell_1$ penalty
and the Lipschitz constant of $\nabla h$, so that smaller $\varepsilon$
brings $h$ closer to $\ell_1$ at the price of a worse conditioning of the
optimization problem. Since $\psi_\varepsilon$ is smooth at the origin,
$h$ does not produce exact zeros; the graph is therefore obtained by
thresholding the estimated partial correlations, and the support
guarantee of Theorem~\ref{thm:support} is a \emph{thresholded} recovery
result. The last property in Lemma~\ref{lem:logcosh}, namely that
$\nabla h$ vanishes on the zero pattern of $\Theta$, is what lets the
sparsity level enter the statistical bounds only through a controlled
bias term.
\end{revised}

\nrev{The gradient of $h$ is} given entrywise by
\[
\bigl[\nabla_{\Theta_t} h\bigr]_{q\ell}
=
\begin{cases}
\tanh\!\left([\Theta_t]_{q\ell}/\varepsilon\right), & q\neq \ell,\\[2mm]
0, & q=\ell.
\end{cases}
\]

The Euclidean gradient of $\mathcal{L}_t$ with respect to $\Theta_t$ is
\begin{equation}
  \nabla_{\Theta_t}\mathcal{L}_t
  =
  -\frac{1}{2}\Theta_t^{-1}
  + \frac{1}{2n_t}\sum_{i=1}^{n_t}
    u_t\!\left(\mathbf{x}_{t,i}^\top\Theta_t \mathbf{x}_{t,i}\right)
    \mathbf{x}_{t,i}\mathbf{x}_{t,i}^\top,
  \label{eq:likelihood_grad}
\end{equation}
where
\[
u_t(s):=-2\frac{g_t'(s)}{g_t(s)}.
\]
Hence the total Euclidean gradient of $f_t$ is
\begin{equation}
  \nabla_{\Theta_t} f_t
  =
  \nabla_{\Theta_t}\mathcal{L}_t
  + \lambda\,\nabla_{\Theta_t} h.
  \label{eq:grad_ft}
\end{equation}

The per-window objective $f_t$ captures local fit and sparsity at each time step independently.
To couple the sequence $\{\theta_t\}$ and promote temporal smoothness, we now introduce a regularizer based on the intrinsic geometry of $\Spp$.

\subsection{Geodesic Temporal Regularization}

To couple consecutive time windows, we penalize changes in the intrinsic geometry of the precision matrices using the affine-invariant Riemannian distance on $\Spp$\nrev{~\cite{pennec2006riemannian,bhatia2009positive}}:
\begin{align}
  \dS^2(\Theta,\widetilde\Theta)
  &=
  \norm{\log\!\left(\Theta^{-1/2}\widetilde\Theta\,\Theta^{-1/2}\right)}_F^2\nonumber\\
  &
  =
  \sum_{j=1}^p \log^2\!\lambda_j\!\left(\Theta^{-1}\widetilde\Theta\right).
  \label{eq:geo_dist}
\end{align}
\nrev{The following proposition states the properties that motivate this choice; its proof is given in Appendix~\ref{app:airm}.}

\begin{revised}
\begin{proposition}[Properties of affine-invariant Riemannian metric (AIRM)]
\label{prop:airm}
Let $\mathbf{A},\mathbf{B}\in\Spp$.
\begin{enumerate}
\item[(a)] \emph{Symmetry and inversion invariance:}
$\dS(\mathbf{A},\mathbf{B})=\dS(\mathbf{B},\mathbf{A})=\dS(\mathbf{A}^{-1},\mathbf{B}^{-1})$.
\item[(b)] \emph{Congruence invariance:}
$\dS(\mathbf{G}\mathbf{A}\mathbf{G}^\top,\mathbf{G}\mathbf{B}\mathbf{G}^\top)=\dS(\mathbf{A},\mathbf{B})$
for every invertible $\mathbf{G}$; in particular for
$\mathbf{G}=\mathbf{C}^{-1}$ with $\mathbf{C}$ diagonal, which corresponds to
the node rescaling $x_i\mapsto c_ix_i$.
\item[(c)] \emph{Multiplicative evolution:} there is a unique
$\mathbf{H}\in\mathcal{S}^p$ with
$\mathbf{B}=\mathbf{A}^{1/2}\exp(\mathbf{H})\mathbf{A}^{1/2}$, namely
$\mathbf{H}=\log(\mathbf{A}^{-1/2}\mathbf{B}\mathbf{A}^{-1/2})$, and
$\dS(\mathbf{A},\mathbf{B})=\|\mathbf{H}\|_F$.
\item[(d)] \emph{Relative-change expansion:} as
$\boldsymbol{\Delta}\to\mathbf{0}$ in $\mathcal{S}^p$,
$\dS^2(\mathbf{A},\mathbf{A}+\boldsymbol{\Delta})
=\|\mathbf{A}^{-1/2}\boldsymbol{\Delta}\mathbf{A}^{-1/2}\|_F^2+\mathcal{O}(\|\boldsymbol{\Delta}\|_F^3)$.
\item[(e)] \emph{Local equivalence with Frobenius:} if
$m\mathbf{I}\preceq\mathbf{A},\mathbf{B}\preceq M\mathbf{I}$, then
\[
\frac{m^2}{M^4}\|\mathbf{A}-\mathbf{B}\|_F^2\le\dS^2(\mathbf{A},\mathbf{B})\le\frac{M^2}{m^4}\|\mathbf{A}-\mathbf{B}\|_F^2.
\]
\end{enumerate}
\end{proposition}

\emph{Why an affine-invariant temporal energy.}
Proposition~\ref{prop:airm}(c) shows that the penalty
$\sum_t\dS^2(\Theta_t,\Theta_{t+1})=\sum_t\|\mathbf{H}_t\|_F^2$ is the
quadratic energy of the increments of the multiplicative evolution model
\begin{equation}
  \Theta_{t+1}=\Theta_t^{1/2}\exp(\mathbf{H}_t)\,\Theta_t^{1/2},
  \qquad \mathbf{H}_t\in\mathcal{S}^p,
  \label{eq:mult_model}
\end{equation}
which keeps every iterate in $\Spp$ and in which a global rescaling
$\Theta_{t+1}=c\,\Theta_t$ costs $p\log^2c$ irrespective of the magnitude
of $\Theta_t$. By (d), to second order the penalty measures the
\emph{relative} change $\Theta_t^{-1/2}\Delta_t\Theta_t^{-1/2}$, whereas
$\|\Delta_t\|_F^2$ measures the \emph{absolute} change and is dominated by
the nodes with the largest precision. By (b), the penalty, the LRaD class
(since $\mathbf{C}^{-1}(\mathbf{Y}\mathbf{Y}^\top+\mathbf{D})\mathbf{C}^{-1}
=(\mathbf{C}^{-1}\mathbf{Y})(\mathbf{C}^{-1}\mathbf{Y})^\top+\mathbf{C}^{-1}\mathbf{D}\mathbf{C}^{-1}$),
and the partial correlations~\eqref{eq:pcorr_graph} are all invariant to
the units in which each node is measured; a Frobenius penalty is not. By
(a), regularizing the precision trajectory is equivalent to regularizing
the covariance (scatter) trajectory, so the temporal model does not depend
on whether the dynamics are thought of in the precision domain, where the
graph lives, or in the covariance domain, where volatility changes are
usually described. Finally, the geodesic
$\gamma(y)=\Theta^{1/2}(\Theta^{-1/2}\widetilde\Theta\Theta^{-1/2})^y\Theta^{1/2}$,
$y\in[0,1]$, remains in $\Spp$, so the penalty never favors an
interpolation that leaves the cone. Property (e) is used in the
statistical analysis.

The affine-invariant energy targets smooth, multiplicative drift; it is
not designed to produce sparse edge changes, for which fused
$\ell_1$ penalties~\cite{hallac2017network,gibberd2017regularized} are
better suited. The two mechanisms are complementary.
\end{revised}

Having defined both the per-window likelihood and the geodesic temporal penalty, we now assemble the complete DEGFM formulation.

\subsection{The DEGFM Optimization Problem}

We now define the \emph{Dynamic Elliptical Graph Factor Model} (DEGFM):
\begin{equation}
  \min_{\{\theta_t\}_{t=1}^T\subset\Bpr}
  \;
  F(\{\theta_t\})
  :=
  \sum_{t=1}^T f_t(\theta_t)
  +
  \mu\sum_{t=1}^{T-1}
  c_t(\theta_t,\theta_{t+1}),
  \label{eq:DEPFM}
\end{equation}
where $\mu\nrev{\ge 0}$ controls temporal regularization and
\[
c_t(\theta_t,\theta_{t+1})
:=
\dS^2\!\left(\varphi(\theta_t),\varphi(\theta_{t+1})\right).
\]
Thus, each window is fitted by a robust elliptical likelihood with sparse off-diagonal structure, while successive precision matrices are encouraged to move smoothly along the affine-invariant geodesic.

\begin{revised}
\begin{remark}[Probabilistic interpretation of the temporal term]
The Riemannian Gaussian distribution on $\Spp$ has density
$\propto\exp\{-\dS^2(\Theta,\bar\Theta)/(2s^2)\}$ with respect to the
affine-invariant volume measure, and its normalizing constant does not
depend on the center $\bar\Theta$~\cite{said2017riemannian}. Hence, for a
random walk $\Theta_{t+1}\mid\Theta_t$ with this law, the negative
log-prior of a trajectory equals $\frac{1}{2s^2}\sum_t\dS^2(\Theta_t,\Theta_{t+1})$
up to an additive constant. Problem~\eqref{eq:DEPFM} can therefore be
read as a penalized-likelihood (MAP-type) estimator under a Riemannian
Gaussian random-walk prior restricted to $\Mpr$, with $\mu$ playing the
role of $1/(2s^2)$ up to the per-sample normalization in~\eqref{eq:ft}.
This interpretation is heuristic, because the prior is restricted to the
submanifold $\Mpr$ and the reference measure differs from the Lebesgue
measure used by the likelihood; it is not used in the analysis.
\end{remark}
\end{revised}

\begin{remark}[Comparison with TVGL\nrev{, IndGGFM, and heavy-tailed dynamic methods}]
Temporal graphical Lasso (TVGL) \cite{hallac2017network} uses a \nrev{Euclidean} penalty between consecutive precision matrices and is typically formulated under a Gaussian likelihood. In contrast, DEGFM directly parameterizes the precision matrix in LRaD form, allows elliptical heavy-tailed observations, and regularizes temporal variation through the intrinsic geodesic distance on $\Spp$. \nrev{Setting $\mu=0$ decouples~\eqref{eq:DEPFM} into $T$ independent static problems; with a Gaussian generator this is the IndGGFM baseline~\cite{zhang2025graph}. Unlike~\cite{javaheri2025time}, which updates Laplacian-constrained graphs sequentially, \eqref{eq:DEPFM} estimates the whole trajectory jointly.}
\end{remark}

\section{Dynamic Elliptical Graph Factor Models via Riemannian Optimization}
\label{sec:geometry}
This section develops the complete Riemannian optimization machinery for solving~\eqref{eq:DEPFM}.
The parameter space $\Bpr = \Rpr\times\Dp$ is an open product manifold, but the map $\varphi$ introduces a gauge redundancy indexed by $\Or$.
\nrev{Subsection}~\ref{sec:geom_manifold} constructs the appropriate quotient manifold $\Bpr/\Or$, \nrev{establishes when it represents LRaD precision matrices without further redundancy,} equips it with a Riemannian metric, and derives all differential-geometric objects required by an optimization algorithm: the horizontal projection, the Riemannian gradient, the retraction, and the vector transport.
\nrev{Subsection}~\ref{sec:algorithm} then assembles these ingredients into a \nrev{safeguarded} Riemannian conjugate gradient (RCG) solver and analyzes its per-iteration complexity in \nrev{Subsection}~\ref{complexity}.

\begin{revised}
\emph{Why a Riemannian formulation.} One could run a Euclidean gradient
method on $(\mathbf{Y},\mathbf{D})$. The Riemannian formulation brings
three concrete benefits, none of which relies on the geometry being new.
(i)~\emph{Feasibility without projections or barriers:} the retraction
below keeps $\mathbf{D}\succ\mathbf{0}$ for every step size, whereas a
Euclidean step on $\mathbf{D}$ requires a projection or a barrier.
\nrev{Full column rank of $\mathbf{Y}$ is an open condition that the
retraction preserves only for sufficiently small steps; it is enforced
by the backtracking rule of Section~\ref{sec:algorithm}.}
(ii)~\emph{Gauge invariance:} the objective is constant along the orbits
$\{(\mathbf{Y}\mathbf{O},\mathbf{D})\}$; search directions are kept
horizontal and momentum terms are transported horizontally, so the
iteration is well defined on the quotient, where, by
Theorem~\ref{thm:identifiability}, nondegenerate critical points are
isolated. (iii)~\emph{Scale adaptivity:} with the affine-invariant metric
on $\Dp$, the Riemannian gradient in the $\mathbf{D}$-block is
$\mathbf{D}G_{\mathbf{D}}\mathbf{D}$, which is, to first order, the
Euclidean gradient in the coordinates $\log\mathbf{D}$; the updates of
$\mathbf{D}$ are therefore relative and invariant to node rescaling,
consistently with the temporal energy of Proposition~\ref{prop:airm}(b).
\end{revised}

\subsection{Riemannian Geometry of \texorpdfstring{$\Bpr$}{B^{p,r}}}
\label{sec:geom_manifold}

We develop the geometry of $\Bpr$ in a natural logical order: we first identify the gauge freedom in the parameterization $\varphi$ and construct the quotient manifold that eliminates it; we then introduce the Riemannian metric and decompose the tangent space into vertical and horizontal components; and we finally derive the horizontal projection and the Riemannian gradient formula, which together enable optimization on the quotient space.
The retraction and vector transport, which map gradient steps back to the manifold, are given at the end of this subsection.

\emph{(a) Quotient Manifold Structure \nrev{and Identifiability}:} The mapping $\varphi$ is not injective: for any $\mathbf{O}\in\Or$,
$\varphi(\mathbf{Y}\mathbf{O},\mathbf{D}) = \mathbf{Y}\mathbf{O}\mathbf{O}^\top \mathbf{Y}^\top + \mathbf{D} = \mathbf{Y}\mathbf{Y}^\top + \mathbf{D} = \varphi(\mathbf{Y},\mathbf{D})$.
The invariance property renders that the solution of \eqref{eq:DEPFM} is not isolated. Hence the natural parameter space accounting for this gauge freedom is the
quotient set
\begin{align}
  \Bpr/\Or &:= \{[\theta]:\theta\in\Bpr\},
  \quad [\theta]:=\{\theta * \mathbf{O}:\mathbf{O}\in\Or\},\nonumber\\
  \theta * \mathbf{O} &:= (\mathbf{Y}\mathbf{O},\mathbf{D}).
\end{align}
The quotient map is $\pi:\Bpr\to\Bpr/\Or$, $\pi(\theta)=[\theta]$.
\nrev{Since $\Or$ is compact and acts freely on $\Bpr$ ($\mathbf{Y}\mathbf{O}=\mathbf{Y}$ with $\mathbf{Y}$ of full column rank forces $\mathbf{O}=\mathbf{I}_r$), the quotient manifold theorem~\cite[Thm.~21.10]{lee2012introduction} shows that $\Bpr/\Or$ is a smooth manifold of dimension $\dr$ and that $\varphi$ descends to a smooth map $\bar\varphi:\Bpr/\Or\to\mathcal{S}^p$, $\bar\varphi([\theta])=\varphi(\theta)$.}
In practice, we work with representatives $\theta\in\Bpr$ and eliminate
gauge freedom via horizontal-space projection.

\begin{revised}
The quotient removes the orthogonal ambiguity, but $\varphi$ may in
principle be non-injective for another reason: a diagonal perturbation
might be absorbed by the low-rank part. To quantify this, let
$\mathbf{L}=\mathbf{Y}\mathbf{Y}^\top$ and let
\begin{equation}
  T_{\mathbf{L}}:=\{\mathbf{Y}\mathbf{W}^\top+\mathbf{W}\mathbf{Y}^\top:\mathbf{W}\in\R^{p\times r}\}
  \label{eq:TL}
\end{equation}
be the tangent space at $\mathbf{L}$ of the manifold of rank-$r$ positive
semidefinite matrices, whose Frobenius projector is
$\PT_{T_{\mathbf{L}}}(\mathbf{X})=\mathbf{P}_{\mathbf{Y}}\mathbf{X}+\mathbf{X}\mathbf{P}_{\mathbf{Y}}-\mathbf{P}_{\mathbf{Y}}\mathbf{X}\mathbf{P}_{\mathbf{Y}}$
for $\mathbf{X}\in\mathcal{S}^p$. Define the \emph{transversality
constant}
\begin{equation}
\begin{aligned}
  \alpha(\mathbf{Y})&:=\|\PT_{T_{\mathbf{L}}}\PT_{\Dsp}\|_{\op}\\
  &=\max_{\boldsymbol{\Delta}\in\Dsp,\ \|\boldsymbol{\Delta}\|_F=1}\|\PT_{T_{\mathbf{L}}}(\boldsymbol{\Delta})\|_F\in[0,1],
\end{aligned}
  \label{eq:alpha}
\end{equation}
the cosine of the minimal principal angle between $T_{\mathbf{L}}$ and the
diagonal subspace $\Dsp$, and the \emph{coherence}
$\mu_c(\mathbf{Y}):=\max_{i}\|\mathbf{P}_{\mathbf{Y}}\mathbf{e}_i\|_2^2$
(the largest leverage score of $\operatorname{col}(\mathbf{Y})$). 
For $\theta=(\mathbf{Y},\mathbf{D})\in\Bpr$, let
$D\varphi(\theta):T_\theta\Bpr\to\mathcal{S}^p$
denote the differential of the parameterization
$\varphi$ at $\theta$. Its action on a tangent vector
$\xi=(\xi_{\mathbf{Y}},\xi_{\mathbf{D}})$ is given by
\begin{equation}
D\varphi(\theta)[\xi]
=\xi_{\mathbf{Y}}\mathbf{Y}^\top
+\mathbf{Y}\xi_{\mathbf{Y}}^\top
+\xi_{\mathbf{D}}.
\label{eq:differential_varphi}
\end{equation}

\begin{theorem}[Local identifiability of the LRaD parameterization]
\label{thm:identifiability}
Let $\theta=(\mathbf{Y},\mathbf{D})\in\Bpr$ and
$\mathcal{V}_\theta,\mathcal{H}_\theta$ be the vertical and horizontal
spaces defined in~\eqref{eq:vertical}--\eqref{eq:horizontal} below.
\begin{enumerate}
\item[(i)] $\ker D\varphi(\theta)\supseteq\mathcal{V}_\theta$, with
equality if and only if $T_{\mathbf{L}}\cap\Dsp=\{\mathbf{0}\}$, i.e.,
if and only if $\alpha(\mathbf{Y})<1$.
\item[(ii)] For every horizontal $\xi=(\xi_{\mathbf{Y}},\xi_{\mathbf{D}})\in\mathcal{H}_\theta$,
\[
\|D\varphi(\theta)[\xi]\|_F^2\ge\bigl(1-\alpha(\mathbf{Y})\bigr)
\bigl(2\sigma_r^2(\mathbf{Y})\|\xi_{\mathbf{Y}}\|_F^2+\|\xi_{\mathbf{D}}\|_F^2\bigr).
\]
\item[(iii)] If $\alpha(\mathbf{Y}^\star)<1$ at
$\theta^\star=(\mathbf{Y}^\star,\mathbf{D}^\star)$, then $\bar\varphi$ is
an immersion on a neighborhood of $[\theta^\star]$, and there is an
$\Or$-invariant open neighborhood $\mathcal{U}\ni\theta^\star$ in $\Bpr$
such that, for $\theta_1,\theta_2\in\mathcal{U}$,
$\varphi(\theta_1)=\varphi(\theta_2)$ if and only if
$\theta_2=\theta_1*\mathbf{O}$ for some $\mathbf{O}\in\Or$.
\item[(iv)] $\alpha(\mathbf{Y})\le\sqrt{2\mu_c(\mathbf{Y})}$. Hence
$\mu_c(\mathbf{Y})<1/2$ suffices for (i)--(iii).
\end{enumerate}
\end{theorem}
\begin{proof}
See Appendix~\ref{app:ident}.
\end{proof}

Theorem~\ref{thm:identifiability} shows that, under transversality, the
only local ambiguity of the LRaD factorization is the orthogonal one
removed by the quotient, so that $\Bpr/\Or$ is locally a faithful
parameterization of $\Mpr$ with dimension $\dr$; item (ii) supplies the
quantitative lower bound on $D\varphi$ used in the statistical analysis.
Since the average leverage score equals $r/p$, the coherence condition
$\mu_c<1/2$ is mild when $r\ll p$ and the aggregates of
Proposition~\ref{prop:conditioning} are not concentrated on a single
node; it is the same condition that governs diagonal-plus-low-rank
decompositions in~\cite{saunderson2012diagonal} and mirrors the
rank--sparsity incoherence of~\cite{chandrasekaran2011rank}. We do not
claim global identifiability.
\end{revised}

\emph{(b) Tangent Space and Riemannian Metric:} The tangent space at $\theta=(\mathbf{Y},\mathbf{D})\in\Bpr$ is
\begin{equation}
  T_\theta\Bpr = \{(\xi_{\mathbf{Y}},\xi_{\mathbf{D}}): \xi_{\mathbf{Y}}\in\R^{p\times r},\;\xi_{\mathbf{D}}\in\mathcal{D}^p\},
\end{equation}
where $\mathcal{D}^p$ denotes the space of $p\times p$ diagonal matrices.

We equip $\Bpr$ with the direct-sum metric
\begin{equation}
  \langle\xi,\zeta\rangle^{\Bpr}_\theta
  = \tr(\xi_{\mathbf{Y}}^\top\zeta_{\mathbf{Y}})
  + \tr(\mathbf{D}^{-1}\xi_{\mathbf{D}} \mathbf{D}^{-1}\zeta_{\mathbf{D}}),
  \label{eq:metric}
\end{equation}
where the first term is the standard Euclidean metric on $\R^{p\times r}$
(adapted from \cite{zheng2025riemannian}) and the second is the
affine-invariant metric on $\Dp$ \cite{bhatia2009positive}.
\nrev{Since $\tr((\xi_{\mathbf{Y}}\mathbf{O})^\top\zeta_{\mathbf{Y}}\mathbf{O})=\tr(\xi_{\mathbf{Y}}^\top\zeta_{\mathbf{Y}})$, the metric is invariant under the action of $\Or$, so $\Bpr/\Or$ inherits a Riemannian quotient structure~\cite[Sec.~3.6.2]{absil2008optimization}.}

\emph{(c) Vertical and Horizontal Spaces:} The equivalence class $[\theta]$ is a \nrev{submanifold} of $\Bpr$, and its tangent space $T_\theta[\theta]$ is a subspace of $T_\theta\Bpr$, which is known as vertical space $\mathcal{V}_\theta$. The vertical space at $\theta$ comprises directions tangent to the orbit
$[\theta]$:
\begin{equation}
  \mathcal{V}_\theta = \{(\mathbf{Y}\Omega,\mathbf{0}):\Omega\in\mathfrak{o}_r\},
  \label{eq:vertical}
\end{equation}
where $\mathfrak{o}_r$ is the space of $r\times r$ skew-symmetric matrices. The space $\mathcal{V}_\theta$ contains directions moving along the equivalence class $[\theta]$, which should be eliminated. Thus, the vectors of interest in $T_\theta\Bpr$ lie in the orthogonal complement of $\mathcal{V}_\theta$. Its orthogonal complement under \eqref{eq:metric} is the horizontal space:
\begin{equation}
  \mathcal{H}_\theta = \{(\mathbf{Z}_1,\mathbf{Z}_2)\in T_\theta\Bpr:
    \mathbf{Y}^\top \mathbf{Z}_1 = \mathbf{Z}_1^\top \mathbf{Y},\; \mathbf{Z}_2\in\mathcal{D}^p\}.
  \label{eq:horizontal}
\end{equation}
We can eliminate the effect of equivalence classes by projecting points onto $\mathcal{H}_\theta$. The projection is defined as the following Proposition:

\begin{proposition}[Horizontal projection]
\label{prop1}
For any ambient vector $\mathbf{Z}=(\mathbf{Z}_1,\mathbf{Z}_2)\in\R^{p\times r}\times\mathcal{D}^p$,
the projection onto $\mathcal{H}_\theta$ is
\begin{equation}
  \PH_\theta(\mathbf{Z}) = (\mathbf{Z}_1 - \mathbf{Y}\Omega,\;\ddiag(\mathbf{Z}_2)),
  \label{eq:proj}
\end{equation}
where $\Omega\in\mathfrak{o}_r$ is the unique skew-symmetric solution to the
Sylvester equation
\begin{equation}
  \Omega\cdot(\mathbf{Y}^\top \mathbf{Y}) + (\mathbf{Y}^\top \mathbf{Y})\cdot\Omega = \mathbf{Y}^\top \mathbf{Z}_1 - \mathbf{Z}_1^\top \mathbf{Y}.
  \label{eq:sylvester}
\end{equation}
\end{proposition}
\begin{proof}
    See Appendix \ref{app1}.
\end{proof}

\emph{(d) Riemannian Gradient:} The objective in~\eqref{eq:DEPFM} is a sum of a data-fitting term, a sparsity penalty, and a temporal geodesic regularizer. Accordingly, for each $t$, we first collect the Euclidean gradient with respect to $\Theta_t$ as
\begin{equation}
  \Gamma_t
  :=
  \nabla_{\Theta_t}\mathcal{L}_t
  + \lambda \nabla_{\Theta_t} h
  + \mu\,\mathcal{G}_t,
  \label{eq:full_theta_gradient}
\end{equation}
where the temporal contribution is
\begin{equation}
  \mathcal{G}_t
  =
  \begin{cases}
    \nabla_{\Theta_1} d_{\mathcal S}^2(\Theta_1,\Theta_2), & t=1,\\[1mm]
    \nabla_{\Theta_t} d_{\mathcal S}^2(\Theta_t,\Theta_{t+1})
    +
    \nabla_{\widetilde\Theta_t} d_{\mathcal S}^2(\Theta_{t-1},\Theta_t), & 1<t<T,\\[1mm]
    \nabla_{\widetilde\Theta_T} d_{\mathcal S}^2(\Theta_{T-1},\Theta_T), & t=T.
  \end{cases}
  \label{eq:temporal_block_gradient}
\end{equation}
Here $\nabla_{\Theta_t}$ and $\nabla_{\widetilde\Theta_t}$ denote derivatives with respect to the first and second arguments of the geodesic distance, respectively.

Under the metric~\eqref{eq:metric}, the Riemannian gradient is characterized in the following proposition.
\begin{proposition}[Riemannian gradient formula\footnote{This expression is horizontal by construction, so no additional gauge-fixing step is required.}]
\label{prop:riem_grad}
Let $f:\Spp\to\R$ be smooth with Euclidean gradient $\nabla_\Theta f$.
Define $\tilde{f}=f\circ\varphi:\Bpr\to\R$.
The Euclidean gradient $\nabla_\theta\tilde{f}=(G_{\mathbf{Y}},G_{\mathbf{D}})$ is
\begin{equation}
  G_{\mathbf{Y}} = 2\nabla_\Theta f\cdot \mathbf{Y} \in\R^{p\times r},
  \qquad
  G_{\mathbf{D}} = \ddiag(\nabla_\Theta f)\in\mathcal{D}^p.
  \label{eq:eucl_grad}
\end{equation}
The Riemannian gradient with respect to metric \eqref{eq:metric} is
\begin{equation}
  \grad\tilde{f}(\theta) = \bigl(G_{\mathbf{Y}},\; \mathbf{D}\,G_{\mathbf{D}}\,\mathbf{D}\bigr),
  \label{eq:riem_grad}
\end{equation}
and $\grad\tilde{f}(\theta)\in\mathcal{H}_\theta$ automatically.
\end{proposition}
\begin{proof}
    See Appendix \ref{app2}.
\end{proof}

\emph{(e) Gradients of Individual Terms:} From~\eqref{eq:likelihood_grad}, the Euclidean gradient of $\mathcal{L}_t$ with
respect to $\Theta_t$ is
\begin{equation}
  \nabla_{\Theta_t}\mathcal{L}_t
  = -\frac{1}{2}\Theta_t^{-1}
    + \frac{1}{2n_t}\sum_{i=1}^{n_t}
      u_t\!\left(\mathbf{x}_{t,i}^\top\Theta_t\,\mathbf{x}_{t,i}\right)
      \mathbf{x}_{t,i}\mathbf{x}_{t,i}^\top.
  \label{eq:grad_likelihood_Theta}
\end{equation}
The quadratic form $\mathbf{x}_{t,i}^\top\Theta_t\,\mathbf{x}_{t,i}$ exploits the LRaD structure
directly:
\begin{equation}
  \mathbf{x}_{t,i}^\top\Theta_t\,\mathbf{x}_{t,i}
  = \norm{\mathbf{Y}_t^\top \mathbf{x}_{t,i}}^2 + \mathbf{x}_{t,i}^\top \mathbf{D}_t\,\mathbf{x}_{t,i},
  \label{eq:quad_form}
\end{equation}
which costs $\mathcal{O}(pr)$ instead of $\mathcal{O}(p^2)$.

From~\eqref{eq:penalty}, $\nabla_{\Theta_t}h$ has entries
\begin{equation}
  \bigl[\nabla_{\Theta_t}h\bigr]_{q\ell}
  = \begin{cases}
    \tanh\!\left([\Theta_t]_{q\ell}/\varepsilon\right) & q\neq\ell, \\
    0 & q=\ell.
  \end{cases}
  \label{eq:grad_penalty}
\end{equation}

The temporal regularizer is based on the affine-invariant geodesic distance on $\mathcal{S}_{++}^p$. We have the following Proposition:
\begin{revised}
\begin{proposition}[Gradient of the squared affine-invariant distance and its exact evaluation]
\label{prop:geo_grad}
Let $\mathbf{A},\mathbf{B}\in\Spp$ and
$\mathbf{M}=\mathbf{A}^{-1/2}\mathbf{B}\mathbf{A}^{-1/2}$.
The Euclidean (Frobenius) gradients of $\dS^2$ on $\mathcal{S}^p$ are
\begin{align}
  \nabla_{\mathbf{A}}\,\dS^2(\mathbf{A},\mathbf{B})
  &= -2\,\mathbf{A}^{-1/2}\log(\mathbf{M})\,\mathbf{A}^{-1/2},
  \label{eq:geo_grad_1} \\
  \nabla_{\mathbf{B}}\,\dS^2(\mathbf{A},\mathbf{B})
  &= -2\,\mathbf{B}^{-1/2}\log\!\left(\mathbf{B}^{-1/2}\mathbf{A}\,\mathbf{B}^{-1/2}\right)\mathbf{B}^{-1/2}.
  \label{eq:geo_grad_2}
\end{align}
Moreover, let $\mathbf{B}\mathbf{V}=\mathbf{A}\mathbf{V}\Lambda$,
$\mathbf{V}^\top\mathbf{A}\mathbf{V}=\mathbf{I}_p$, be the symmetric-definite
generalized eigendecomposition of the pencil $(\mathbf{B},\mathbf{A})$,
with $\Lambda=\diag(\lambda_1,\ldots,\lambda_p)\succ\mathbf{0}$. Then
\begin{equation}
\begin{aligned}
  \dS^2(\mathbf{A},\mathbf{B})&=\textstyle\sum_{j}\log^2\lambda_j,\
  \nabla_{\mathbf{A}}\dS^2=-2\,\mathbf{V}\log(\Lambda)\mathbf{V}^\top,\\
  \nabla_{\mathbf{B}}\dS^2&=2\,\mathbf{V}\Lambda^{-1}\log(\Lambda)\mathbf{V}^\top.
\end{aligned}
\label{eq:geo_grad_eig}
\end{equation}
Equivalently, the affine-invariant Riemannian gradient of
$\frac12\dS^2(\cdot,\mathbf{B})$ at $\mathbf{A}$ is
$\mathbf{A}\,\nabla_{\mathbf{A}}(\tfrac12\dS^2)\,\mathbf{A}=-\Log_{\mathbf{A}}(\mathbf{B})$,
where $\Log_{\mathbf{A}}(\mathbf{B})=\mathbf{A}^{1/2}\log(\mathbf{M})\mathbf{A}^{1/2}$
is the Riemannian logarithm.
\end{proposition}
\begin{proof}
See Appendix~\ref{app3}.
\end{proof}

A scalar check confirms~\eqref{eq:geo_grad_1}--\eqref{eq:geo_grad_2}: for
$p=1$, $\dS^2(a,b)=\log^2(b/a)$, and
$\partial_a\dS^2=-\tfrac{2}{a}\log(b/a)$,
$\partial_b\dS^2=\tfrac{2}{b}\log(b/a)$. All quantities that depend on $\mathcal{G}_t$
in~\eqref{eq:full_theta_gradient}, and hence
Algorithm~\ref{alg:depfm}, are evaluated with~\eqref{eq:geo_grad_eig}.
For the pair $(\Theta_t,\Theta_{t+1})$ with generalized eigenpairs
$(\Lambda_t,\mathbf{V}_t)$ of $(\Theta_{t+1},\Theta_t)$, the temporal block
gradient~\eqref{eq:temporal_block_gradient} reads
\begin{equation}
  \mathcal{G}_t=-2\,\mathbf{V}_t\log(\Lambda_t)\mathbf{V}_t^\top
  +2\,\mathbf{V}_{t-1}\Lambda_{t-1}^{-1}\log(\Lambda_{t-1})\mathbf{V}_{t-1}^\top,
  \label{eq:Gt_eig}
\end{equation}
with the first (second) term absent for $t=T$ ($t=1$).
\end{revised}

\emph{(f) Retraction:} To update the iterates while staying on $\Bpr$, we employ the second-order retraction. Following \cite{zhang2025graph}, we define the second-order retraction on
$\Bpr$:
\begin{equation}
  \Retr_\theta(\xi) =
  \left(\mathbf{Y}+\xi_{\mathbf{Y}},\;\;
    \mathbf{D} + \xi_{\mathbf{D}} + \tfrac{1}{2}\xi_{\mathbf{D}} \mathbf{D}^{-1}\xi_{\mathbf{D}}\right).
  \label{eq:retraction}
\end{equation}
This satisfies $\Retr_\theta(0)=\theta$ and
$D\Retr_\theta(0)=\mathrm{id}$ by direct inspection.
Positive definiteness of the diagonal component holds
\emph{unconditionally}: setting
$\mathbf{E}:=\mathbf{D}^{-1/2}\xi_{\mathbf{D}}\mathbf{D}^{-1/2}$
and using the fact that $\mathbf{D}$ is diagonal (hence commutes with
$\xi_{\mathbf{D}}$), one verifies $
  \mathbf{D}+\xi_{\mathbf{D}}+\tfrac{1}{2}\xi_{\mathbf{D}} \mathbf{D}^{-1}\xi_{\mathbf{D}}
  \;=\;
  \mathbf{D}^{1/2}\!\left(\mathbf{I}+\mathbf{E}+\tfrac{1}{2}\mathbf{E}^{2}\right)\!\mathbf{D}^{1/2}$. Since $\mathbf{D}\succ\mathbf{0}$, it suffices to show
$\mathbf{I}+\mathbf{E}+\frac{1}{2}\mathbf{E}^{2}\succ\mathbf{0}$.
Because $\mathbf{E}$ is diagonal with entries $e_{i}$, this reduces
to the scalar inequality $
  1 + e_{i} + \tfrac{1}{2}e_{i}^{2}
  \;=\;
  \tfrac{1}{2}\!\left[(e_{i}+1)^{2}+1\right]
  \;\geq\; \tfrac{1}{2} \;>\; 0,
  \; \forall\, e_{i}\in\mathbb{R}$, which holds for all $\xi_{\mathbf{D}}$ without any
step-size restriction.
\begin{nrevised}
In contrast, the first component $\mathbf{Y}+\xi_{\mathbf{Y}}$ may leave
$\Rpr$ (e.g., $\xi_{\mathbf{Y}}=-\mathbf{Y}$ gives $\mathbf{0}$), so
\eqref{eq:retraction} is a retraction on $\Bpr$ only in the local sense
of~\cite[Sec.~4.1]{absil2008optimization}, namely on the open set
\begin{equation}
\begin{aligned}
  \mathcal{O}_\theta&:=\{\xi\in T_\theta\Bpr:\ \mathbf{Y}+\xi_{\mathbf{Y}}\in\Rpr\}\\
  &\ \supseteq\{\xi\in T_\theta\Bpr:\ \|\xi_{\mathbf{Y}}\|_{\op}<\sigma_r(\mathbf{Y})\},
\end{aligned}
  \label{eq:retr_domain}
\end{equation}
where the inclusion follows from Weyl's inequality
$\sigma_r(\mathbf{Y}+\xi_{\mathbf{Y}})\ge\sigma_r(\mathbf{Y})-\|\xi_{\mathbf{Y}}\|_{\op}$.
Algorithm~\ref{alg:depfm} accepts only trial steps that pass the rank
test~\eqref{eq:rank_test} below, and the convergence analysis uses
\eqref{eq:retraction} only on a uniform neighborhood of the zero section;
a globally defined retraction is not required.
\end{nrevised}

\emph{(g) Vector Transport:} For the Riemannian conjugate-gradient method, we transport tangent vectors by horizontal projection. Given $\theta_1,\theta_2\in\Bpr$ and $\xi_{\theta_1}\in\mathcal{H}_{\theta_1}$,
the vector transport to $\mathcal{H}_{\theta_2}$ is
\begin{equation}
  \mathcal{T}^{\Bpr}_{\theta_1\to\theta_2}(\xi_{\theta_1})
  = \PH_{\theta_2}(\xi_{\theta_1}),
  \label{eq:transport}
\end{equation}
where $\PH_{\theta_2}$ is defined by Proposition~\ref{prop1}.

\subsection{Algorithm Development}
\label{sec:algorithm}

\emph{(a) Full Objective on the Product Manifold:} The $T$ parameter pairs $\{\theta_t\}$ live on the product manifold
$\mathcal{M}=\prod_{t=1}^T\Bpr$ equipped with the product metric.
The total objective $F:\mathcal{M}\to\R$ is~\eqref{eq:DEPFM}.
The Riemannian gradient of $F$ with respect to $\theta_t$ is
\begin{align}
  \grad_{\theta_t}F
  &= \grad f_t(\theta_t)
    + \mu\,\grad_{\theta_t}c_{t-1}(\theta_{t-1},\theta_t)\nonumber\\
    &\quad+ \mu\,\grad_{\theta_t}c_t(\theta_t,\theta_{t+1}),
  \label{eq:product_grad}
\end{align}
with boundary conventions $c_0=c_T\equiv 0$.
Each summand is computed via Proposition~\ref{prop:riem_grad} applied to the
corresponding Euclidean gradient: $\nabla_{\Theta_t}f_t$ from
\eqref{eq:grad_likelihood_Theta}--\eqref{eq:grad_penalty}, and
$\nabla_{\Theta_t}c_{t-1}$, $\nabla_{\Theta_t}c_t$ from
Proposition~\ref{prop:geo_grad}. \nrev{Equivalently, $\grad_{\theta_t}F=(2\Gamma_t\mathbf{Y}_t,\ \mathbf{D}_t\ddiag(\Gamma_t)\mathbf{D}_t)$ with $\Gamma_t$ from~\eqref{eq:full_theta_gradient} and~\eqref{eq:Gt_eig}.}

\emph{(b) Efficient Computation via Woodbury Identity:} All computations involving $\Theta_t^{-1}$ \nrev{in the likelihood gradient} exploit the Woodbury identity:
\begin{align}
  \Theta_t^{-1}
  &= (\mathbf{Y}_t \mathbf{Y}_t^\top + \mathbf{D}_t)^{-1}\nonumber\\
  &= \mathbf{D}_t^{-1}
    - \mathbf{D}_t^{-1}\mathbf{Y}_t\!\left(I_r + \mathbf{Y}_t^\top \mathbf{D}_t^{-1}\mathbf{Y}_t\right)^{-1}\!\mathbf{Y}_t^\top \mathbf{D}_t^{-1},
  \label{eq:woodbury}
\end{align}
\nrev{so that the products $\Theta_t^{-1}\mathbf{Y}_t$ and $\ddiag(\Theta_t^{-1})$ needed in $2\Gamma_t\mathbf{Y}_t$ and $\ddiag(\Gamma_t)$ cost $\mathcal{O}(pr^2+r^3)$ without forming $\Theta_t^{-1}$.}
The quadratic form~\eqref{eq:quad_form} costs only $\mathcal{O}(pr)$ directly.

\begin{revised}
\begin{remark}[Why the temporal term does not inherit the LRaD structure]
\label{rem:no_lowrank_log}
One might hope to evaluate $\log(\Theta_t^{-1/2}\Theta_{t+1}\Theta_t^{-1/2})$
in $\mathcal{O}(pr^2)$ by exploiting
$\Theta_{t+1}-\Theta_t=(\mathbf{Y}_{t+1}\mathbf{Y}_{t+1}^\top-\mathbf{Y}_t\mathbf{Y}_t^\top)+(\mathbf{D}_{t+1}-\mathbf{D}_t)$.
This is not possible in general: although $\mathbf{D}_{t+1}-\mathbf{D}_t$ is
diagonal, its congruence $\Theta_t^{-1/2}(\mathbf{D}_{t+1}-\mathbf{D}_t)\Theta_t^{-1/2}$
is a dense full-rank matrix, and a rank-$r$ perturbation of
$\mathbf{D}_t$ moves all $p$ eigenvalues (subject only to interlacing),
not $r$ of them. The temporal block therefore requires a dense
symmetric-definite generalized eigendecomposition per pair of windows,
computed by a Cholesky factorization of $\Theta_t$ followed by a
symmetric eigendecomposition, in $\mathcal{O}(p^3)$ operations. Since only
$\mathcal{G}_t\mathbf{Y}_t$ and $\ddiag(\mathcal{G}_t)$ are needed, Krylov
approximations of matrix-function--vector products~\cite{higham2008functions}
are a natural route to large $p$; their analysis (inexact gradients) is
left for future work and no such approximation is used in this paper.
\end{remark}
\end{revised}

\emph{(c) \nrev{Safeguarded} Riemannian Conjugate-Gradient Solver:}
With the gradient, retraction, and vector transport defined above, we solve the optimization problem by a \nrev{safeguarded} Riemannian conjugate-gradient scheme. \nrev{Let $\theta^{(k)}=(\theta_1^{(k)},\ldots,\theta_T^{(k)})$, $g^{(k)}:=\grad F(\theta^{(k)})$, and let $\mathcal{T}_k$ denote the blockwise transport~\eqref{eq:transport} from $\theta^{(k-1)}$ to $\theta^{(k)}$. With $y^{(k)}:=g^{(k)}-\mathcal{T}_k(g^{(k-1)})$, the Hestenes--Stiefel candidate direction is}
\begin{equation}
\nrev{\begin{aligned}
  \tilde\eta^{(k)}
  &=
  -g^{(k)}+\beta^{(k)}\mathcal{T}_k(\eta^{(k-1)}),\\
  \beta^{(k)}&=\frac{\langle g^{(k)},y^{(k)}\rangle_{\theta^{(k)}}}{\langle \mathcal{T}_k(\eta^{(k-1)}),y^{(k)}\rangle_{\theta^{(k)}}},
\end{aligned}}
  \label{eq:rcg_direction_update}
\end{equation}
\nrev{where all inner products are taken at the same point $\theta^{(k)}$ (and $\beta^{(k)}:=0$ if the denominator vanishes). The direction is then safeguarded: for fixed constants $0<c_0\le 1\le C_\eta$,
\begin{equation}
  \eta^{(k)}=\begin{cases}
  \tilde\eta^{(k)}, & \text{if }\langle g^{(k)},\tilde\eta^{(k)}\rangle\le-c_0\|g^{(k)}\|^2\\
   & \text{and }\|\tilde\eta^{(k)}\|\le C_\eta\|g^{(k)}\|,\\
  -g^{(k)}, & \text{otherwise (restart)}.
  \end{cases}
  \label{eq:safeguard}
\end{equation}
Write the search direction as
$\eta^{(k)}=(\eta_t^{(k)})_{t=1}^T$,
where
$\eta_t^{(k)}
=(\eta_{\mathbf{Y},t}^{(k)},
  \eta_{\mathbf{D},t}^{(k)})$
denotes the tangent direction associated with
$(\mathbf{Y}_t^{(k)},\mathbf{D}_t^{(k)})$.

}
\begin{nrevised}
The step size is computed by a \emph{rank-preserving Armijo
backtracking}. Fix constants $c_1,\beta,\gamma_\sigma\in(0,1)$ and
$0<s_{\min}\le s_{\max}$, choose an initial trial
$s_k\in[s_{\min},s_{\max}]$ (e.g.,
$s_k=\min\{s_{\max},\max\{s_{\min},\alpha_{k-1}/\beta\}\}$), and set
$\alpha_k:=s_k\beta^{j_k}$, where $j_k$ is the smallest integer $j\ge0$
such that $\alpha=s_k\beta^{j}$ satisfies
\begin{align}
&\sigma_r\bigl(\mathbf{Y}_t^{(k)}+\alpha\,\eta_{\mathbf{Y},t}^{(k)}\bigr)\ge\gamma_\sigma\,\sigma_r\bigl(\mathbf{Y}_t^{(k)}\bigr),\quad t=1,\ldots,T,
\label{eq:rank_test}\\
&F\bigl(\Retr_{\theta^{(k)}}(\alpha\eta^{(k)})\bigr)\le F(\theta^{(k)})+c_1\alpha\,\langle g^{(k)},\eta^{(k)}\rangle_{\theta^{(k)}}.
\label{eq:armijo}
\end{align}
The rank test~\eqref{eq:rank_test} keeps every accepted point in the
domain~\eqref{eq:retr_domain} of the retraction, and~\eqref{eq:armijo}
is the Armijo sufficient-decrease condition~\cite{armijo1966minimization}.
By Theorem~\ref{thm:convergence}, $j_k$ is finite and uniformly bounded.
The iterate is then updated as
\end{nrevised}
\begin{equation}
  \theta^{(k+1)}
  =
  \Retr_{\theta^{(k)}}\!\left(\alpha_k\eta^{(k)}\right).
  \label{eq:rcg_retraction_step}
\end{equation}
Since all search directions are maintained in the horizontal space, the orthogonal non-identifiability of $\mathbf{Y}_t$ is handled implicitly throughout the optimization.

The DEGFM RCG algorithm is presented as Algorithm~\ref{alg:depfm}.

\begin{algorithm}[!htbp]
\DontPrintSemicolon
\caption{DEGFM: Safeguarded Riemannian Conjugate Gradient \nrev{with Armijo Backtracking}}
\label{alg:depfm}
\KwIn{\nrev{Data $\{\mathbf{x}_{t,i}\}$; parameters
$\lambda,\mu,r,\varepsilon$; tolerance $\epsilon_{\rm tol}$;
threshold $\tau_{\rm thr}$; initialization floor
$\epsilon_{\rm init}>0$; Armijo constants
$c_1,\beta\in(0,1)$, $0<s_{\min}\le s_{\max}$;
rank safeguard $\gamma_\sigma\in(0,1)$;
safeguards $0<c_0\le1\le C_\eta$}}
\KwOut{Graph sequence $\{A_t\}_{t=1}^T$}
\textbf{Initialize}: \nrev{for $t=1,\ldots,T$, set $
\mathbf{S}_t=\frac{1}{n_t}\sum_{i=1}^{n_t}
\mathbf{x}_{t,i}\mathbf{x}_{t,i}^{\top}$,
and let $\mathbf{U}_{t,r}\in\mathbb{R}^{p\times r}$ contain any $r$
orthonormal eigenvectors of $\mathbf{S}_t$ associated with its $r$
largest eigenvalues $\{s_{t,j}\}_{j=1}^{r}$. Define $
\mathbf{Y}_t^{(0)}
=
\mathbf{U}_{t,r}
\diag\!\left(
\sqrt{\max\{s_{t,1},\epsilon_{\rm init}\}},
\ldots,
\sqrt{\max\{s_{t,r},\epsilon_{\rm init}\}}
\right)$,
and $
\mathbf{D}_t^{(0)}
=
\diag(\mathbf{S}_t)+\epsilon_{\rm init}\mathbf{I}_p$.
Then $\mathbf{Y}_t^{(0)}\in\mathcal{R}_{p,r}$ and
$\mathbf{D}_t^{(0)}\in\mathcal{D}_{++}^{p}$ by construction}\;
\For{$k=0,1,2,\ldots$}{
  \For{$t=1,\ldots,T$}{
    Form $\Theta_t^{(k)}=\mathbf{Y}_t^{(k)}\mathbf{Y}_t^{(k)\top}+\mathbf{D}_t^{(k)}$\tcp*{$\mathcal{O}(p^2r)$}
    Compute $\mathbf{x}_{t,i}^\top\Theta_t^{(k)}\mathbf{x}_{t,i}$ via \eqref{eq:quad_form} and the weights $u_t(\cdot)$\tcp*{$\mathcal{O}(n_tpr)$}
    Compute $\nabla_{\Theta_t}f_t$ via \eqref{eq:grad_ft}, using \eqref{eq:woodbury} for $\Theta_t^{-1}$\;
  }
  \For{$t=1,\ldots,T-1$}{
    Generalized eigendecomposition $\Theta_{t+1}^{(k)}\mathbf{V}_t=\Theta_t^{(k)}\mathbf{V}_t\Lambda_t$, $\mathbf{V}_t^\top\Theta_t^{(k)}\mathbf{V}_t=\mathbf{I}$\tcp*{$\mathcal{O}(p^3)$}
  }
  \For{$t=1,\ldots,T$}{
    Assemble $\Gamma_t$ via \eqref{eq:full_theta_gradient} and \eqref{eq:Gt_eig}\;
    $g_t^{(k)}=\bigl(2\Gamma_t\mathbf{Y}_t^{(k)},\ \mathbf{D}_t^{(k)}\ddiag(\Gamma_t)\mathbf{D}_t^{(k)}\bigr)$ via \eqref{eq:riem_grad}\;
  }
  \lIf{$\|g^{(k)}\|_{\theta^{(k)}}<\epsilon_{\rm tol}$}{\textbf{break}}
  \lIf{$k=0$}{$\tilde\eta^{(0)}=-g^{(0)}$}
  \Else{
    Transport $\mathcal{T}_k(\eta^{(k-1)})$ and $\mathcal{T}_k(g^{(k-1)})$ blockwise via \eqref{eq:transport}\;
    Compute $\beta^{(k)}$ and $\tilde\eta^{(k)}$ via \eqref{eq:rcg_direction_update}\;
  }
  Set $\eta^{(k)}$ by the safeguard \eqref{eq:safeguard}\;
  \nrev{Choose $s_k\in[s_{\min},s_{\max}]$ and set $\alpha\leftarrow s_k$}\;
  \While{\nrev{\eqref{eq:rank_test} or \eqref{eq:armijo} is violated at $\alpha$}}{
    \nrev{$\alpha\leftarrow\beta\alpha$}\tcp*{\nrev{Armijo backtracking}}
  }
  \nrev{$\alpha_k\leftarrow\alpha$}\;
  \For{$t=1,\ldots,T$}{
    Update: $\theta_t^{(k+1)}=\Retr_{\theta_t^{(k)}}(\alpha_k\eta_t^{(k)})$ via \eqref{eq:retraction}\;
  }
}
Compute $\hat\Theta_t=\varphi(\hat\theta_t)=\hat{\mathbf{Y}}_t\hat{\mathbf{Y}}_t^\top+\hat{\mathbf{D}}_t$;
form partial correlations $\hat\varrho_{t,ij}
  = -[\hat\Theta_t]_{ij}/\sqrt{[\hat\Theta_t]_{ii}[\hat\Theta_t]_{jj}}$\;
\Return $A_t[i,j]=\mathbf{1}\left[\abs{\hat\varrho_{t,ij}}\geq\tau_{\rm thr}\right]$
\end{algorithm}

\subsection{Complexity Analysis}
\label{complexity}

\begin{revised}
We count the cost of one gradient evaluation, which dominates one
iteration of Algorithm~\ref{alg:depfm}.
\emph{Likelihood.} The quadratic forms~\eqref{eq:quad_form} cost
$\mathcal{O}(n_tpr)$; the data part of $2\nabla\mathcal{L}_t\mathbf{Y}_t$ is
$\mathbf{X}_t^\top\diag(u_{t,\cdot})(\mathbf{X}_t\mathbf{Y}_t)/n_t$, which
costs $\mathcal{O}(n_tpr)$; the term $\Theta_t^{-1}\mathbf{Y}_t$ costs
$\mathcal{O}(pr^2+r^3)$ by~\eqref{eq:woodbury}; and the diagonal parts
cost $\mathcal{O}(n_tp+pr^2)$.
\emph{Sparsity penalty.} Because $\nabla h$ depends on every entry of
$\Theta_t$, forming $\Theta_t$ and $\nabla h(\Theta_t)\mathbf{Y}_t$ costs
$\mathcal{O}(p^2r)$.
\emph{Temporal term.} The generalized eigendecomposition of each pair
costs $\mathcal{O}(p^3)$; then
$\mathcal{G}_t\mathbf{Y}_t$ costs $\mathcal{O}(p^2r)$ (via
$\mathbf{V}_t^\top\mathbf{Y}_t$), and $\ddiag(\mathcal{G}_t)$ costs
$\mathcal{O}(p^2)$ since
$[\mathbf{V}\log(\Lambda)\mathbf{V}^\top]_{ii}=\sum_j\log(\lambda_j)V_{ij}^2$.
\emph{Geometry.} The Sylvester equation~\eqref{eq:sylvester} costs
$\mathcal{O}(pr^2+r^3)$, and the retraction and inner products cost
$\mathcal{O}(pr+p)$.
Each \nrev{trial step of the backtracking requires one objective
evaluation, which needs the} generalized
eigenvalues (but not eigenvectors) of each pair and therefore also costs
$\mathcal{O}(Tp^3)$ in the worst case\nrev{, plus the rank
test~\eqref{eq:rank_test}, which costs $\mathcal{O}(Tpr^2)$ through the
$r\times r$ Gram matrices; by Theorem~\ref{thm:convergence}, the number of
trials per iteration is bounded by a constant $1+\bar\jmath$}. In total,
\begin{equation}
\begin{aligned}
  \text{cost per gradient evaluation}
  &=\mathcal{O}\!\left(T\left(n_{\max}pr+p^2r+p^3\right)\right)\\
  &=\mathcal{O}\!\left(T\left(n_{\max}pr+p^3\right)\right),
\end{aligned}
  \label{eq:complexity}
\end{equation}
with memory $\mathcal{O}(Tp^2)$. For $\mu=0$ (independent windows), the
$p^3$ term disappears and the cost reduces to
$\mathcal{O}(T(n_{\max}pr+p^2r))$.
The $\mathcal{O}(Tp^3)$ term is of the same order as the per-iteration
cost of ADMM solvers for TVGL, which require an eigendecomposition per
window. The advantage of \textsc{Degfm} is therefore primarily
statistical (an intrinsic dimension $\dr$ instead of $p(p+1)/2$ and the
temporal pooling quantified in Section~\ref{sec:theory}), not a reduction
of the asymptotic arithmetic cost. The exact affine-invariant computation
is the current scalability bottleneck (Remark~\ref{rem:no_lowrank_log}).
\end{revised}
\section{Theoretical Analysis}
\label{sec:theory}
\begin{revised}
This section establishes optimization and statistical guarantees.
Subsection~\ref{subsect_conv} shows that the safeguarded solver\nrev{,
with a rank-preserving Armijo backtracking, is well defined and}
converges to critical points. Subsections~\ref{subsec:stat_setting}--\ref{subsec:temporal}
analyze a \emph{localized} estimator, i.e., a minimizer of~\eqref{eq:DEPFM}
over a neighborhood of the truth, and quantify how temporal coupling
trades variance for bias\nrev{; the local curvature condition on which
this analysis rests is proved for Student-$t$ observations
(Proposition~\ref{prop:t_rsc})}. Subsection~\ref{subsec:misspec} treats non-LRaD
truths and the choice of $r$, and Subsection~\ref{subsec:support}
derives thresholded graph recovery. We do not claim a global landscape
result: the statistical guarantees concern points in a neighborhood of
the truth, and they are connected to the algorithm through the fact that
such points are local minimizers (hence critical points) of~\eqref{eq:DEPFM}.
\end{revised}

\subsection{Convergence of the \nrev{Safeguarded} Riemannian CG Algorithm}
\label{subsect_conv}
\begin{nrevised}
Let $\mathfrak{B}:=(\Bpr)^T$ with the product metric, let
$\{\theta^{(k)}\}$ be the iterates of Algorithm~\ref{alg:depfm}, and let
$\phi_k(\alpha):=F(\Retr_{\theta^{(k)}}(\alpha\eta^{(k)}))$. Since
$D\Retr_\theta(0)=\mathrm{id}$, $\phi_k'(0)=\langle g^{(k)},\eta^{(k)}\rangle_{\theta^{(k)}}$.
Two facts shape the analysis. First, $\mathbf{Y}\mathbf{Y}^\top+\mathbf{D}\succ\mathbf{0}$
for \emph{every} $\mathbf{Y}\in\R^{p\times r}$ when $\mathbf{D}\in\Dp$, so
$F$ is defined and smooth on the larger open set
$\mathfrak{E}:=(\R^{p\times r}\times\Dp)^T\supset\mathfrak{B}$, and $\phi_k$
is defined for all $\alpha\ge0$. Second, the quotient structure, the
projection~\eqref{eq:proj}, and hence the transport~\eqref{eq:transport},
require $\mathbf{Y}_t\in\Rpr$, which the retraction preserves only on the
domain~\eqref{eq:retr_domain}. The iterates must therefore remain in
$\mathfrak{B}$, and the analysis may use the retraction only near the
zero section. 

\begin{assumption}[Regularity and compact sublevel set]
\label{ass:smooth}
(a) Each $\rho_t$ extends to a $C^2$ function on an open interval
containing $[0,\infty)$.
(b) The sublevel set $\mathcal{S}_0:=\{\theta\in\mathfrak{B}:F(\theta)\le F(\theta^{(0)})\}$
is compact.
\end{assumption}

Assumption~\ref{ass:smooth}(b) involves only $F$ and $\theta^{(0)}$; it
refers neither to step sizes nor to search curves. Since $\Rpr$ and $\Dp$
are open, it is not automatic: it excludes $\sigma_r(\mathbf{Y}_t)\to0$,
$[\mathbf{D}_t]_{ii}\to0$, and $\|\Theta_t\|\to\infty$ on $\mathcal{S}_0$. This is a standard level-set regularity assumption for convergence
analysis on open manifolds; we do not claim that it follows
automatically from the objective in~\eqref{eq:DEPFM}.
By compactness and continuity, the constants
\begin{equation}
\begin{aligned}
&s_0:=\min_{\mathcal{S}_0}\min_t\sigma_r(\mathbf{Y}_t),\qquad
k_d:=\min_{\mathcal{S}_0}\min_{t,i}[\mathbf{D}_t]_{ii},\\
&K_d:=\max_{\mathcal{S}_0}\max_{t,i}[\mathbf{D}_t]_{ii},\qquad
Y_{\max}:=\max_{\mathcal{S}_0}\max_t\|\mathbf{Y}_t\|_F,\\
&G:=\max_{\mathcal{S}_0}\|\grad F(\theta)\|_\theta
\end{aligned}
\label{eq:conv_consts}
\end{equation}
are finite, and $s_0,k_d>0$. Let
\begin{equation}
\begin{aligned}
\mathcal{K}:=\Bigl\{\theta:\ &\sigma_r(\mathbf{Y}_t)\ge\gamma_\sigma s_0,\ \|\mathbf{Y}_t\|_F\le Y_{\max}+1,\\
&\tfrac{k_d}{2}\le[\mathbf{D}_t]_{ii}\le\tfrac{5K_d}{2}\ \ \forall\,t,i\Bigr\}\subset\mathfrak{B},
\end{aligned}
\label{eq:tube}
\end{equation}
which is compact, let $C_H:=\max_{\mathcal{K}}\|\nabla^2F\|_{\op}$ and
$C_G:=\max_{\mathcal{K}}\|\nabla F\|_F$ be the Euclidean derivatives of $F$
in the coordinates $(\mathbf{Y}_t,\mathbf{D}_t)_t$, and set
\begin{equation}
\begin{aligned}
&L_R:=C_H\max\{1,2K_d\}^2+C_GK_d,\\
&\bar\alpha:=\frac{\min\{(1-\gamma_\sigma)s_0,\,1\}}{C_\eta G},\quad
\bar\alpha_A:=\min\Bigl\{\bar\alpha,\ \frac{2(1-c_1)c_0}{L_RC_\eta^2}\Bigr\},\\
&\alpha_{\rm low}:=\min\{s_{\min},\,\beta\bar\alpha_A\}.
\end{aligned}
\label{eq:conv_consts2}
\end{equation}

\begin{theorem}[Well-posedness and convergence]
\label{thm:convergence}
Let Assumption~\ref{ass:smooth} hold and let $G>0$ (if $G=0$, then
$g^{(0)}=0$ and the algorithm stops at a critical point). Then:
\begin{enumerate}
\item[(i)] (Well-posedness) At every iteration with $g^{(k)}\neq0$, the
backtracking terminates after
$j_k\le\bar\jmath:=\max\{0,\lceil\log(s_{\max}/\bar\alpha_A)/\log(1/\beta)\rceil\}$
reductions, and $\alpha_k\ge\alpha_{\rm low}$. All iterates lie in
$\mathcal{S}_0\subset\mathfrak{B}$; in particular
$\sigma_r(\mathbf{Y}_t^{(k)})\ge s_0$, so the Sylvester
equation~\eqref{eq:sylvester} is uniquely solvable and the
transport~\eqref{eq:transport} is well defined at every iterate.
\item[(ii)] (Sufficient decrease)
$F(\theta^{(k+1)})\le F(\theta^{(k)})-c_1c_0\,\alpha_{\rm low}\|g^{(k)}\|_{\theta^{(k)}}^2$.
\item[(iii)] (Convergence and rate) With $F_{\rm low}:=\min_{\mathcal{S}_0}F$,
\begin{equation}
  \sum_{k=0}^\infty\|g^{(k)}\|_{\theta^{(k)}}^2
  \le\frac{F(\theta^{(0)})-F_{\rm low}}{c_1c_0\,\alpha_{\rm low}}.
  \label{eq:sum_grad}
\end{equation}
Consequently, $\lim_{k\to\infty}\|g^{(k)}\|_{\theta^{(k)}}=0$, every
accumulation point of $\{\theta^{(k)}\}$ (at least one exists) is a
critical point of $F$, and
\begin{equation}
  \min_{0\le k<K}\|g^{(k)}\|_{\theta^{(k)}}^2
  \le\frac{F(\theta^{(0)})-F_{\rm low}}{c_1c_0\,\alpha_{\rm low}\,K},
  \label{eq:conv_rate}
\end{equation}
which is achieved with at most $K(1+\bar\jmath)$ evaluations of $F$.
\end{enumerate}
\end{theorem}
\begin{proof}
See Appendix~\ref{app4}.
\end{proof}

\emph{Discussion.} The safeguard~\eqref{eq:safeguard} supplies a uniform
angle and length condition, and the Armijo rule then yields a sufficient
decrease with step sizes bounded away from zero; this is the classical
mechanism of~\cite[Thm.~4.3.1]{absil2008optimization}
and~\cite{boumal2019global,boumal2023introduction}. What is specific to
$\Bpr$ is that the $\mathbf{Y}$-block of~\eqref{eq:retraction} is only
locally a retraction. The proof therefore ties all constants to
$\mathcal{S}_0$ and to the explicit compact tube~\eqref{eq:tube}, on which
only steps $\alpha\le\bar\alpha$ are used; since no assumption involves
the step sizes or the search curves, the existence of an admissible step
and the feasibility of the search curve are established independently
rather than presupposed. Three further remarks are in order.
(a)~Under Assumption~\ref{ass:smooth}(b), $F\ge F(\theta^{(0)})$ on the
rank-deficient set $\mathfrak{E}\setminus\mathfrak{B}$
(Appendix~\ref{app4}, Step~0). Since accepted steps strictly decrease
$F$, the Armijo condition alone already prevents the iterates from
reaching that set. The test~\eqref{eq:rank_test} is nevertheless kept, at
negligible cost, because it controls the conditioning of
$\mathbf{Y}_t^\top\mathbf{Y}_t$ in~\eqref{eq:sylvester}, and because it
makes every iteration well defined even when
Assumption~\ref{ass:smooth}(b) fails: at any $\theta^{(k)}\in\mathfrak{B}$
with $g^{(k)}\neq0$, both~\eqref{eq:rank_test} and~\eqref{eq:armijo} hold
for all sufficiently small $\alpha$ by continuity.
(b)~A strong Wolfe search could be used instead, but its curvature
condition would have to be met inside $(0,\alpha_{\max,k})$, with
$\alpha_{\max,k}:=\sup\{\alpha>0:\mathbf{Y}_t^{(k)}+s\eta_{\mathbf{Y},t}^{(k)}\in\Rpr\ \forall s\in[0,\alpha],\ \forall t\}$,
and the standard existence argument~\cite[Lemma~3.1]{nocedal2006numerical}
does not cover a search curve that may exit the feasible set at a finite
step. Because the safeguard already enforces the angle and length
conditions, the curvature condition is not needed, and the
Hestenes--Stiefel momentum is retained whenever it passes the safeguard.
(c)~Critical points are characterized up to the $\Or$ action: $g^{(k)}$
is horizontal, so the result is a statement on the quotient.
\end{nrevised}

\nrev{Theorem~\ref{thm:convergence} guarantees that Algorithm~\ref{alg:depfm} reaches critical points regardless of non-convexity. We next ask how close a solution in a neighborhood of the truth is to the true precision sequence.}

\subsection{Statistical Setting}
\label{subsec:stat_setting}
\begin{revised}
\begin{assumption}[Sampling model and weight function]
\label{ass:ellip}
For each $t$, $\mathbf{x}_{t,1},\ldots,\mathbf{x}_{t,n_t}\overset{\rm i.i.d.}{\sim}\mathcal{ES}(0,\Sigma_t^\star,g_t)$
with $\Sigma_t^\star=(\Theta_t^\star)^{-1}$, the windows are independent,
and the loss~\eqref{eq:likelihood} uses the same generator $g_t$. Let
$Q_t:=\mathbf{x}^\top\Theta_t^\star\mathbf{x}$ for a generic sample of
window $t$. For constants $\kappa_\psi\ge 1$ and $\sigma_\psi\ge 0$:
(a) $0\le\psi_t(s)\le\kappa_\psi p$ for all $s\ge 0$;
(b) (Fisher consistency) $\E[u_t(Q_t)\mathbf{x}\mathbf{x}^\top]=\Sigma_t^\star$;
(c) $\operatorname{Var}[\psi_t(Q_t)]\le\sigma_\psi^2p$.
\end{assumption}

\begin{remark}[Scope of the sampling assumption]
\label{rem:window_independence}
The independence across windows in Assumption~\ref{ass:ellip} is used
only in the concentration argument that yields the temporal effective
dimension in Theorem~\ref{thm:estimation}. Accordingly, the present
finite-sample variance-reduction guarantee applies directly to
independent, non-overlapping windows. For temporally dependent or
overlapping windows, the deterministic optimization formulation remains
unchanged, whereas the stochastic analysis would require concentration
bounds adapted to the corresponding dependence structure.
\end{remark}

\begin{example}[Student-$t$]
\label{ex:t}
For the $t_\nu$ generator, $\psi_t(s)=(\nu+p)s/(\nu+s)\le(1+\nu)p$, so
(a) holds with $\kappa_\psi=1+\nu$. If the data are $t_\nu$ with the
same $\nu$, then $Q_t/(\nu+Q_t)\sim\mathrm{Beta}(p/2,\nu/2)$, whence
$\E\psi_t(Q_t)=p$, (b) holds, and
$\operatorname{Var}\psi_t(Q_t)=2p\nu/(p+\nu+2)\le 2p$, i.e., (c) holds with
$\sigma_\psi^2=2$. These facts hold for every $\nu>0$, including $\nu\le 2$
where second moments do not exist. The Gaussian loss violates (a)
because $\psi(s)=s$ is unbounded; the analysis below extends to it by a
standard truncation argument at the price of additional logarithmic
factors, which we do not pursue.
\end{example}

\begin{assumption}[Structural regularity]
\label{ass:structure}
For each $t$, $\Theta_t^\star=\mathbf{L}_t^\star+\mathbf{D}_t^\star$ with
$\mathbf{L}_t^\star=\mathbf{Y}_t^\star\mathbf{Y}_t^{\star\top}$,
$\mathbf{Y}_t^\star\in\Rpr$, and
(a) $\lambda_r(\mathbf{L}_t^\star)\ge\sigma_\star>0$;
(b) $\alpha(\mathbf{Y}_t^\star)\le\alpha_\star<1$ (cf.~\eqref{eq:alpha});
(c) $d_{\rm lo}\mathbf{I}\preceq\mathbf{D}_t^\star\preceq d_{\rm hi}\mathbf{I}$
for known constants $0<d_{\rm lo}<d_{\rm hi}$.
\end{assumption}

Under Assumption~\ref{ass:structure}(b), the subspace
\[
\mathcal{T}_t:=T_{\mathbf{L}_t^\star}+\Dsp\subset\mathcal{S}^p
\]
has dimension exactly $\dr$ (Theorem~\ref{thm:identifiability}); it is
the tangent space of $\Mpr$ at $\Theta_t^\star$, and $\PT_t$ denotes its
Frobenius projector. Fix a radius $0<\varrho\le(1-\alpha_\star)\sigma_\star/4$
\nrev{and a curvature radius $r_c\in(0,\infty]$ ($r_c=\infty$ unless
stated otherwise; the Student-$t$ analysis of
Proposition~\ref{prop:t_rsc} uses $r_c=d_{\rm lo}/5$)}
and define the neighborhood (in decomposed coordinates)
\begin{equation}
\begin{aligned}
\Nb:=\Bigl\{(\mathbf{L}_t,\mathbf{D}_t)_{t=1}^T:\ &\mathbf{L}_t\succeq\mathbf{0},\ \operatorname{rank}\mathbf{L}_t\le r,\
\|\mathbf{L}_t-\mathbf{L}_t^\star\|_F\le\varrho,\\
&\mathbf{D}_t\in\Dsp,\ d_{\rm lo}\mathbf{I}\preceq\mathbf{D}_t\preceq d_{\rm hi}\mathbf{I},\\
&\nrev{\|\mathbf{L}_t+\mathbf{D}_t-\Theta_t^\star\|_F\le r_c}\Bigr\},
\end{aligned}
\label{eq:nbhd}
\end{equation}
and $\Theta_t:=\mathbf{L}_t+\mathbf{D}_t$. Since $\varrho<\sigma_\star$, every
$\mathbf{L}_t$ in $\Nb$ has rank exactly $r$, so $\Nb$ is a compact subset
of LRaD sequences, and every $\Theta_t$ in $\Nb$ satisfies
$m\mathbf{I}\preceq\Theta_t\preceq M\mathbf{I}$ with
\[
m:=d_{\rm lo},\quad M:=d_{\rm hi}+\max_t\|\mathbf{L}_t^\star\|_{\op}+\varrho,\quad c_F:=m^2/M^4.
\]
We write $\delta_t:=\dS(\Theta_t^\star,\Theta_{t+1}^\star)$,
$\bar\delta^2:=\frac1T\sum_{t=1}^{T-1}\delta_t^2$,
$a_t:=\|\Theta_t^\star\|_{0,\mathrm{off}}$, $\bar a:=\frac1T\sum_ta_t$, and
$n_{\min}:=\min_tn_t$.

\begin{assumption}[Restricted strong convexity]
\label{ass:rsc}
There are $\ell>0$ and $\eta_0\in[0,1)$ such that, with probability at
least $1-\eta_0$, for all $t$ and all $(\mathbf{L}_t,\mathbf{D}_t)$ in the
$t$-th slice of $\Nb$,
\[
\mathcal{L}_t(\Theta_t)-\mathcal{L}_t(\Theta_t^\star)-\langle\nabla\mathcal{L}_t(\Theta_t^\star),\Theta_t-\Theta_t^\star\rangle
\ge\tfrac{\ell}{2}\|\Theta_t-\Theta_t^\star\|_F^2.
\]
\end{assumption}

\begin{lemma}[Deterministic RSC for convex $\rho_t$]
\label{lem:rsc_convex}
If each $\rho_t$ is convex (e.g., the Gaussian or Huber-type losses),
Assumption~\ref{ass:rsc} holds with $\eta_0=0$ and $\ell=1/(2M^2)$\nrev{,
for every $r_c\in(0,\infty]$}.
\end{lemma}
\begin{proof}
See Appendix~\ref{app5}.
\end{proof}

\begin{nrevised}
\emph{Heavy tails: Student-$t$ observations.}
Assumption~\ref{ass:rsc} is the standard local curvature condition for
nonconvex M-estimators~\cite{loh2015regularized,negahban2012unified}. For
the Student-$t$ loss, $\rho_t(s)=\frac{\nu+p}{2}\log(1+s/\nu)$ is
concave, Lemma~\ref{lem:rsc_convex} does not apply, and the empirical risk
is not convex in $\Theta$; the condition must therefore be proved. Two
structural facts drive the argument: the Student-$t$ negative
log-likelihood is geodesically convex for the affine-invariant
geometry~\cite{wiesel2012geodesic,sra2015conic}, and its Fisher
information~\cite{lange1989robust} is positive but weak along the single
direction $\Theta_t^\star$ that changes the overall scale. Let
\[
\kappa_\nu:=\frac{\nu}{p+\nu+2}\in(0,1),\qquad
\bar{\mathcal{L}}_t:=\E\mathcal{L}_t .
\]

\begin{proposition}[Local curvature of the Student-$t$ loss]
\label{prop:t_rsc}
For each $t$, let $\mathbf{x}_{t,1},\ldots,\mathbf{x}_{t,n_t}\overset{\rm i.i.d.}{\sim}t_\nu(\mathbf{0},\Sigma_t^\star)$,
$\Theta_t^\star=(\Sigma_t^\star)^{-1}$,
$m\mathbf{I}\preceq\Theta_t^\star\preceq M\mathbf{I}$, and let
$\mathcal{L}_t$ use the $t_\nu$ generator with the same $\nu>0$.
\begin{enumerate}
\item[(i)] (Population curvature) For every
$\boldsymbol{\Delta}\in\mathcal{S}^p$, with
$\mathbf{A}:=\Sigma_t^{\star1/2}\boldsymbol{\Delta}\Sigma_t^{\star1/2}$,
\begin{align*}
D^2\bar{\mathcal{L}}_t(\Theta_t^\star)[\boldsymbol{\Delta},\boldsymbol{\Delta}]
&=\frac{(p+\nu)\|\mathbf{A}\|_F^2-(\tr\mathbf{A})^2}{2(p+\nu+2)}\\
&\ge\frac{\kappa_\nu}{2}\|\mathbf{A}\|_F^2\ \ge\ \frac{\kappa_\nu}{2M^2}\|\boldsymbol{\Delta}\|_F^2,
\end{align*}
and the first inequality is an equality if and only if
$\boldsymbol{\Delta}\in\operatorname{span}\{\Theta_t^\star\}$.
\item[(ii)] (Sharpness) Let $d_{\rm lo}<\lambda_{\min}(\mathbf{D}_t^\star)$
and $\lambda_{\max}(\mathbf{D}_t^\star)<d_{\rm hi}$. If, for some
$\ell>0$, the inequality of Assumption~\ref{ass:rsc} holds with
$\bar{\mathcal{L}}_t$ in place of $\mathcal{L}_t$ on the $t$-th slice of
$\Nb$, then
$\ell\le p\kappa_\nu/(2\|\Theta_t^\star\|_F^2)\le\kappa_\nu/(2m^2)$.
\item[(iii)] (Empirical local RSC) Let $\eta\in(0,1)$ and
\begin{equation}
  n_{\min}\ \ge\ 950\,\frac{(p+\nu+2)^3}{\nu^2}\,\log\frac{4pT}{\eta}.
  \label{eq:t_sample}
\end{equation}
Then, with probability at least $1-\eta$, simultaneously for all $t$ and
all $\Theta\in\mathcal{S}^p$ with $\|\Theta-\Theta_t^\star\|_F\le m/5$,
\begin{equation}
\begin{aligned}
&\mathcal{L}_t(\Theta)-\mathcal{L}_t(\Theta_t^\star)-\langle\nabla\mathcal{L}_t(\Theta_t^\star),\Theta-\Theta_t^\star\rangle\\
&\qquad\ge\frac{\ell_\nu}{2}\|\Theta-\Theta_t^\star\|_F^2,
\qquad \ell_\nu:=\frac{\kappa_\nu}{28M^2}.
\end{aligned}
\label{eq:t_rsc}
\end{equation}
\end{enumerate}
\end{proposition}
\begin{proof}
See Appendix~\ref{app:t_rsc}.
\end{proof}

Under Student-$t$ sampling, Assumption~\ref{ass:rsc} therefore holds with
$\ell=\ell_\nu$, $r_c=m/5$, and $\eta_0=\eta$ as soon
as~\eqref{eq:t_sample} holds (Corollary~\ref{cor:student} below). Three
comments are in order. First, (iii) holds on a full Frobenius ball and
does not use the LRaD structure. Its proof rests on an exact
decomposition of the Bregman divergence into a divergence along the
affine-invariant geodesic, which is controlled from below termwise by
geodesic convexity and a matrix Chernoff bound, and a geodesic--chord
discrepancy, which is paired with the whitened score and controlled by a
matrix Bernstein bound. Second, (i)--(ii) show that the factor
$\kappa_\nu=\nu/(p+\nu+2)$ in $\ell_\nu$ cannot be removed: along the
admissible scale direction $\boldsymbol{\Delta}\propto\Theta_t^\star$, the
Fisher information is $p\kappa_\nu/2$, compared with $p/2$ for Gaussian
data, because the radial factor of a $t_\nu$ vector confounds the overall
scale. In contrast with the Gaussian case ($\ell=1/(2M^2)$,
Lemma~\ref{lem:rsc_convex}), no dimension-free constant $\ell$ can hold in
Assumption~\ref{ass:rsc}. Third, the weak direction is one-dimensional
per window: for perturbations with
$\tr\mathbf{A}=\langle\boldsymbol{\Delta},\Sigma_t^\star\rangle=0$, (i)
gives the Gaussian-order curvature
$\frac{p+\nu}{2(p+\nu+2)}\|\mathbf{A}\|_F^2$.
\end{nrevised}

\emph{Localized estimator.} Let $\hat\Theta=(\hat{\mathbf{L}}_t+\hat{\mathbf{D}}_t)_t$
with $(\hat{\mathbf{L}}_t,\hat{\mathbf{D}}_t)_t\in\Nb$ be any point satisfying
$F(\hat\Theta)\le F(\Theta^\star)$; for instance, any minimizer of $F$ over
the compact set $\Nb$, which exists because $F$ is continuous.

\emph{Effective temporal dimension.} For $\tau\ge 0$, let
\begin{equation}
  \dtime(\tau):=\frac1T\tr\bigl[(\mathbf{I}_T+\tau\LT)^{-1}\bigr]
  =\frac1T\sum_{k=0}^{T-1}\frac{1}{1+4\tau\sin^2\!\bigl(\frac{\pi k}{2T}\bigr)}.
  \label{eq:dtime}
\end{equation}
\begin{lemma}
\label{lem:dtime}
$\dtime(0)=1$, $\dtime$ is decreasing, $\lim_{\tau\to\infty}\dtime(\tau)=1/T$, and
$\dtime(\tau)\le\frac1T+(1+4\tau)^{-1/2}$ for all $\tau\ge0$.
\end{lemma}
\begin{proof}
See Appendix~\ref{app5}.
\end{proof}
\end{revised}

\subsection{Finite-Sample Oracle Inequality}
\label{subsec:oracle}
\begin{revised}
Let $\eta\in(0,1)$, $L_\eta:=\log(6epT/\eta)$, and
\[
\sigma_W^2:=\frac{2+\sigma_\psi^2}{m^2},\qquad
C_\omega:=\sqrt{2(1+\sigma_\psi^2)}+1+3\sqrt{2+\sigma_\psi^2}.
\]

\begin{theorem}[Oracle inequality with temporal variance reduction]
\label{thm:estimation}
Let Assumptions~\ref{ass:ellip}--\ref{ass:rsc} hold, let $\lambda,\mu\ge0$,
and suppose that
\begin{equation}
\begin{aligned}
  &n_{\min}\ \ge\ C_n\,r\,\dr\,L_\eta,\\
  &C_n:=\max\Bigl\{2(\kappa_\psi+1)^2,\Bigl(\frac{16\,C_\omega}{3(1-\alpha_\star)\ell\sigma_\star m}\Bigr)^2\Bigr\}.
\end{aligned}
  \label{eq:sample_size}
\end{equation}
Then, with probability at least $1-\eta-\eta_0$, every localized estimator
$\hat\Theta$ satisfies
\begin{equation}
\begin{aligned}
  &\frac{1}{T}\sum_{t=1}^T\|\hat\Theta_t-\Theta_t^\star\|_F^2
  \le\ \bar{\mathcal{E}}(\mu,\lambda)\\
  &:=\underbrace{\frac{72\,\sigma_W^2}{\ell^2}\,\frac{\dr L_\eta}{n_{\min}}\;\dtime\!\Bigl(\frac{2\mu c_F}{\ell}\Bigr)}_{\text{stochastic error}}\\
  &\quad+\underbrace{\frac{32\lambda^2}{\ell^2}\,\bar a}_{\text{sparsity bias}}
  +\underbrace{\frac{32\mu}{\ell}\,\bar\delta^2}_{\text{temporal bias}},
\end{aligned}
\label{eq:error_bound}
\end{equation}
and $\frac1T\sum_t\dS^2(\hat\Theta_t,\Theta_t^\star)\le\frac{M^2}{m^4}\bar{\mathcal{E}}(\mu,\lambda)$.
If, in addition, $T\,\bar{\mathcal{E}}(\mu,\lambda)<\frac{3(1-\alpha_\star)}{4}r_0^2$
with $r_0:=\min\{\varrho,\ \nrev{r_c,}\ \min_t\lambda_{\min}(\mathbf{D}_t^\star)-d_{\rm lo},\ d_{\rm hi}-\max_t\lambda_{\max}(\mathbf{D}_t^\star)\}>0$,
then $\hat\Theta$ lies in the interior of $\Nb$; in that case every
minimizer of $F$ over $\Nb$ is a local minimizer, hence a critical point,
of~\eqref{eq:DEPFM}.
\end{theorem}
\begin{proof}
See Appendix~\ref{app5}.
\end{proof}

\emph{Interpretation.} The three terms in~\eqref{eq:error_bound} separate
cleanly. (i)~The stochastic error scales with the intrinsic dimension
$\dr=pr+p-r(r-1)/2$, not with $p(p+1)/2$, and is multiplied by
$\dtime\in[1/T,1]$, which equals $1$ for independent windows ($\mu=0$)
and decreases monotonically to $1/T$ as $\mu$ grows: temporal coupling
pools the $T$ windows and \emph{reduces variance}. (ii)~The temporal
bias grows linearly in $\mu$ and vanishes when the truth is static.
(iii)~The sparsity level $\bar a$ enters only through a bias term,
because $\nabla h(\Theta_t^\star)$ is supported on the true edges
(Lemma~\ref{lem:logcosh}); with $\lambda^2\lesssim\dr L_\eta/(\bar a\,n_{\min})$
this term does not exceed the stochastic error. The smooth sparsity
penalty does not improve the rate in this analysis; its role is to
regularize the off-diagonal pattern and to stabilize the thresholding
step. The sample-size condition~\eqref{eq:sample_size} is nearly linear
in the intrinsic dimension (up to the factor $r$ and logarithms); in
particular, consistency requires $n_{\min}\gg r\dr\log(pT)$, not a sample
size logarithmic in $p$. The constants are not optimized.
\end{revised}

\begin{nrevised}
Combining Theorem~\ref{thm:estimation} with
Proposition~\ref{prop:t_rsc} yields an explicit finite-sample
specialization of the general theory to correctly specified
Student-$t$ observations.

\begin{corollary}[Guarantees under Student-$t$ observations]
\label{cor:student}
Let the data be as in Proposition~\ref{prop:t_rsc}, let
Assumption~\ref{ass:structure} hold, and construct $\Nb$ with
$r_c=m/5$. Let $\eta,\eta'\in(0,1)$, and suppose that $n_{\min}$
satisfies~\eqref{eq:sample_size} with $\ell=\ell_\nu$,
$\kappa_\psi=1+\nu$, and $\sigma_\psi^2=2$, and~\eqref{eq:t_sample} with
$\eta'$ in place of $\eta$. Then, with probability at least
$1-\eta-\eta'$, every localized estimator satisfies~\eqref{eq:error_bound}
with $\ell=\ell_\nu$, the conclusions of Corollary~\ref{cor:temporal} hold
with $\ell=\ell_\nu$, and, under~\eqref{eq:support_cond},
$\hat E_t=E_t^\star$ for all $t$. Thus, no additional restricted-strong-convexity assumption is needed:
the required local curvature follows explicitly from
Proposition~\ref{prop:t_rsc} under the stated sample-size condition.
\end{corollary}
\begin{proof}
By Example~\ref{ex:t}, Assumption~\ref{ass:ellip} holds with
$\kappa_\psi=1+\nu$ and $\sigma_\psi^2=2$. Since
$\Theta_t^\star=\mathbf{L}_t^\star+\mathbf{D}_t^\star\succeq d_{\rm lo}\mathbf{I}=m\mathbf{I}$
and $\Theta_t^\star\preceq M\mathbf{I}$, and since every $\Theta_t$ in the
$t$-th slice of $\Nb$ satisfies $\|\Theta_t-\Theta_t^\star\|_F\le m/5$,
Proposition~\ref{prop:t_rsc}(iii), applied with $\eta'$, shows that
Assumption~\ref{ass:rsc} holds with $\ell=\ell_\nu$ and $\eta_0=\eta'$.
Theorem~\ref{thm:estimation}, Corollary~\ref{cor:temporal}, and
Theorem~\ref{thm:support} then apply verbatim.
\end{proof}

\begin{remark}[Order of the constants for heavy tails]
\label{rem:t_scaling}
With $\ell=\ell_\nu$ and $\sigma_W^2=4/m^2$, the stochastic term
of~\eqref{eq:error_bound} is of order
$\frac{M^4}{m^2}\bigl(\frac{p+\nu+2}{\nu}\bigr)^2\frac{\dr\log(pT)}{n_{\min}}\dtime$,
and~\eqref{eq:sample_size} becomes
$n_{\min}\gtrsim\bigl(\frac{p+\nu+2}{\nu}\bigr)^2r\dr\log(pT)$. Since
$\sigma_\star\le\|\mathbf{L}_t^\star\|_{\op}\le M$ and $m\le M$, the
constant $C_n$ in~\eqref{eq:sample_size} is then at least
$(16C_\omega\cdot28/3)^2\kappa_\nu^{-2}>1.98\cdot10^6\kappa_\nu^{-2}$ (here $C_\omega=\sqrt6+7$), so
that~\eqref{eq:sample_size} implies~\eqref{eq:t_sample} (with
$\eta'=\eta$) whenever $\nu\le2000p$. For
$\nu\gtrsim p$ these orders coincide with those obtained for a
dimension-free curvature constant; for fixed $\nu$ they are pessimistic
by the factor $(p/\nu)^2$. By Proposition~\ref{prop:t_rsc}(ii), this factor
is unavoidable for any analysis based on an isotropic curvature constant.
It is, however, not intrinsic to the estimation problem: in the scale
direction both the curvature and the score variance are reduced, since
$\operatorname{Var}\langle u_t(Q_t)\mathbf{x}\mathbf{x}^\top,\Theta_t^\star\rangle=\operatorname{Var}\psi_t(Q_t)=2p\kappa_\nu$
(Example~\ref{ex:t}). The Fisher information of
Proposition~\ref{prop:t_rsc}(i) therefore suggests that a
direction-dependent version of the master inequality of
Appendix~\ref{app5}, which treats the single scale direction of each
window separately, would yield a stochastic term of order
$(\dr+p/\nu)\log(pT)/n_{\min}$; we do not pursue this refinement here.
\end{remark}
\end{nrevised}

\subsection{Benefit of Temporal Coupling}
\label{subsec:temporal}
\begin{revised}
Let $A:=72\,\sigma_W^2\dr L_\eta/(\ell^2n_{\min})$ and $B:=16\,\bar\delta^2/c_F$,
so that, by~\eqref{eq:error_bound} with $\tau=2\mu c_F/\ell$,
$\bar{\mathcal{E}}=A\,\dtime(\tau)+B\tau+32\lambda^2\bar a/\ell^2$.

\begin{corollary}[Variance--bias trade-off in $\mu$]
\label{cor:temporal}
On the event of Theorem~\ref{thm:estimation}:
\begin{enumerate}
\item[(i)] (Independent windows) For $\mu=0$:
$\frac1T\sum_t\|\hat\Theta_t-\Theta_t^\star\|_F^2\le A+32\lambda^2\bar a/\ell^2$.
\item[(ii)] (Static truth) If $\bar\delta=0$ and $\mu\ge\ell(T^2-1)/(8c_F)$:
$\frac1T\sum_t\|\hat\Theta_t-\Theta_t^\star\|_F^2\le 2A/T+32\lambda^2\bar a/\ell^2$,
i.e., the pooled-sample rate with $Tn_{\min}$ samples.
\item[(iii)] (Slowly varying truth) If $\bar\delta>0$, consider the
analytically balanced choice
\[
\mu_{\rm bal}
:=
\frac{\ell}{2c_F}
\left(\frac{A}{4B}\right)^{2/3},
\]
which minimizes the upper envelope used in the proof. Then
\[
\frac1T\sum_t\|\hat\Theta_t-\Theta_t^\star\|_F^2\le\frac{A}{T}+1.2\,A^{2/3}B^{1/3}+\frac{32\lambda^2\bar a}{\ell^2}.
\]
In particular, if $B\le A/8$ and $T\ge 4$, the stochastic-plus-temporal
part of this bound is at most $0.85A$, strictly below the bound (i) for
independent windows.
\end{enumerate}
\end{corollary}
\begin{proof}
See Appendix~\ref{app5}.
\end{proof}

Corollary~\ref{cor:temporal} identifies a regime in which
affine-invariant temporal coupling provably improves the finite-sample
bound. When the per-window variation $\bar\delta^2$ is small relative to
the statistical resolution $A$, the analytically balanced choice $\mu_{\rm bal}$ yields $
\mathcal{O}\!\left(
\frac{\dr\log(pT)}{n_{\min}T}
+
\left(\frac{\dr\log(pT)}{n_{\min}}\right)^{2/3}
\bar\delta^{2/3}
\right)$, interpolating between the independent-window and fully pooled regimes.
The value $\mu_{\rm bal}$ minimizes the explicit upper envelope used in
the proof and is not claimed to minimize the original oracle bound
exactly. In practice, $\mu$ is selected by validation. When the temporal
variation is large relative to the statistical error, the theory instead
favors weak temporal coupling.
\end{revised}

\subsection{Model Misspecification and Rank Selection}
\label{subsec:misspec}
\begin{revised}
We now drop the requirement $\Theta_t^\star\in\Mpr$ in
Assumption~\ref{ass:structure}: $\Theta_t^\star\in\Spp$ is arbitrary,
and the data satisfy Assumption~\ref{ass:ellip}. Let
$\Theta^\circ=(\mathbf{L}_t^\circ+\mathbf{D}_t^\circ)_t$ be any LRaD
\emph{comparator} sequence satisfying Assumption~\ref{ass:structure} in
place of $\Theta^\star$ (e.g., a best rank-$r$ LRaD approximation), let
$\Nb^\circ$, $\mathcal{T}_t^\circ$, $\PT_t^\circ$, $m,M,c_F$ be defined as
above with $\Theta^\circ$ in place of $\Theta^\star$, and let
Assumption~\ref{ass:rsc} hold at $\Theta^\circ$. Define the population
gradient and the centered score
\[
\mathbf{g}_t^\circ:=\E\nabla\mathcal{L}_t(\Theta_t^\circ),\qquad
\mathbf{Z}_t^\circ:=2\bigl(\nabla\mathcal{L}_t(\Theta_t^\circ)-\mathbf{g}_t^\circ\bigr),
\]
and let $\sigma_\circ^2$ be such that
$\operatorname{Var}\langle u_t(\mathbf{x}^\top\Theta_t^\circ\mathbf{x})\mathbf{x}\mathbf{x}^\top,\mathbf{E}\rangle\le\sigma_\circ^2\|\mathbf{E}\|_F^2$
for all $\mathbf{E}\in\mathcal{S}^p$ and all $t$ (finite whenever the
weights are bounded). Let $\delta_t^\circ$, $\bar\delta_\circ$, $a_t^\circ$,
$\bar a_\circ$ be the temporal variation and sparsity of $\Theta^\circ$.

\begin{theorem}[Oracle inequality under misspecification]
\label{thm:misspec}
Let $\Omega^\circ$ be the event
$\max_t\bigl(\|\mathbf{Z}_t^\circ\|_{\op}+\|\PT_t^\circ\mathbf{Z}_t^\circ\|_F\bigr)\le\frac{3(1-\alpha_\star)\ell\sigma_\star}{16\sqrt r}$.
For every $\eta\in(0,1)$, with probability at least
$1-\eta-\eta_0-\mathbb{P}((\Omega^\circ)^c)$, every
$\hat\Theta\in\Nb^\circ$ with $F(\hat\Theta)\le F(\Theta^\circ)$ satisfies
\begin{equation}
\begin{aligned}
\frac1T\sum_{t=1}^T\|\hat\Theta_t-\Theta_t^\star\|_F^2
\le\ &\frac{16\,\sigma_\circ^2}{\eta\,\ell^2}\frac{\dr}{n_{\min}}\,\dtime\!\Bigl(\frac{2\mu c_F}{\ell}\Bigr)\\
&+\frac{64\lambda^2}{\ell^2}\bar a_\circ+\frac{64\mu}{\ell}\bar\delta_\circ^2+\mathcal{E}_{\rm mis}(r),
\end{aligned}
\label{eq:misspec}
\end{equation}
where
$\mathcal{E}_{\rm mis}(r):=\frac1T\sum_t\bigl(2\|\Theta_t^\circ-\Theta_t^\star\|_F^2+\frac{128}{\ell^2}\|\mathbf{g}_t^\circ\|_F^2\bigr)$.
\end{theorem}
\begin{proof}
See Appendix~\ref{app5}.
\end{proof}

If $\Theta^\star$ is itself LRaD and $\Theta^\circ=\Theta^\star$, then
$\mathbf{g}_t^\circ=\mathbf{0}$ by Fisher consistency and
$\mathcal{E}_{\rm mis}=0$, recovering Theorem~\ref{thm:estimation} up to
constants and the $1/\eta$ factor (the latter can be replaced by
$L_\eta$ under the boundedness conditions of Assumption~\ref{ass:ellip}
at $\Theta^\circ$). For the Gaussian loss,
$\mathbf{g}_t^\circ=\frac12(\Sigma_t^\star-(\Theta_t^\circ)^{-1})$, so
$\mathcal{E}_{\rm mis}$ is controlled by the distance from
$\Theta_t^\star$ to the LRaD class. Theorem~\ref{thm:misspec} explains the
moderate loss of accuracy of any LRaD method on graphs outside the LRaD
family (cf.\ the graph structure discussion in Section~\ref{sec:lrad}):
$\mathcal{E}_{\rm mis}(r)>0$ does not vanish with $n$.

\begin{remark}[Rank selection]
\label{rem:rank}
In~\eqref{eq:misspec}, the stochastic term grows linearly in $r$ through
$\dr$, while the approximation term $\mathcal{E}_{\rm mis}(r)$ typically
decreases as $r$ grows, which is the classical estimation--approximation
trade-off. This motivates selecting $r\in\{1,\ldots,r_{\max}\}$ by
held-out likelihood or by the information criterion
\begin{align*}
\mathrm{IC}(r)=\ &2\sum_{t=1}^Tn_t\,\mathcal{L}_t(\hat\Theta_t^{(r)})\\
&+\log\Bigl(\sum_tn_t\Bigr)\,T\,\dr(r)\,\dtime\!\Bigl(\frac{2\mu c_F}{\ell}\Bigr),
\end{align*}
in which the effective number of parameters $T\dr\dtime$ is suggested by
Theorem~\ref{thm:estimation}: $T\dr$ for independent windows and $\dr$ for a
fully pooled static model. This criterion is a heuristic consequence of
the theory, not a consistency result.
\end{remark}
\end{revised}

\subsection{Thresholded Recovery of Partial-Correlation Graphs}
\label{subsec:support}
\begin{revised}
Let $\varrho^\star_{t,ij}$ be the partial correlations~\eqref{eq:pcorr_graph}
of $\Theta_t^\star$, $E_t^\star:=\{(i,j):i\neq j,\ \varrho_{t,ij}^\star\neq0\}$,
and $\tau_{\min}:=\min_t\min_{(i,j)\in E_t^\star}|\varrho^\star_{t,ij}|$.

\begin{theorem}[Thresholded edge recovery]
\label{thm:support}
Define $
e_{\rm pc}
:=
\frac{2M}{m^2}
\sqrt{T\,\bar{\mathcal{E}}(\mu,\lambda)}$. On the event of Theorem~\ref{thm:estimation}, suppose that
\begin{equation}
  T\,\bar{\mathcal{E}}(\mu,\lambda)<\Bigl(\frac{m^2\tau_{\min}}{4M}\Bigr)^2.
  \label{eq:support_cond}
\end{equation}
Then $e_{\rm pc}<\tau_{\min}/2$, and every threshold satisfying
\[
e_{\rm pc}
<
\tau_{\rm thr}
<
\tau_{\min}-e_{\rm pc}
\]
recovers the partial-correlation graphs exactly:
\[
\widehat E_t=E_t^\star,
\qquad t=1,\ldots,T.
\]
In particular, the oracle choice
$\tau_{\rm thr}=\tau_{\min}/2$ is admissible.
The conclusion holds with probability at least $1-\eta-\eta_0$.
\end{theorem}
\begin{proof}
See Appendix~\ref{app6}.
\end{proof}

\begin{remark}[Threshold interpretation]
\label{rem:support_threshold}
The value $\tau_{\min}/2$ is used only to characterize a separating
threshold theoretically; the practical procedure does not assume
knowledge of $\tau_{\min}$. Theorem~\ref{thm:support} shows more
generally that exact recovery is possible whenever the estimation error
leaves a nonempty separation interval
$(e_{\rm pc},\,\tau_{\min}-e_{\rm pc})$.
In the numerical experiments, $\tau_{\rm thr}$ is selected independently
of the unknown population graph.
\end{remark}

Condition~\eqref{eq:support_cond} requires each of the three terms
of~\eqref{eq:error_bound} to be small, and in particular makes explicit
that the temporal bias must satisfy
$32\mu\bar\delta^2/\ell<(m^2\tau_{\min}/(4M))^2/T$: no sample size can
compensate for a temporal bias that does not vanish. With
$\mu=\mu^\star$ of Corollary~\ref{cor:temporal}, this holds as soon as
$n_{\min}$ is large enough relative to $T$, $\tau_{\min}^{-2}$, and
$\bar\delta^2$. The factor $T$ comes from bounding the error of each
window by the sum over windows and is conservative. Under non-Gaussian
elliptical models, Theorem~\ref{thm:support} concerns the
partial-correlation graph, which differs in general from the
conditional-independence graph.
\end{revised}

\section{Experimental Results}
\label{sec:experiments}

The experiments are organized to test, one by one, the modelling and
theoretical claims of Sections~\ref{sec:problem}--\ref{sec:theory}, rather
than to provide a single aggregate ranking of methods.
Table~\ref{tab:exp_map} summarizes the correspondence between each claim
and the experiment that probes it. Synthetic experiments, for which the
ground truth is available, are reported first
(Section~\ref{subsec:synthetic}); two real-data studies on financial and
neuroimaging time series follow (Sections~\ref{subsec:sp500}
and~\ref{subsec:hcp}).

\begin{table}[!t]
\centering
\caption{Correspondence between theoretical or modelling claims and the
experiments that examine them.}
\label{tab:exp_map}
\footnotesize
\setlength{\tabcolsep}{3pt}
\resizebox{\columnwidth}{!}{%
\begin{tabular}{@{}lll@{}}
\toprule
Claim & Experiment & Figure\\
\midrule
Reduced intrinsic dimension of LRaD & sample-size sweep & \ref{fig:lrad_main}\\
Robustness of the $t$ likelihood (Prop.~\ref{prop:t_rsc}) & d.o.f.\ sweep & \ref{fig:dof}\\
Geometry of the temporal penalty & NoTemp/Frob/Fused/AIRM & \ref{fig:abl_smooth}, \ref{fig:abl_abrupt}\\
Effective temporal dimension $\dtime$ & static-truth $T$-sweep & \ref{fig:pooling}\\
Variance--bias trade-off (Cor.~\ref{cor:temporal}) & $\mu\times\delta$ sweep & \ref{fig:varbias}\\
Misspecification error $\mathcal{E}_{\rm mis}(r)$ (Thm.~\ref{thm:misspec}) & graph families, $\gamma$-sweep & \ref{fig:misspec_graph}, \ref{fig:misspec_gamma}\\
Rank trade-off and IC (Remark~\ref{rem:rank}) & rank sweep & \ref{fig:rank}\\
Nonconvexity of~\eqref{eq:DEPFM} & initialization study & \ref{fig:init}\\
Convergence to critical points (Thm.~\ref{thm:convergence}) & Armijo/rank diagnostics & \ref{fig:conv}\\
Cost $\mathcal{O}(T(n_{\max}pr+p^3))$, Eq.~\eqref{eq:complexity} & scaling in $p,r,T$ & \ref{fig:scal}\\
Congruence invariance of $d_{\rm AI}$ (Prop.~\ref{prop:airm}) & node rescaling & \ref{fig:invariance}\\
Practical relevance & S\&P~500, HCP & \ref{fig:sp_quant}--\ref{fig:chord}\\
\bottomrule
\end{tabular}}
\end{table}

\subsection{Experimental Protocol}
\label{subsec:protocol}

\paragraph{Competing methods}
We compare DEGFM with four baselines that isolate its three ingredients:
the LRaD parameterization, the elliptical likelihood, and the geodesic
temporal coupling. \emph{GLasso}~\cite{friedman2008sparse} is fitted
independently in each window and has none of the three ingredients.
\emph{TVGL}~\cite{hallac2017network} couples sparse Gaussian precision
matrices through an element-wise $\ell_1$ penalty on consecutive
differences. \emph{IndGGFM}~\cite{zhang2025graph} uses the Gaussian LRaD
model without temporal coupling, i.e., the special case $\mu=0$ of
\eqref{eq:DEPFM} with a Gaussian generator. \emph{Javaheri'25}~\cite{javaheri2025time}
learns time-varying Laplacian-constrained graphs under a Student-$t$
likelihood and is therefore the closest competitor in terms of
heavy-tail modelling and dynamics. In addition, four DEGFM variants that
share the likelihood, rank, sparsity penalty, initialization, and
stopping rule, and differ only in the temporal regularizer, are used for
ablation: \emph{NoTemp} ($\mu=0$), \emph{Frob}
($\mu\sum_t\|\Theta_{t+1}-\Theta_t\|_F^2$), \emph{Fused}
($\mu\sum_t\|\Theta_{t+1}-\Theta_t\|_{1,\mathrm{off}}$), and the proposed
\emph{AIRM} regularizer ($\mu\sum_t d_{\rm AI}^2(\Theta_t,\Theta_{t+1})$, with $d_{\rm AI}^2(\cdot,\cdot)=\dS^2(\cdot,\cdot)$).

\paragraph{Evaluation metrics}
Estimation accuracy is measured by the window-averaged relative error
$\mathrm{RE}=T^{-1}\sum_t\|\widehat\Theta_t-\Theta_t^\star\|_F/\|\Theta_t^\star\|_F$
and by the mean squared error
$\mathrm{MSE}=T^{-1}\sum_t\|\widehat\Theta_t-\Theta_t^\star\|_F^2$,
which is the quantity bounded in Theorem~\ref{thm:estimation}. Graph
recovery refers to the partial-correlation graph~\eqref{eq:pcorr_graph}:
the AUC is computed by ranking off-diagonal pairs by
$|\widehat\varrho_{t,ij}|$, pooled over windows, and the F1 score is
computed after thresholding $|\widehat\varrho_{t,ij}|$ at the
data-driven level of Remark~\ref{rem:support_threshold}. When no ground
truth is available, we report the held-out negative log-likelihood (NLL)
per observation and node, evaluated for all methods under a common
multivariate-$t_{\hat\nu}$ predictive density with $\hat\nu$ fitted on the
training data, so that Gaussian and elliptical estimators are scored on
the same scale. Community structure is assessed by modularity $Q$,
normalized mutual information (NMI), and adjusted Rand index (ARI) of a
spectral partition of $|\widehat\varrho_t|$, and temporal smoothness by
$\mathrm{TV}_F=\sum_t\|\widehat\Theta_t-\widehat\Theta_{t+1}\|_F$ and
$\mathrm{TV}_{\rm AI}=\sum_t d_{\rm AI}^2(\widehat\Theta_t,\widehat\Theta_{t+1})$.
Because each method regularizes a different notion of temporal variation,
neither TV measure is by itself a measure of accuracy; both are reported
for completeness.

\paragraph{Tuning and repetitions}
All methods are tuned with the same protocol. In synthetic experiments,
an independent validation sample of the same size as the training sample
is drawn in every window, and $(\lambda,\mu,r)$ are selected by
minimizing the validation NLL over the grids
$\lambda\in10^{\{-4,-3.5,\ldots,-1\}}$,
$\mu\in\{0,10^{-3},3\times10^{-3},\ldots,3,10\}$, and
$r\in\{2,5,10,15,20\}$ (the baselines use the same $\lambda$ grid and
their own temporal weights on the $\mu$ grid). The true support is never
used for tuning. Real data are split chronologically within each window
into $70\%$ training, $15\%$ validation, and $15\%$ test samples. Unless
stated otherwise, synthetic results are averaged over
$N_{\rm MC}=\res{30}$ independent Monte Carlo replications, all methods
are run on the same data realizations and random seeds, and we report
mean\,$\pm$\,standard deviation. The log-cosh parameter is fixed at
$\varepsilon=10^{-2}$.

\paragraph{Implementation}
DEGFM is initialized spectrally (Section~\ref{sec:algorithm}) and solved
by the safeguarded RCG method with the rank-preserving Armijo rule
\eqref{eq:rank_test}--\eqref{eq:armijo}, using
$c_1=10^{-4}$, $\beta=0.5$, and
$\gamma_\sigma=0.1$. The AIRM terms are evaluated exactly through
generalized eigendecompositions. All iterative methods stop when the
(Riemannian) gradient norm, or the primal--dual residual for ADMM-based
baselines, falls below $10^{-5}$, or after $k_{\max}=500$ iterations.
Experiments were run on an Intel Core i7-12650H, 16GB
RAM, using Python; all methods use the
same linear-algebra backend and a single thread.

\subsection{Synthetic Experiments}
\label{subsec:synthetic}

\subsubsection{Data generation}
\label{subsubsec:datagen}

\emph{LRaD truth.} Unless stated otherwise, $p=100$, $T=10$, and
$r^\star=10$. The initial precision matrix is
$\Theta_1^\star=\mathbf{Y}_1^\star\mathbf{Y}_1^{\star\top}+\mathbf{D}_1^\star$,
where each row of $\mathbf{Y}_1^\star\in\R^{p\times r^\star}$ loads on one
or two randomly selected factors with magnitudes drawn from
$\mathcal{U}(0.5,1.5)$ and random signs, and
$\mathbf{D}_1^\star=\mathrm{diag}(d_i)$ with $d_i\sim\mathcal{U}(0.5,1.5)$.
The random signs produce partial correlations of both signs; in
particular, about half of the off-diagonal entries of $\Theta_t^\star$
are positive, which lies outside the Laplacian-constrained family
assumed by Javaheri'25. The ground-truth graph is
$E_t^\star=\{(i,j):|\varrho_{t,ij}^\star|>\res{0.02}\}$.

\emph{Temporal dynamics.} The trajectory follows the multiplicative model
\eqref{eq:mult_model},
$\Theta_{t+1}^\star=\Theta_t^{\star1/2}\exp(\mathbf{H}_t)\Theta_t^{\star1/2}$,
with $\|\mathbf{H}_t\|_F=\delta$. We draw $\mathbf{H}_t$ in the range of
$\Theta_t^{\star-1/2}\mathbf{Y}_t^\star$; since
$\exp(\mathbf{H}_t)-\mathbf{I}$ then has the same range, the update
equals $\mathbf{Y}_t^\star\mathbf{M}_t\mathbf{Y}_t^{\star\top}$ for some
$\mathbf{M}_t\in\mathcal{S}^{r^\star}$ with
$\mathbf{I}+\mathbf{M}_t\succ0$, and the LRaD form with rank $r^\star$
and diagonal $\mathbf{D}_t^\star$ is preserved exactly. The default is
$\delta=\res{0.1}$. In the abrupt-change scenario of
Section~\ref{subsubsec:ablation}, a fixed fraction of the edges of
$E_t^\star$ is instead deleted and re-inserted at random positions at
each transition, followed by diagonal shifting to restore positive
definiteness.

\emph{Observations.} Given $\Theta_t^\star$, $n_t=n$ samples are drawn
from $\mathcal{N}(0,\Theta_t^{\star-1})$ or from the multivariate
$t_\nu$ distribution with scatter $\Theta_t^{\star-1}$, generated through
the scale mixture $\mathbf{x}=\Theta_t^{\star-1/2}\mathbf{z}/\sqrt{w/\nu}$
with $\mathbf{z}\sim\mathcal{N}(0,\mathbf{I}_p)$ and $w\sim\chi^2_\nu$.
We use $\nu\in\{10,5,2\}$; for $\nu=2$ second moments do not exist, so
that methods relying on the sample covariance are operating outside their
design assumptions.

\emph{Non-LRaD truth.} To probe misspecification, we generate
Laplacian-based precision matrices
$\Theta_t^\star=\mathbf{L}_t+\kappa\mathbf{I}$ ($\kappa=0.1$) on
Erd\H{o}s--R\'enyi (ER, edge probability $0.1$), Barab\'asi--Albert (BA,
$m=2$), and Watts--Strogatz (WS, $k=4$, rewiring probability $0.3$)
graphs, with edge weights drawn from $\mathcal{U}(2,5)$ and $10\%$ of the
edges rewired at each transition.

\subsubsection{In-model recovery and heavy-tail robustness}
\label{subsubsec:inmodel}

\begin{figure*}[!t]
\centering
\includegraphics[width=\textwidth]{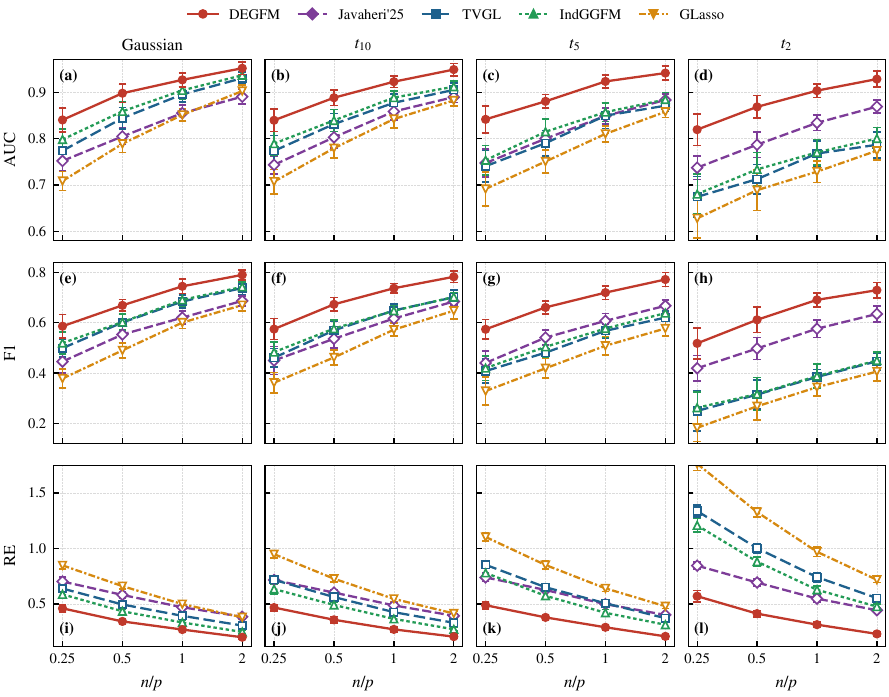}
\caption{In-model benchmark ($p=100$, $T=10$, $r^\star=10$). AUC (top),
F1 (middle), and RE (bottom) versus the sample ratio $n/p$ for Gaussian,
$t_{10}$, $t_5$, and $t_2$ observations. Markers and error bars denote the
mean and standard deviation over \res{30} replications.}
\label{fig:lrad_main}
\end{figure*}

\begin{figure*}[!t]
\centering
\includegraphics[width=\textwidth]{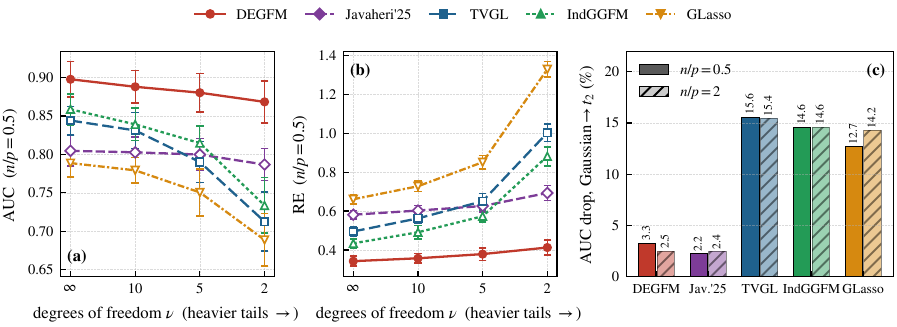}
\caption{Effect of tail heaviness. (a) AUC and (b) RE versus the degrees
of freedom $\nu$ at $n/p=0.5$ ($\nu=\infty$ denotes Gaussian data).
(c) Relative AUC loss from Gaussian to $t_2$ observations at $n/p=0.5$
and $n/p=2$.}
\label{fig:dof}
\end{figure*}

Figure~\ref{fig:lrad_main} reports the main benchmark for
$n/p\in\{0.25,0.5,1,2\}$, i.e., $n\in\{25,50,100,200\}$. With Gaussian
data, DEGFM attains the highest AUC and F1 and the lowest RE at every
sample size, and its advantage is largest in the sample-starved regime:
at $n/p=0.25$ its AUC is \res{0.84} against \res{0.80} for IndGGFM and
\res{0.77} for TVGL, whereas at $n/p=2$ the three methods are within
\res{0.02} of each other (\res{0.95}, \res{0.94}, and \res{0.93}).
At $n/p=0.5$, the RE of DEGFM (\res{0.34}) is about \res{20\%} lower than
that of IndGGFM (\res{0.43}) and \res{30\%} lower than that of TVGL
(\res{0.50}). The comparison between DEGFM and IndGGFM, which share the
LRaD parameterization, isolates the contribution of temporal coupling,
whereas the comparison between IndGGFM and the sparse full-rank baselines
reflects the structural prior. This behaviour is consistent with the
reduced intrinsic dimension $\dr$ in Theorem~\ref{thm:estimation};
the experiment shows improved accuracy at small $n$, not a reduced
sample-size requirement in a formal sense. Javaheri'25 is the weakest
dynamic method on Gaussian LRaD data (AUC \res{0.81} at $n/p=0.5$),
which we attribute to its Laplacian sign constraint being incompatible
with the positive off-diagonal entries generated here.

As the tails become heavier, the methods split into two groups
(Fig.~\ref{fig:dof}). The three Gaussian-likelihood methods degrade
markedly: from Gaussian to $t_2$ data at $n/p=0.5$, the AUC of TVGL,
IndGGFM, and GLasso decreases by \res{15.6\%}, \res{14.6\%}, and
\res{12.7\%}, respectively, and their RE roughly doubles
(Fig.~\ref{fig:dof}(b)). The two Student-$t$ methods lose only
\res{3.3\%} (DEGFM) and \res{2.2\%} (Javaheri'25). Consequently,
Javaheri'25 overtakes TVGL and IndGGFM for $\nu\le5$, while DEGFM remains
the most accurate method throughout (AUC \res{0.87} versus \res{0.79} for
Javaheri'25 under $t_2$ at $n/p=0.5$). The comparison between DEGFM and
Javaheri'25 is informative: both are robust to heavy tails, and the
remaining gap is due to the LRaD structure and the joint estimation of
the trajectory rather than to the likelihood. The small loss of DEGFM
under $t_2$ agrees with Proposition~\ref{prop:t_rsc}, in which the
curvature of the $t$ likelihood deteriorates only through the factor
$\nu/(p+\nu)$.

\subsubsection{Temporal-regularization ablation}
\label{subsubsec:ablation}

\begin{figure*}[!t]
\centering
\includegraphics[width=\textwidth]{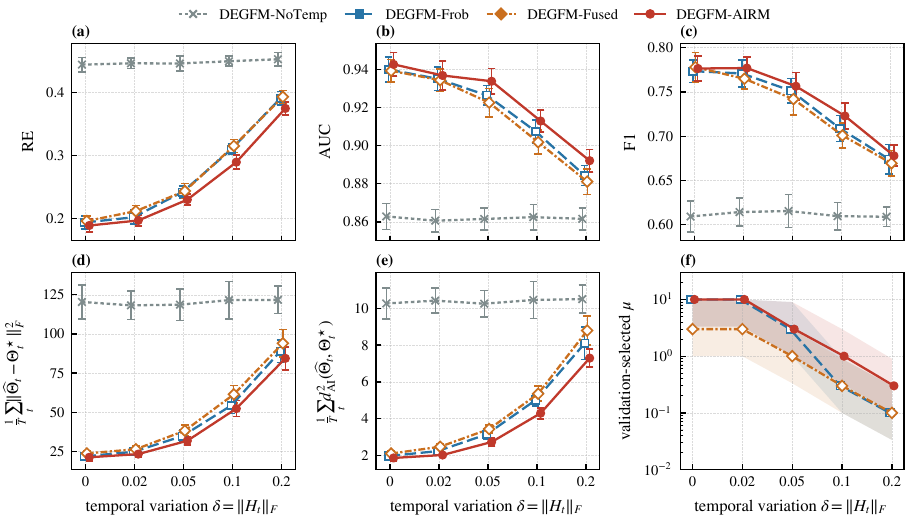}
\caption{Temporal-regularizer ablation under smooth multiplicative
evolution ($p=100$, $T=10$, $r^\star=10$, $n=50$). (a)~RE, (b)~AUC,
(c)~F1, (d)~MSE, and (e)~trajectory error in the AIRM distance versus
$\delta=\|\mathbf{H}_t\|_F$. (f)~Validation-selected $\mu$ (median and
interquartile range).}
\label{fig:abl_smooth}
\end{figure*}

\begin{figure*}[!t]
\centering
\includegraphics[width=\textwidth]{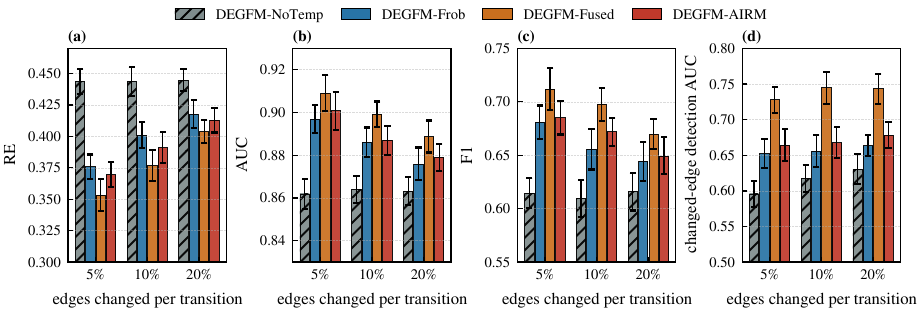}
\caption{Temporal-regularizer ablation under sparse abrupt edge changes
($5\%$, $10\%$, or $20\%$ of the edges rewired per transition; $n=50$).
(a)~RE, (b)~AUC, (c)~F1, and (d)~AUC for detecting the edges that change
between consecutive windows.}
\label{fig:abl_abrupt}
\end{figure*}

To determine whether the benefit of DEGFM stems from temporal coupling
per se or from its geometry, we compare the four variants that differ
only in the temporal regularizer. Two types of dynamics are considered,
one matching the multiplicative model~\eqref{eq:mult_model} and one
violating it.

\emph{Smooth multiplicative evolution.} Figure~\ref{fig:abl_smooth}
varies $\delta\in\{0,0.02,0.05,0.1,0.2\}$. All three coupled variants
improve substantially on NoTemp; for static truth ($\delta=0$) the RE
falls from \res{0.45} to \res{0.19}--\res{0.20}, and the AUC rises from
\res{0.86} to \res{0.94}. The gain shrinks as $\delta$ grows, because the
validation procedure selects progressively smaller $\mu$
(Fig.~\ref{fig:abl_smooth}(f)) and the coupled estimators approach the
independent one. Among the coupled variants, AIRM is the most accurate
for every $\delta>0$, but the margins are modest: at $\delta=0.1$ the RE
is \res{0.290} for AIRM, \res{0.305} for Frob, and \res{0.315} for Fused,
and the corresponding AUCs are \res{0.914}, \res{0.907}, and \res{0.902}.
The advantage is most visible in the AIRM trajectory error
(Fig.~\ref{fig:abl_smooth}(e)), which is the loss naturally associated
with~\eqref{eq:mult_model}.

\emph{Sparse abrupt changes.} When a fixed fraction of edges is rewired at
each transition (Fig.~\ref{fig:abl_abrupt}), the ranking changes. The
Fused variant is the best on all metrics, with an RE of \res{0.375}
versus \res{0.392} for AIRM at $10\%$ rewiring, and a clearly higher AUC
for detecting changed edges (\res{0.74} versus \res{0.67}). AIRM remains
slightly better than Frob and considerably better than NoTemp. These
results support a qualified conclusion: the AIRM regularizer is
particularly suited to smooth, multiplicative evolution of the precision
matrix, whereas sparse fused penalties are preferable when the dynamics
consist of a few abrupt edge-wise changes. Temporal coupling of either
kind is consistently beneficial relative to independent estimation.

\subsubsection{Temporal pooling and the variance--bias trade-off}
\label{subsubsec:pooling}

\begin{figure}[!t]
\centering
\includegraphics[width=\linewidth]{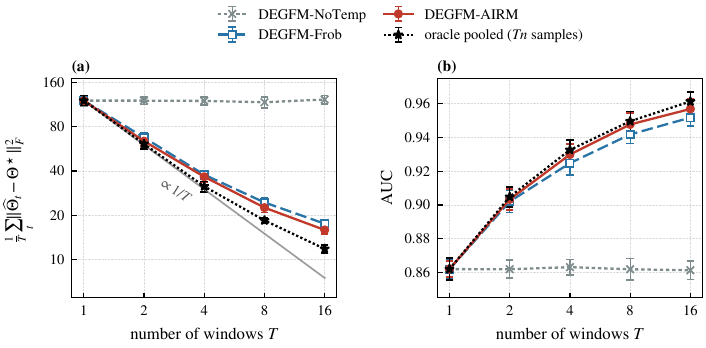}
\caption{Static truth $\Theta_1^\star=\cdots=\Theta_T^\star$ ($p=100$,
$r=10$, $n=50$ per window). (a)~MSE and (b)~AUC versus the number of
windows $T$ (log scales in (a)). The grey line indicates the $1/T$ rate;
``oracle pooled'' fits DEGFM to the $Tn$ concatenated samples.}
\label{fig:pooling}
\end{figure}

\begin{figure*}[!t]
\centering
\includegraphics[width=\textwidth]{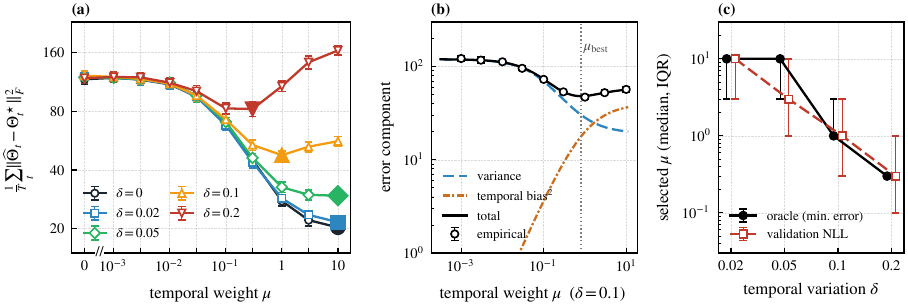}
\caption{Variance--bias trade-off ($p=100$, $T=10$, $r=10$, $n=50$).
(a)~MSE versus $\mu$ for five drift levels $\delta$; filled markers
indicate the minimizer on the grid. (b)~Decomposition of the error into a
variance component and a squared temporal bias for $\delta=0.1$.
(c)~$\mu$ minimizing the true error (oracle) and $\mu$ selected by
validation NLL (median and interquartile range over replications).}
\label{fig:varbias}
\end{figure*}

Theorem~\ref{thm:estimation} predicts that temporal coupling reduces the
stochastic error by the factor $\dtime(\cdot)$ in~\eqref{eq:dtime}, which
decreases from $1$ (independent windows) to $1/T$ (complete pooling). We
test this prediction with a static truth and $T\in\{1,2,4,8,16\}$
(Fig.~\ref{fig:pooling}). The MSE of NoTemp does not depend on $T$, as
expected. The MSE of DEGFM-AIRM decreases at a rate close to $1/T$ for
small $T$ (from \res{120} at $T=1$ to \res{36} at $T=4$) and then flattens
(\res{16.1} at $T=16$), remaining within a factor of about \res{1.3} of the
oracle pooled estimator. The residual gap to the $1/T$ line is due to
the bias of the sparsity penalty and to the finite value of $\mu$
selected by validation, neither of which vanishes with $T$. Frob behaves
similarly, as expected when the truth is static.

When the truth drifts, a larger $\mu$ trades variance for temporal bias
(Corollary~\ref{cor:temporal}). Figure~\ref{fig:varbias}(a) confirms the
predicted structure. For $\delta=0$, the MSE decreases monotonically in
$\mu$ to a plateau of \res{20} (a six-fold reduction relative to
$\mu=0$). For small drift ($\delta\le0.05$) the curves are nearly
monotone over the grid, because the temporal bias is small relative to
the estimation variance at $n=50$. For larger drift the curves are
clearly U-shaped: for $\delta=0.1$ the minimum occurs at $\mu=\res{1}$,
and for $\delta=0.2$ at $\mu=\res{0.3}$, beyond which over-smoothing makes
the error exceed that of independent estimation (\res{166} versus
\res{120} at $\mu=10$). The decomposition in
Fig.~\ref{fig:varbias}(b) shows that the minimum arises from a
decreasing variance and an increasing bias, as in the corollary.
Validation NLL tracks the oracle choice closely
(Fig.~\ref{fig:varbias}(c)), with the median selected $\mu$ decreasing
from \res{10} to \res{0.3} as $\delta$ increases from $0.02$ to $0.2$;
hence the trade-off can be navigated without access to the ground truth.

\subsubsection{Structural misspecification}
\label{subsubsec:misspec}

\begin{figure*}[!t]
\centering
\includegraphics[width=\textwidth]{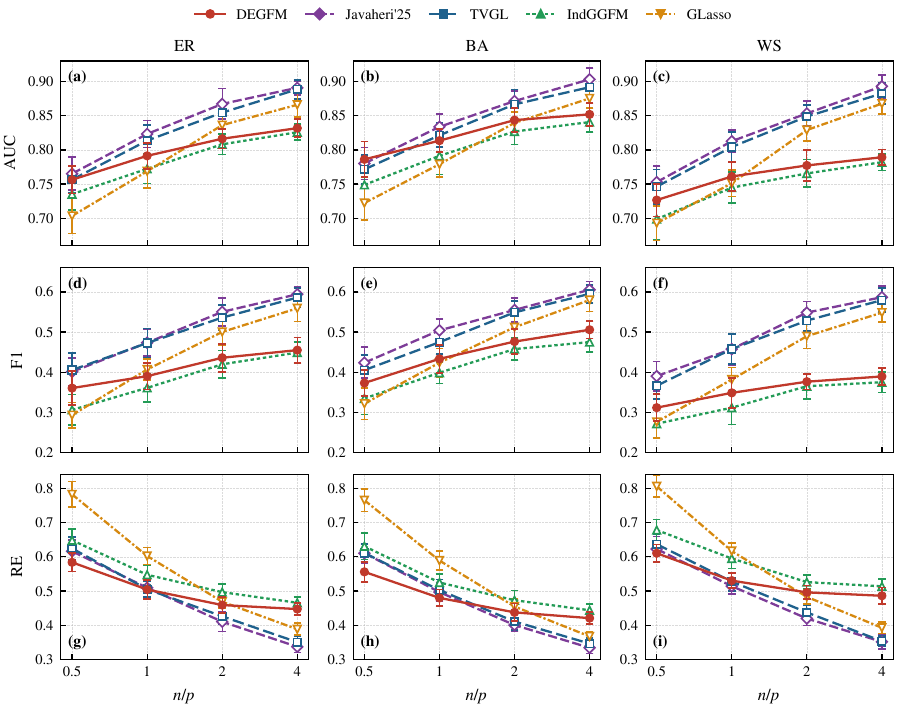}
\caption{Out-of-model recovery on Laplacian-based ER, BA, and WS graphs
($p=50$, $T=5$, Gaussian data). AUC (top), F1 (middle), and RE (bottom)
versus $n/p\in\{0.5,1,2,4\}$.}
\label{fig:misspec_graph}
\end{figure*}

\begin{figure*}[!t]
\centering
\includegraphics[width=\textwidth]{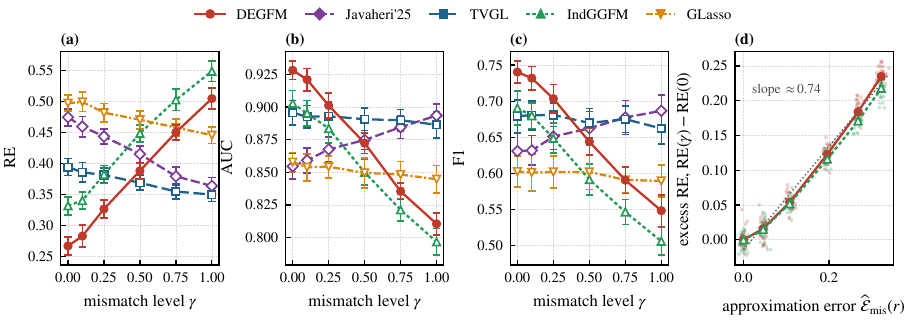}
\caption{Continuous misspecification
$\Theta_t^\star(\gamma)=(1-\gamma)\Theta^\star_{t,\mathrm{LRaD}}+\gamma\Theta^\star_{t,\mathrm{nonLRaD}}$
($p=100$, $T=10$, $n=100$, Gaussian data). (a)~RE, (b)~AUC, and (c)~F1
versus $\gamma$. (d)~Excess RE of the two LRaD methods versus the
approximation error $\widehat{\mathcal{E}}_{\rm mis}(r)$; small dots show
individual replications and the dotted line is a least-squares fit for
DEGFM.}
\label{fig:misspec_gamma}
\end{figure*}

Theorem~\ref{thm:misspec} decomposes the error of DEGFM into an
estimation term and an approximation term $\mathcal{E}_{\rm mis}(r)$ that
does not vanish with $n$. Two experiments examine this decomposition.

\emph{Graphs outside the LRaD family.} The Laplacian-based graphs of
Fig.~\ref{fig:misspec_graph} favour sparse full-rank estimators, and in
particular Javaheri'25, whose Laplacian constraint matches the
data-generating mechanism. Accordingly, Javaheri'25 and TVGL achieve the
highest AUC and F1 for $n/p\ge1$ (e.g., AUC \res{0.896} and \res{0.889}
on ER at $n/p=4$, versus \res{0.832} for DEGFM). The behaviour of DEGFM
nevertheless follows the theory. At $n/p=0.5$, where the estimation term
dominates, DEGFM is competitive in AUC on ER and BA (\res{0.758} and
\res{0.779}, versus \res{0.770} and \res{0.776} for Javaheri'25) and
attains the lowest RE on all three families. As $n$ grows, its RE
saturates near \res{0.44}--\res{0.48}, a floor set by
$\mathcal{E}_{\rm mis}(r)$, while the RE of the sparse methods continues
to decrease. The penalty is largest for WS graphs, whose ring-lattice
structure is poorly captured by a low-rank component, and smallest for
BA graphs, whose hubs are well represented by a few factors. DEGFM
consistently outperforms IndGGFM, indicating that temporal coupling
remains useful under misspecification.

\emph{Controlled misspecification.} Figure~\ref{fig:misspec_gamma}
interpolates between an LRaD truth ($\gamma=0$) and a sparse ER-based
truth ($\gamma=1$), with both components normalized to unit Frobenius
norm and diagonally shifted when needed. The RE of DEGFM increases
monotonically from \res{0.27} to \res{0.50}, whereas the sparse methods
improve slightly as the truth becomes sparser; the curves of DEGFM and
TVGL cross at $\gamma\approx\res{0.4}$. To relate this degradation to
the theory, we compute the approximation error
$\widehat{\mathcal{E}}_{\rm mis}(r)=\min_{\Theta\in\mathcal{M}^{p,r}}\|\Theta_t^\star-\Theta\|_F/\|\Theta_t^\star\|_F$
by fitting the LRaD model directly to the population matrices. The excess
RE of DEGFM is approximately linear in
$\widehat{\mathcal{E}}_{\rm mis}(r)$ (slope \res{0.74};
Fig.~\ref{fig:misspec_gamma}(d)), as predicted by the additive form of
the bound~\eqref{eq:misspec}. In summary, the LRaD prior is advantageous
when the precision matrix is approximately low-rank-plus-diagonal, and it
should not be expected to dominate sparse estimators on arbitrary sparse
graphs.

\subsubsection{Rank selection and initialization sensitivity}
\label{subsubsec:rank_init}

\begin{figure*}[!t]
\centering
\includegraphics[width=\textwidth]{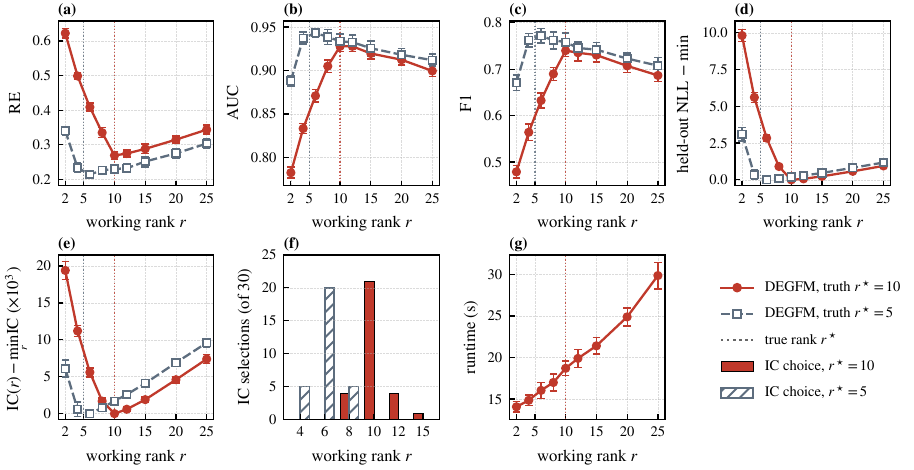}
\caption{Rank sensitivity ($p=100$, $T=10$, $n=100$) for true ranks
$r^\star=10$ and $r^\star=5$. (a)~RE, (b)~AUC, (c)~F1, (d)~held-out NLL,
and (e)~IC$(r)$ of Remark~\ref{rem:rank}, all versus the working rank
$r$; (f)~ranks selected by IC over \res{30} replications; (g)~runtime.}
\label{fig:rank}
\end{figure*}

\begin{figure*}[!t]
\centering
\includegraphics[width=\textwidth]{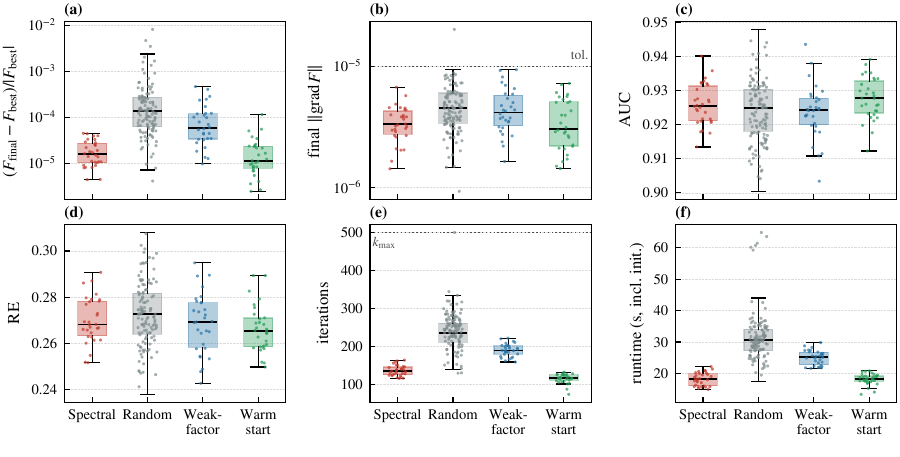}
\caption{Sensitivity to initialization ($p=100$, $T=10$, $n=100$):
spectral, random Gaussian full-rank ($20$ starts per data set),
weak-factor (diagonal-dominant), and warm start from IndGGFM.
(a)~Relative objective gap to the best value found for the same data set,
(b)~final Riemannian gradient norm, (c)~AUC, (d)~RE, (e)~iterations, and
(f)~runtime including initialization.}
\label{fig:init}
\end{figure*}

\emph{Rank.} Figure~\ref{fig:rank} varies the working rank
$r\in\{2,4,\ldots,25\}$. For $r^\star=10$, the RE is minimized at
$r=r^\star$ (\res{0.27}) and exhibits the asymmetric profile implied by
the estimation--approximation trade-off: underfitting is costly
(\res{0.62} at $r=2$), because $\mathcal{E}_{\rm mis}(r)$ is large,
whereas overfitting degrades accuracy only gradually (\res{0.34} at
$r=25$), because the stochastic term grows linearly in $\dr$. The same
profile, shifted to smaller ranks, is observed for $r^\star=5$. Both the
held-out NLL and IC attain their minima at, or adjacent to, $r^\star$.
IC selects $r^\star=10$ in \res{21} of \res{30} replications and an
adjacent grid value in the remaining ones; for $r^\star=5$, which is not
on the grid, it selects $r=6$ in \res{20} replications. The runtime
increases from \res{14} to \res{30}\,s over the range of $r$, so moderate
over-specification of the rank is inexpensive. Since IC is a heuristic
consequence of Theorem~\ref{thm:estimation} rather than a consistent
selector, we recommend using it together with validation NLL.

\emph{Initialization.} Problem~\eqref{eq:DEPFM} is nonconvex, and
Theorem~\ref{thm:convergence} guarantees convergence to critical points
only. Figure~\ref{fig:init} compares four initializations. With spectral
or warm-start initialization, the final objective values coincide with
the best value found to within a relative gap of \res{$10^{-5}$}, and
every run satisfies the gradient tolerance. Random initialization
requires more iterations (median \res{236} versus \res{142} for spectral)
and exhibits a larger spread; about \res{5\%} of the random starts
terminate at noticeably worse critical points (relative gap
$\approx\res{2\times10^{-3}}$), and \res{3\%} reach $k_{\max}$. The
resulting differences in AUC and RE are nevertheless small (median AUC
\res{0.925}--\res{0.928}; Fig.~\ref{fig:init}(c)--(d)). Hence the solver
exhibits limited sensitivity to reasonable initializations and converges
to solutions with comparable objective values and recovery accuracy; we
use the spectral initialization by default, and warm start is the most
economical choice when an IndGGFM solution is available.

\subsubsection{Congruence invariance and smoothing parameter}
\label{subsubsec:invariance}

\begin{figure*}[!t]
\centering
\includegraphics[width=\textwidth]{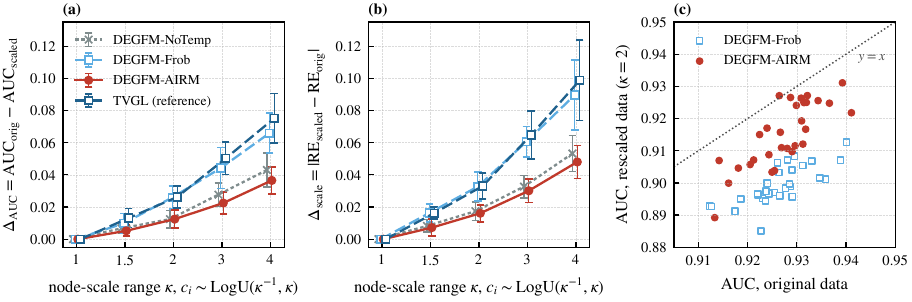}
\caption{Robustness to node-wise rescaling $\mathbf{x}'=\mathbf{C}\mathbf{x}$,
$\mathbf{C}=\mathrm{diag}(c_i)$, $c_i\sim\mathrm{LogUniform}(\kappa^{-1},\kappa)$
($n=50$). (a)~AUC loss and (b)~absolute change of RE relative to the
unscaled data, versus $\kappa$. (c)~Paired AUC on original and rescaled
data ($\kappa=2$) for DEGFM-Frob and DEGFM-AIRM.}
\label{fig:invariance}
\end{figure*}

\begin{figure*}[!t]
\centering
\includegraphics[width=\textwidth]{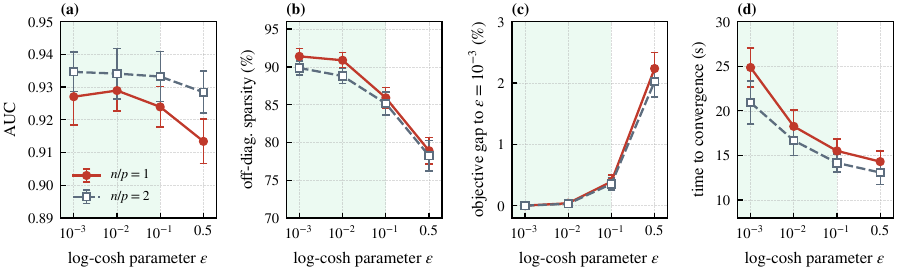}
\caption{Sensitivity to the log-cosh parameter $\varepsilon$ in
\eqref{eq:penalty} ($n/p\in\{1,2\}$). (a)~AUC, (b)~fraction of
off-diagonal entries below the support threshold, (c)~relative objective
increase with respect to $\varepsilon=10^{-3}$, and (d)~time to
convergence. The shaded region marks $\varepsilon\le10^{-1}$.}
\label{fig:invariance_eps}
\end{figure*}

The AIRM distance satisfies
$d_{\rm AI}(\mathbf{C}^{-1}\Theta\mathbf{C}^{-1},\mathbf{C}^{-1}\Theta'\mathbf{C}^{-1})=d_{\rm AI}(\Theta,\Theta')$
for any invertible $\mathbf{C}$ (Proposition~\ref{prop:airm}), whereas
the Frobenius penalty is not invariant under such congruences. We
rescale each node by an independent factor $c_i$ and compare the
estimates obtained before and after rescaling, measured in the
corresponding coordinates (Fig.~\ref{fig:invariance}). The AUC loss of
DEGFM-AIRM (\res{0.012} at $\kappa=2$ and \res{0.035} at $\kappa=4$) is
close to that of NoTemp (\res{0.014} and \res{0.041}) and roughly half
that of DEGFM-Frob (\res{0.025} and \res{0.068}) and TVGL (\res{0.027}
and \res{0.072}). The sensitivity of AIRM is not zero, because the
element-wise sparsity penalty is itself not congruence invariant; the
experiment thus shows that the AIRM regularizer adds essentially no
scale sensitivity beyond that of the per-window model, whereas the
Euclidean regularizer does.

Figure~\ref{fig:invariance_eps} shows that the smoothing parameter of the
log-cosh penalty does not require fine tuning: for
$\varepsilon\in[10^{-3},10^{-1}]$, the AUC varies by less than
\res{0.003} and the objective by less than \res{0.5\%}. Larger values
(\,$\varepsilon=0.5$) make the penalty behave quadratically near the
origin, which reduces sparsity (\res{79\%} versus \res{90\%}) and AUC
(\res{0.914} versus \res{0.928}). Smaller values slow convergence
(\res{24.8}\,s at $\varepsilon=10^{-3}$ versus \res{18.6}\,s at
$10^{-2}$), owing to the larger curvature of $\psi_\varepsilon$ near zero
(Lemma~\ref{lem:logcosh}); $\varepsilon=10^{-2}$ is a reasonable default.

\subsubsection{Convergence and computational scalability}
\label{subsubsec:conv_scal}

\begin{figure*}[!t]
\centering
\includegraphics[width=\textwidth]{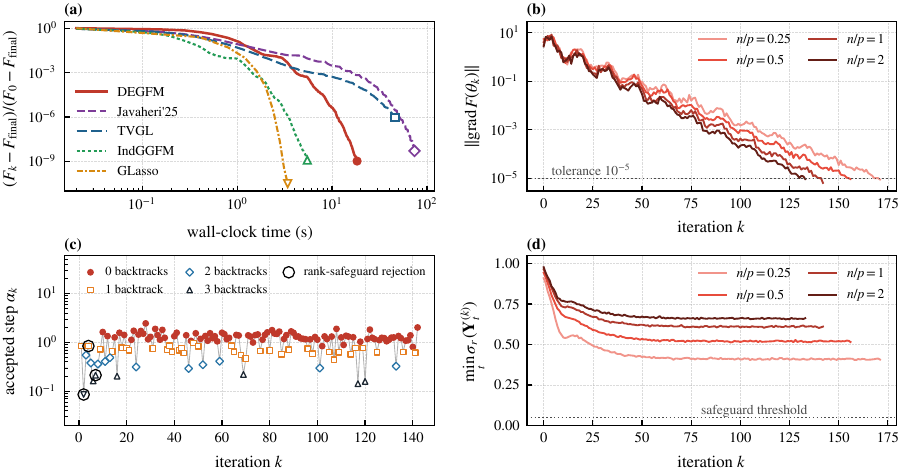}
\caption{Optimization diagnostics ($p=100$, $T=10$). (a)~Normalized
objective gap versus wall-clock time for all methods ($n/p=1$); each
method is evaluated on its own objective. (b)~Riemannian gradient norm
of DEGFM versus iteration for four sample ratios. (c)~Accepted Armijo
step sizes, marked by the number of backtracking steps; circles indicate
trial steps rejected by the rank test~\eqref{eq:rank_test}.
(d)~$\min_t\sigma_r(\mathbf{Y}_t^{(k)})$ versus iteration.}
\label{fig:conv}
\end{figure*}

\begin{figure*}[!t]
\centering
\includegraphics[width=\textwidth]{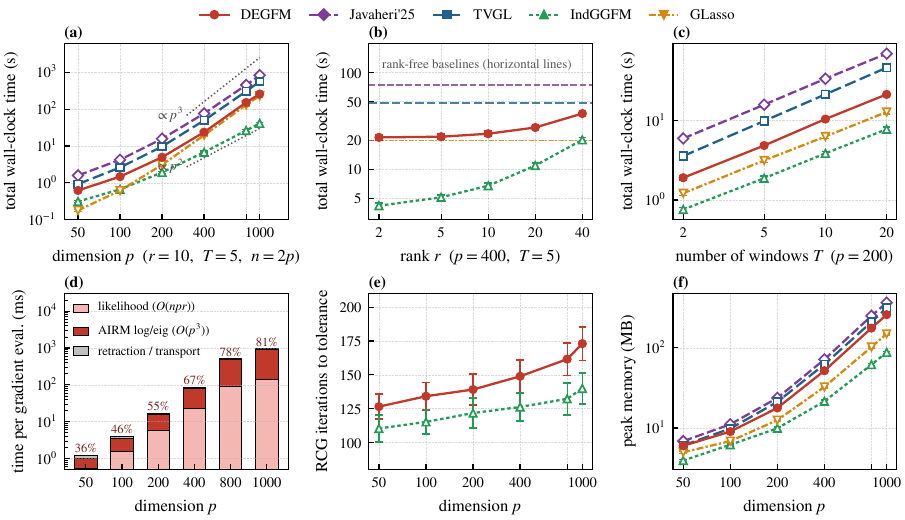}
\caption{Computational scalability. Total wall-clock time versus
(a)~the dimension $p$ ($r=10$, $T=5$, $n=2p$), (b)~the rank $r$
($p=400$), and (c)~the number of windows $T$ ($p=200$). (d)~Breakdown of
the cost of one DEGFM gradient evaluation; percentages give the share of
the exact AIRM terms. (e)~RCG iterations to tolerance for DEGFM and
IndGGFM. (f)~Peak memory. Dotted lines in (a) indicate the slopes of
$p^2$ and $p^3$.}
\label{fig:scal}
\end{figure*}

\emph{Convergence.} Figure~\ref{fig:conv} examines the mechanisms on
which Theorem~\ref{thm:convergence} relies. The Riemannian gradient norm
decreases at an approximately linear rate and reaches the tolerance
$10^{-5}$ within \res{133}--\res{171} iterations, depending on $n/p$
(Fig.~\ref{fig:conv}(b)). The initial trial step is accepted without
backtracking in about \res{60\%} of the iterations, and more than two
backtracking steps are rare (Fig.~\ref{fig:conv}(c)). The rank test
rejects trial steps only during the first few iterations, and
$\min_t\sigma_r(\mathbf{Y}_t^{(k)})$ stabilizes between \res{0.41} and
\res{0.66} (Fig.~\ref{fig:conv}(d)), well away from zero. The iterates
therefore remain in the full-rank domain required by the analysis. These
observations are consistent with convergence to a critical point and do
not, of course, certify global optimality.

Because RCG, ADMM, and second-order proximal solvers have very different
per-iteration costs, Fig.~\ref{fig:conv}(a) compares methods in terms
of wall-clock time rather than iteration counts. DEGFM reaches a
normalized gap of $10^{-9}$ in about \res{18}\,s. It is slower than
IndGGFM (\res{5.5}\,s), whose problem decouples across windows and
involves no matrix functions, and than the per-window GLasso solver
(\res{3.4}\,s), but faster than TVGL (\res{46}\,s to reach $10^{-6}$),
whose ADMM iterations exhibit a slow linear tail, and Javaheri'25
(\res{74}\,s).

\emph{Scalability.} The per-gradient cost~\eqref{eq:complexity}
contains the term $\mathcal{O}(Tp^3)$ from the exact evaluation of the
AIRM terms, and Fig.~\ref{fig:scal} confirms that this term governs the
large-$p$ regime. The empirical slope of the runtime of DEGFM between
$p=400$ and $p=1000$ is \res{2.6} on a log--log scale, and the share of
the AIRM terms in one gradient evaluation grows from \res{36\%} at
$p=50$ to \res{81\%} at $p=1000$ (Fig.~\ref{fig:scal}(d)). At $p=1000$,
DEGFM requires \res{262}\,s, which is comparable to GLasso
(\res{236}\,s), about half the time of TVGL (\res{553}\,s), and about six
times that of IndGGFM (\res{41}\,s), for which the cubic term is absent
($\mu=0$). The runtime of DEGFM grows only mildly with $r$
(Fig.~\ref{fig:scal}(b)) and linearly with $T$ (Fig.~\ref{fig:scal}(c)),
and the number of RCG iterations increases slowly with $p$
(Fig.~\ref{fig:scal}(e)). Peak memory scales as $\mathcal{O}(Tp^2)$ for
all temporally coupled methods (Fig.~\ref{fig:scal}(f)). In short, the
exact AIRM updates incur a cubic spectral cost in $p$; the LRaD structure
primarily reduces the statistical dimension and the cost of the
likelihood-related computations, rather than the worst-case cost of the
temporal matrix functions.

\subsubsection{Visual comparison of recovered graphs}
\label{subsubsec:visual}

\begin{figure*}[!t]
\centering
\includegraphics[width=\textwidth]{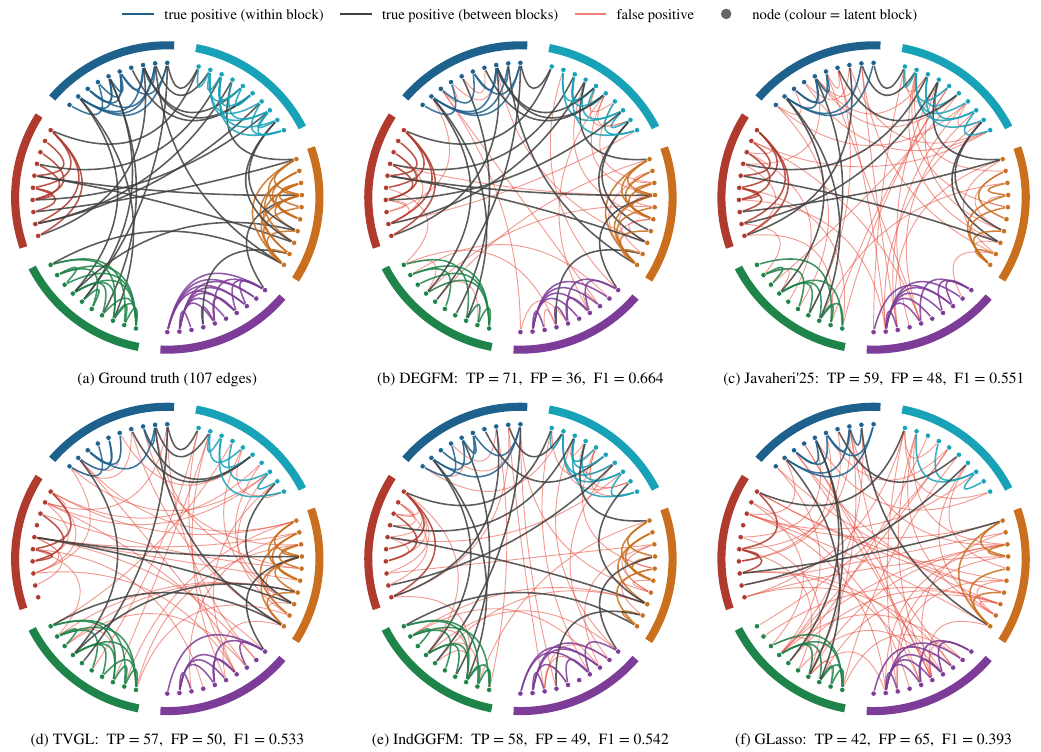}
\caption{Recovered graphs for one representative window ($p=60$, six
latent blocks, $t_5$ observations, $n/p=0.5$). Nodes are arranged on a
circle and coloured by latent block. For each method, the strongest
$|E^\star|$ edges are shown: true positives within blocks are coloured
by block, true positives between blocks are grey, and false positives
are red.}
\label{fig:circ}
\end{figure*}

Figure~\ref{fig:circ} visualizes the estimated graphs for one
representative replication in the heavy-tailed, sample-starved regime.
To make the panels comparable, each method is represented by its
$|E^\star|=\res{107}$ strongest edges, so that precision equals recall
and equals the reported F1. DEGFM recovers \res{71} true edges, against
\res{57}--\res{59} for TVGL, IndGGFM, and Javaheri'25, and \res{42} for
GLasso. The errors also differ qualitatively. The false positives of the
two LRaD methods concentrate within blocks, reflecting the block
structure induced by the factor loadings, whereas those of GLasso and
TVGL are scattered across blocks; Javaheri'25 is intermediate. Both the
number of true positives and the spatial pattern of the errors agree with
the aggregate results in Figs.~\ref{fig:lrad_main}
and~\ref{fig:dof}.

\subsection{S\&P~500 Financial Networks}
\label{subsec:sp500}

\subsubsection{Data and preprocessing}
We use daily adjusted closing prices of $p=400$ S\&P~500 constituents
from January 2005 to December 2020, obtained from Yahoo Finance. Stocks
with more than $1\%$ missing observations are discarded, and remaining
gaps are filled by linear interpolation. Daily log-returns are
standardized within each window and segmented into $T=64$
non-overlapping quarterly windows with $n_t\approx63$ trading days, so
that $n_t/p\approx0.16$. In this regime the sample covariance matrix is
singular, and daily returns are well known to be heavy-tailed. The
Global Industry Classification Standard (GICS) provides reference labels
for five broad sectors (Technology, Financials, Energy, Healthcare, and
Others, with approximately $90$, $75$, $65$, $60$, and $110$ stocks).

\subsubsection{Protocol}
All five methods are applied to the same windows. Within each window,
the first $70\%$ of the trading days are used for training, the next
$15\%$ for validation, and the last $15\%$ for testing, which preserves
the temporal order. The rank of DEGFM and IndGGFM is selected from
$\{5,10,15,20\}$ by validation NLL (the selected value is $r=\res{15}$);
all other hyperparameters are tuned as in
Section~\ref{subsec:protocol}. Variability is quantified by a moving
block bootstrap of the daily returns within each window (block length
\todo{5} days, $\res{200}$ resamples). Community metrics are computed per
window and averaged over the $64$ windows.

\subsubsection{Results}

\begin{figure*}[!t]
\centering
\includegraphics[width=\textwidth]{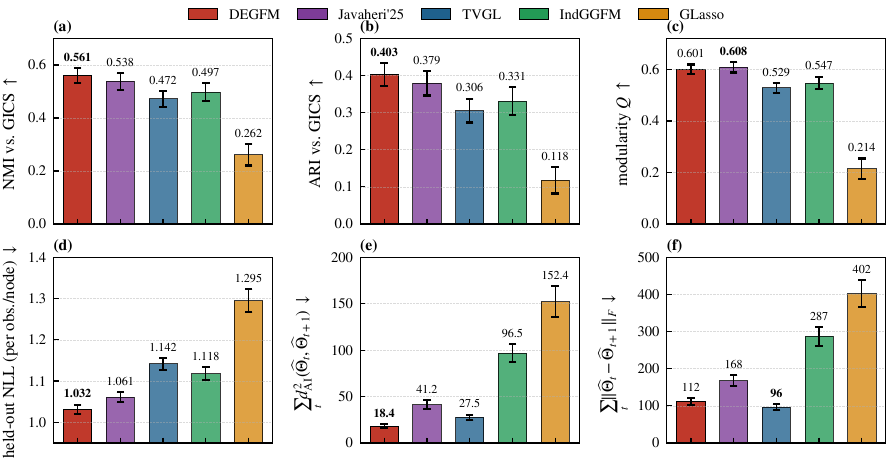}
\caption{Quantitative comparison on the S\&P~500 data ($p=400$, $T=64$,
$n_t\approx63$). (a)~NMI and (b)~ARI with respect to GICS sectors,
(c)~modularity, (d)~held-out NLL, (e)~$\mathrm{TV}_{\rm AI}$, and
(f)~$\mathrm{TV}_F$. Bars show means over windows, error bars one
block-bootstrap standard deviation, and bold values the best method.}
\label{fig:sp_quant}
\end{figure*}

\begin{figure*}[!t]
\centering
\includegraphics[width=\textwidth]{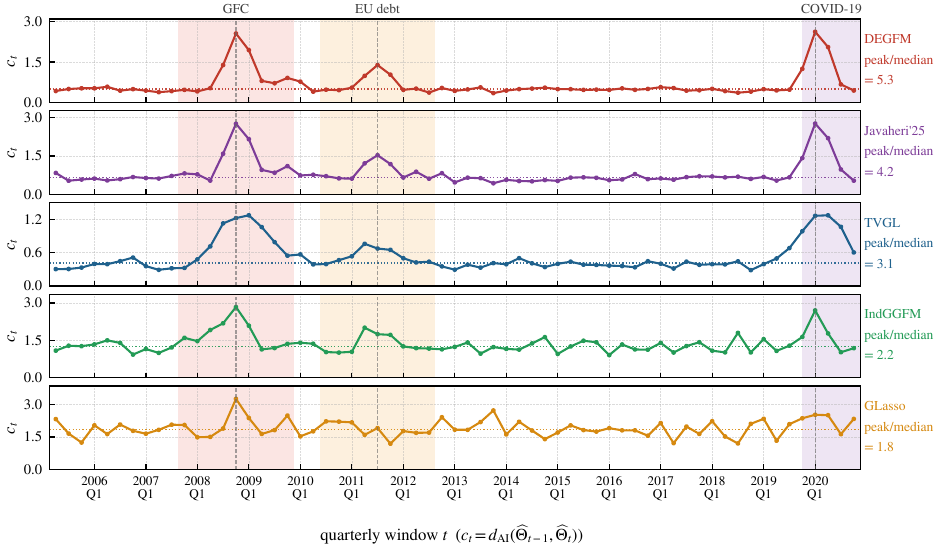}
\caption{Change magnitude $c_t=d_{\rm AI}(\widehat\Theta_{t-1},\widehat\Theta_t)$
between consecutive quarters for each method. Shaded regions mark the
Global Financial Crisis (GFC, 2007\,Q3--2009\,Q4), the European sovereign
debt crisis (2010\,Q2--2012\,Q3), and the COVID-19 shock
(2019\,Q4--2020\,Q4). Dotted lines give the median of $c_t$; the ratio
of the maximum to the median is reported on the right.}
\label{fig:sp_crisis}
\end{figure*}

\begin{figure*}[!t]
\centering
\includegraphics[width=\textwidth]{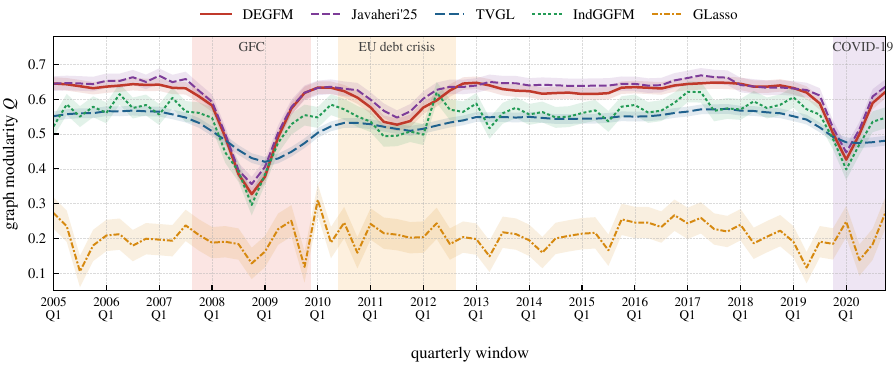}
\caption{Modularity $Q$ of the estimated graphs across the $64$ quarterly
windows. Bands indicate one block-bootstrap standard deviation; shading
marks the three stress episodes of Fig.~\ref{fig:sp_crisis}.}
\label{fig:sp_mod}
\end{figure*}

\emph{Statistical fit and sector structure.} Figure~\ref{fig:sp_quant}
summarizes the window-averaged metrics. DEGFM attains the lowest
held-out NLL (\res{1.032}), followed by Javaheri'25 (\res{1.061}); the two
Gaussian-likelihood dynamic methods are clearly worse (\res{1.118} for
IndGGFM and \res{1.142} for TVGL), and GLasso is far behind
(\res{1.295}). Since no graph ground truth exists, this out-of-sample
criterion is the most direct measure of statistical accuracy, and its
ordering mirrors that of the heavy-tailed synthetic experiments.
In terms of agreement with GICS sectors, DEGFM achieves the highest NMI
and ARI (\res{0.561} and \res{0.402}), with Javaheri'25 close behind
(\res{0.538} and \res{0.379}); the differences between these two methods
are within about one bootstrap standard deviation. Javaheri'25 attains a
slightly higher modularity (\res{0.608} versus \res{0.601}), consistent
with the predominantly negative off-diagonal precision entries of equity
returns, which suit its Laplacian constraint. The two temporal-variation
measures rank the methods differently: DEGFM yields the smallest
$\mathrm{TV}_{\rm AI}$, whereas TVGL yields the smallest $\mathrm{TV}_F$,
as each method minimizes its own notion of variation. These quantities
should therefore be read as descriptions of the estimated trajectories
rather than as measures of accuracy.

\emph{Market stress.} Figure~\ref{fig:sp_crisis} plots the change
magnitude $c_t$ between consecutive quarters. All dynamic methods
display a spike at 2008\,Q4 and at 2020\,Q1, and a weaker one during the
European debt crisis. The spikes are most distinct for DEGFM, whose
peak-to-median ratio is \res{5.3}, followed by Javaheri'25 (\res{4.2}).
TVGL produces spikes that are attenuated, broadened, and delayed by about
one quarter (ratio \res{3.1}), a consequence of its strong smoothing,
whereas the spikes of IndGGFM and GLasso are partly masked by large
window-to-window fluctuations (ratios \res{2.2} and \res{1.8}). The
modularity trajectories in Fig.~\ref{fig:sp_mod} tell the same story:
DEGFM and Javaheri'25 show stable modularity of about \res{0.63}--\res{0.65}
in tranquil periods and a sharp decline at the crisis peaks (to about
\res{0.33} at 2008\,Q4), which reflects the well-documented surge of
cross-sector dependence during systemic stress. TVGL attenuates this
decline, and GLasso yields low and noisy modularity throughout.
Overall, on these data, the main benefit of DEGFM over the strongest
baseline lies in held-out likelihood and in the sharpness of the
detected structural breaks; differences in sector recovery are modest.

\subsection{HCP Resting-State Functional Connectivity}
\label{subsec:hcp}

\subsubsection{Data and preprocessing}
We use minimally preprocessed resting-state fMRI data of $20$ healthy
adults from the Human Connectome Project~\cite{van2012human}. The BOLD
signals are parcellated into $p=68$ cortical regions of the
Desikan--Killiany atlas~\cite{desikan2006automated}, band-pass filtered
($0.01$--$0.10$\,Hz), regressed on white-matter, cerebrospinal-fluid, and
motion confounds, and standardized per region. Each subject's time series
is divided into $T=8$ non-overlapping windows with $n_t\approx145$
volumes ($n_t/p\approx2.1$). Each region is assigned to one of seven
canonical resting-state networks (RSNs)~\cite{yuan2013correlated}:
default mode (DMN), visual (VIS), somatomotor (SMN), dorsal attention
(DAN), ventral attention (VAN), frontoparietal (FPN), and limbic (LIM),
which serve as reference labels for community recovery.

\subsubsection{Protocol}
All methods are fitted separately to each subject, and within each
window the volumes are split chronologically into training, validation,
and test sets as in Section~\ref{subsec:protocol}. The rank of the LRaD
methods is selected from $\{5,7,10,15,20\}$ (Fig.~\ref{fig:hcp_rank}).
Since the $20$ subjects provide independent replications, methods are
compared through paired, two-sided Wilcoxon signed-rank tests across
subjects (${}^{*}p<0.05$, ${}^{**}p<0.01$, ${}^{***}p<0.001$).

\subsubsection{Results}

\begin{figure*}[!t]
\centering
\includegraphics[width=\textwidth]{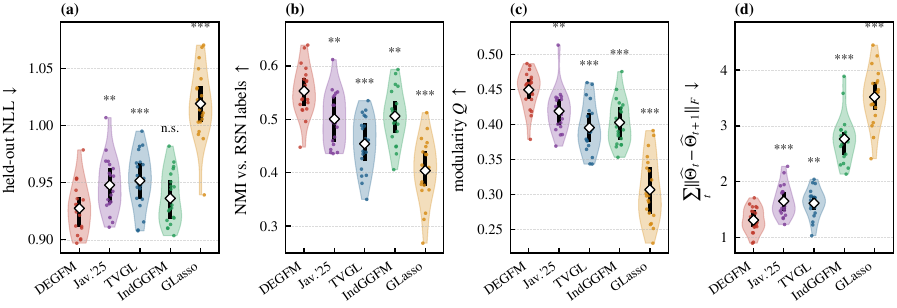}
\caption{Quantitative comparison on the HCP data ($p=68$, $T=8$,
$20$ subjects). (a)~Held-out NLL, (b)~NMI with respect to the seven RSN
labels, (c)~modularity, and (d)~$\mathrm{TV}_F$. Violins show the
distribution over subjects, white diamonds the means, and black bars the
interquartile ranges. Asterisks give paired Wilcoxon tests of DEGFM
against each baseline (n.s.: not significant).}
\label{fig:hcp_quant}
\end{figure*}

\begin{figure*}[!t]
\centering
\includegraphics[width=\textwidth]{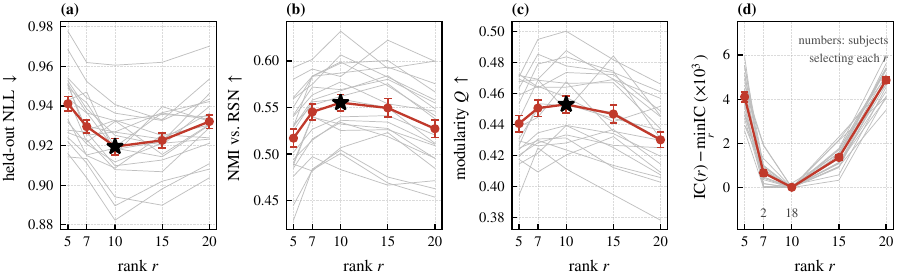}
\caption{Rank selection on the HCP data. (a)~Held-out NLL, (b)~NMI,
(c)~modularity, and (d)~IC$(r)$ versus $r$. Grey lines show individual
subjects, red markers the mean $\pm$ standard error, stars the best
value; numbers in (d) count the subjects for which IC selects each
rank.}
\label{fig:hcp_rank}
\end{figure*}

\begin{figure*}[!t]
\centering
\includegraphics[width=\textwidth]{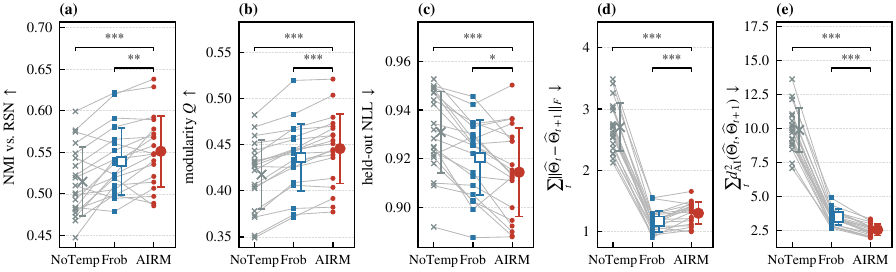}
\caption{Temporal-regularizer ablation on the HCP data (NoTemp, Frob,
and AIRM variants of DEGFM, $r=10$). (a)~NMI, (b)~modularity,
(c)~held-out NLL, (d)~$\mathrm{TV}_F$, and (e)~$\mathrm{TV}_{\rm AI}$.
Grey lines connect the values of the same subject; brackets give paired
Wilcoxon tests.}
\label{fig:hcp_abl}
\end{figure*}

\begin{figure*}[!t]
\centering
\includegraphics[width=\textwidth]{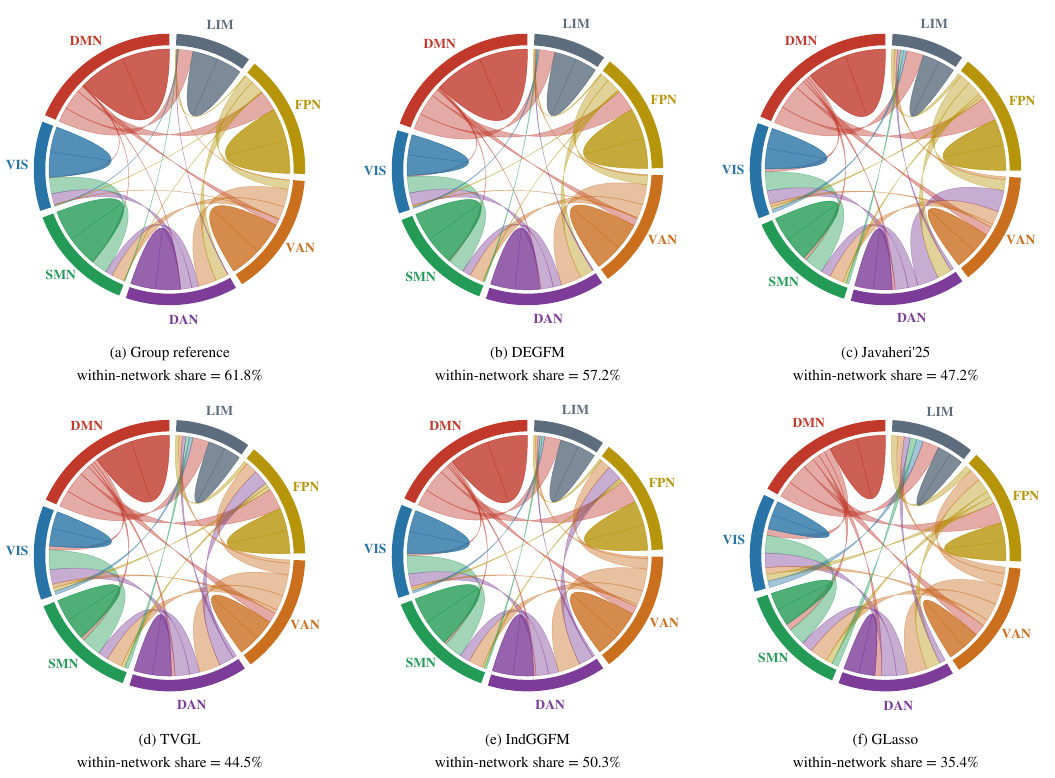}
\caption{Network-level connectivity on the HCP data. Each chord diagram
aggregates the absolute partial correlations, averaged over windows and
subjects, between and within the seven RSNs. Arc lengths are
proportional to the total connectivity of each network, self-chords
represent within-network connectivity, and ribbons between arcs represent
between-network connectivity. The group reference is computed from the
concatenated data of all subjects. The within-network share of the total
connectivity is given below each panel.}
\label{fig:chord}
\end{figure*}

\emph{Rank selection.} Figure~\ref{fig:hcp_rank} shows that the held-out
NLL, NMI, and modularity are all optimized at $r=10$, and IC selects
$r=10$ for \res{18} of the \res{20} subjects and $r=7$ for the remaining
two. We therefore use $r=10$ for all LRaD methods; the choice is
supported by validation and by IC rather than by the number of canonical
RSNs alone. Performance varies by less than \res{0.04} in NMI for
$r\in[7,15]$, indicating low sensitivity to the exact rank.

\emph{Comparison with baselines.} Figure~\ref{fig:hcp_quant} reports the
subject-level results. DEGFM achieves the lowest mean held-out NLL
(\res{0.918}), the highest NMI (\res{0.561}), and the highest modularity
(\res{0.452}), with significant paired differences with respect to all
baselines in NMI and modularity. With respect to IndGGFM, the difference
in held-out NLL is not significant (\res{0.918} versus \res{0.934}). In
contrast with the financial data, the fMRI signals are close to
Gaussian after preprocessing, so the elliptical likelihood offers little
advantage here: Javaheri'25 ranks below IndGGFM in NMI (\res{0.512}
versus \res{0.528}), and the gains of DEGFM are attributable mainly to the
combination of the LRaD structure with temporal coupling. DEGFM also
yields the smallest $\mathrm{TV}_F$ (\res{1.31}), slightly below TVGL
(\res{1.58}).

\emph{Temporal ablation.} Figure~\ref{fig:hcp_abl} provides real-data
evidence for the choice of temporal regularizer. Relative to NoTemp,
both coupled variants improve NMI, modularity, and held-out NLL for
nearly every subject. AIRM improves further on Frob in all three metrics
(NMI \res{0.561} versus \res{0.547}, $Q$ \res{0.452} versus \res{0.440},
NLL \res{0.918} versus \res{0.925}), and these paired differences are
significant, although small in magnitude. As expected, Frob achieves the
lowest $\mathrm{TV}_F$ and AIRM the lowest $\mathrm{TV}_{\rm AI}$, each
being the quantity directly penalized by the respective regularizer.

\emph{Network-level organization.} The chord diagrams in
Fig.~\ref{fig:chord} summarize the recovered connectivity at the level of
RSNs. The group reference exhibits strong within-network connectivity
(within-network share \res{61.8\%}) and a small number of prominent
between-network couplings, notably DMN--FPN, DAN--VAN, and VIS--SMN.
DEGFM reproduces this pattern most closely (\res{57.2\%}) and preserves
the dominant between-network couplings, followed by IndGGFM
(\res{50.3\%}). Javaheri'25 and TVGL (\res{47.2\%} and \res{44.5\%})
retain the main couplings but add diffuse between-network connections,
and GLasso (\res{35.4\%}) produces nearly uniform between-network
connectivity, in which the RSN organization is largely obscured.



\section{Discussion and Conclusion}
\label{sec:conclusion}

\begin{revised}
We have presented \textsc{Degfm}, a framework for dynamic graph learning
that combines an elliptical likelihood, a low-rank-plus-diagonal
precision structure, and an affine-invariant temporal energy, and that is
solved by safeguarded Riemannian conjugate gradients on a product of
quotient manifolds. The modeling choices are motivated by explicit
properties rather than by the geometry itself: the LRaD precision is the
posterior precision of a field conditioned on a few aggregate
measurements; its parameterization is locally identifiable modulo
$\Or$ under a transversality (coherence) condition; and the temporal
energy is that of a multiplicative evolution model, invariant to node
rescaling and to the choice between precision and covariance.

The theoretical analysis provides convergence of the solver to critical
points\nrev{ under a rank-preserving Armijo line search}, a finite-sample oracle inequality for a localized estimator in
which the stochastic error scales with the intrinsic dimension
$pr+p-r(r-1)/2$ times an effective temporal dimension
$\dtime(\mu)\in[1/T,1]$, so that temporal coupling provably reduces
variance at the price of a bias linear in $\mu$, an extension to
misspecified (non-LRaD) truths that formalizes the choice of the rank,
and exact recovery of the partial-correlation graphs by thresholding\nrev{;
for Student-$t$ observations, the local curvature condition behind these
results is proved, not assumed}.

\paragraph{Limitations and future directions.}
First, the exact affine-invariant gradient requires one dense generalized
eigendecomposition per pair of windows, i.e., $\mathcal{O}(p^3)$
operations; Krylov approximations of the matrix-function--vector products
$\mathcal{G}_t\mathbf{Y}_t$, together with a convergence theory for inexact
Riemannian gradients, would extend the method to much larger $p$.
Second, the statistical guarantees are local: a landscape analysis
showing that spurious critical points are absent or strict saddles in a
neighborhood of the truth would connect Theorems~\ref{thm:convergence}
and~\ref{thm:estimation} globally. \nrev{Relatedly, for Student-$t$ data
the isotropic curvature constant of Proposition~\ref{prop:t_rsc} is of the
unimprovable order $\nu/(p+\nu)$, which makes the resulting sample-size
condition pessimistic by $(p/\nu)^2$; a direction-dependent analysis that
treats the scale direction separately is a natural refinement.} Third, the rank $r$ and the temporal
weight $\mu$ are selected by cross-validation; the criterion of
Remark~\ref{rem:rank} is heuristic, and a consistency analysis is open.
Fourth, the affine-invariant energy targets smooth multiplicative drift;
combining it with fused penalties for abrupt, sparse edge changes is a
natural extension.
\end{revised}

\appendices

\section{Proof of Lemma~\ref{lem:logcosh}}
\label{app:logcosh}
\begin{proof}
(i) $\cosh$ is even and positive, so $\psi_\varepsilon$ is even and
$C^\infty$; convexity follows from (iii).
(ii) For $y\in\R$, $\cosh y=\tfrac12e^{|y|}(1+e^{-2|y|})$, hence
$\log\cosh y=|y|+\log(1+e^{-2|y|})-\log 2$. With $y=s/\varepsilon$,
\[
|s|-\psi_\varepsilon(s)=\varepsilon\log2-\varepsilon\log\bigl(1+e^{-2|s|/\varepsilon}\bigr).
\]
Since $0<\log(1+e^{-2|s|/\varepsilon})\le\log 2$, the right-hand side lies
in $[0,\varepsilon\log 2)$.
(iii) Direct differentiation gives $\psi_\varepsilon'(s)=\tanh(s/\varepsilon)$
and $\psi_\varepsilon''(s)=\varepsilon^{-1}\operatorname{sech}^2(s/\varepsilon)\in(0,\varepsilon^{-1}]$.

\emph{Consequences.} $h$ is a sum of convex functions of linear maps
$\Theta\mapsto\Theta_{q\ell}$ and is therefore convex. For a symmetric
direction $\mathbf{E}$,
$Dh(\Theta)[\mathbf{E}]=\sum_{q\neq\ell}\tanh(\Theta_{q\ell}/\varepsilon)E_{q\ell}=\langle\nabla h(\Theta),\mathbf{E}\rangle$
with $\nabla h$ as in~\eqref{eq:grad_penalty}. Since $|\tanh'|\le1$,
$|\tanh(a/\varepsilon)-\tanh(b/\varepsilon)|\le|a-b|/\varepsilon$, which
gives the Lipschitz bound entrywise and hence in Frobenius norm. Summing
(ii) over the $p(p-1)$ off-diagonal entries yields the approximation
bound. Finally, $\tanh(0)=0$, so $[\nabla h(\Theta)]_{q\ell}=0$ whenever
$\Theta_{q\ell}=0$, and, since $|\tanh|<1$,
$\|\nabla h(\Theta)\|_F^2\le\#\{(q,\ell):q\neq\ell,\ \Theta_{q\ell}\neq0\}=\|\Theta\|_{0,\mathrm{off}}$.
\end{proof}

\section{Proof of Proposition~\ref{prop:airm}}
\label{app:airm}
\begin{proof}
Throughout, $\mathbf{M}:=\mathbf{A}^{-1/2}\mathbf{B}\mathbf{A}^{-1/2}\in\Spp$.
Since $\mathbf{A}^{-1}\mathbf{B}=\mathbf{A}^{-1/2}\mathbf{M}\mathbf{A}^{1/2}$,
the matrix $\mathbf{A}^{-1}\mathbf{B}$ is similar to $\mathbf{M}$ and has the
positive eigenvalues $\lambda_j(\mathbf{M})$; as
$\|\log\mathbf{M}\|_F^2=\sum_j\log^2\lambda_j(\mathbf{M})$, this proves the
second equality in~\eqref{eq:geo_dist}.

(a) $\mathbf{B}^{-1}\mathbf{A}=(\mathbf{A}^{-1}\mathbf{B})^{-1}$ has the
eigenvalues $\lambda_j^{-1}$, and $\log^2(\lambda^{-1})=\log^2\lambda$, so
$\dS(\mathbf{B},\mathbf{A})=\dS(\mathbf{A},\mathbf{B})$. Further,
$(\mathbf{A}^{-1})^{-1}\mathbf{B}^{-1}=\mathbf{A}\mathbf{B}^{-1}=\mathbf{A}(\mathbf{B}^{-1}\mathbf{A})\mathbf{A}^{-1}$
is similar to $\mathbf{B}^{-1}\mathbf{A}$, so
$\dS(\mathbf{A}^{-1},\mathbf{B}^{-1})=\dS(\mathbf{B},\mathbf{A})$.

(b) $(\mathbf{G}\mathbf{A}\mathbf{G}^\top)^{-1}\mathbf{G}\mathbf{B}\mathbf{G}^\top
=\mathbf{G}^{-\top}(\mathbf{A}^{-1}\mathbf{B})\mathbf{G}^{\top}$ is similar to
$\mathbf{A}^{-1}\mathbf{B}$. If $\tilde{\mathbf{x}}=\mathbf{C}\mathbf{x}$ with
$\mathbf{C}$ diagonal and invertible, the scatter becomes
$\mathbf{C}\Sigma\mathbf{C}$ and the precision
$\mathbf{C}^{-1}\Theta\mathbf{C}^{-1}$, i.e., $\mathbf{G}=\mathbf{C}^{-1}$.

(c) The matrix exponential is a bijection from $\mathcal{S}^p$ onto $\Spp$
with inverse $\log$ (spectral calculus). Hence
$\mathbf{B}=\mathbf{A}^{1/2}e^{\mathbf{H}}\mathbf{A}^{1/2}$ if and only if
$e^{\mathbf{H}}=\mathbf{M}$, i.e., $\mathbf{H}=\log\mathbf{M}$, and then
$\dS(\mathbf{A},\mathbf{B})=\|\log\mathbf{M}\|_F=\|\mathbf{H}\|_F$.

(d) Let $\mathbf{E}:=\mathbf{A}^{-1/2}\boldsymbol{\Delta}\mathbf{A}^{-1/2}$,
so $\mathbf{M}=\mathbf{I}+\mathbf{E}$ and
$\|\mathbf{E}\|_F\le\|\boldsymbol{\Delta}\|_F/\lambda_{\min}(\mathbf{A})$.
Let $e_j$ be the eigenvalues of $\mathbf{E}$ and assume
$\|\mathbf{E}\|_{\op}\le\frac12$. For $|x|\le\frac12$,
$|\log(1+x)-x|\le\sum_{k\ge2}|x|^k/k\le\frac{x^2}{2(1-|x|)}\le x^2$. Writing
$\log(1+e_j)=e_j+r_j$ with $|r_j|\le e_j^2$,
\[
\Bigl|\dS^2(\mathbf{A},\mathbf{A}+\boldsymbol{\Delta})-\|\mathbf{E}\|_F^2\Bigr|
\le2\sum_j|e_j|^3+\sum_je_j^4\le3\|\mathbf{E}\|_F^3,
\]
where we used $\sum_j|e_j|^3\le(\sum_je_j^2)^{3/2}$ and
$\sum_je_j^4\le\|\mathbf{E}\|_F^4\le\frac12\|\mathbf{E}\|_F^3$ (as
$\|\mathbf{E}\|_F\le\frac12$ may be assumed for small $\boldsymbol{\Delta}$).
Hence the remainder is at most
$3\|\boldsymbol{\Delta}\|_F^3/\lambda_{\min}^3(\mathbf{A})$.

(e) For a unit vector $\mathbf{v}$,
$\mathbf{v}^\top\mathbf{M}\mathbf{v}=\mathbf{w}^\top\mathbf{B}\mathbf{w}$ with
$\mathbf{w}=\mathbf{A}^{-1/2}\mathbf{v}$ and $\|\mathbf{w}\|^2\in[M^{-1},m^{-1}]$,
so the eigenvalues $\mu_j$ of $\mathbf{M}$ lie in $[m/M,M/m]$. Let
$x_j:=\log\mu_j$. By the mean-value theorem, $\mu_j-1=e^{\xi_j}x_j$ for some
$\xi_j$ between $0$ and $x_j$, and $e^{\xi_j}\in[\min(1,\mu_j),\max(1,\mu_j)]\subseteq[m/M,M/m]$.
Thus $(m/M)|x_j|\le|\mu_j-1|\le(M/m)|x_j|$, and summing squares,
\begin{equation}
\frac{m^2}{M^2}\dS^2(\mathbf{A},\mathbf{B})\le\|\mathbf{M}-\mathbf{I}\|_F^2\le\frac{M^2}{m^2}\dS^2(\mathbf{A},\mathbf{B}).
\label{eq:airm_equiv_proof}
\end{equation}
Since $\mathbf{M}-\mathbf{I}=\mathbf{A}^{-1/2}(\mathbf{B}-\mathbf{A})\mathbf{A}^{-1/2}$,
we have $\|\mathbf{M}-\mathbf{I}\|_F\le\|\mathbf{B}-\mathbf{A}\|_F/m$ and
$\|\mathbf{B}-\mathbf{A}\|_F=\|\mathbf{A}^{1/2}(\mathbf{M}-\mathbf{I})\mathbf{A}^{1/2}\|_F\le M\|\mathbf{M}-\mathbf{I}\|_F$.
Combining with~\eqref{eq:airm_equiv_proof} gives
$\dS^2\le\frac{M^2}{m^2}\cdot\frac{\|\mathbf{A}-\mathbf{B}\|_F^2}{m^2}$ and
$\|\mathbf{A}-\mathbf{B}\|_F^2\le M^2\cdot\frac{M^2}{m^2}\dS^2$, which is (e).
\end{proof}

\section{Proof of Theorem~\ref{thm:identifiability}}
\label{app:ident}
\begin{lemma}
\label{lem:sym_kernel}
Let $\mathbf{Y}\in\Rpr$ and $\mathbf{W}\in\R^{p\times r}$. Then
$\mathbf{W}\mathbf{Y}^\top+\mathbf{Y}\mathbf{W}^\top=\mathbf{0}$ if and only if
$\mathbf{W}=\mathbf{Y}\Omega$ for some $\Omega\in\mathfrak{o}_r$.
\end{lemma}
\begin{proof}
If $\mathbf{W}=\mathbf{Y}\Omega$, then
$\mathbf{W}\mathbf{Y}^\top+\mathbf{Y}\mathbf{W}^\top=\mathbf{Y}(\Omega+\Omega^\top)\mathbf{Y}^\top=\mathbf{0}$.
Conversely, multiply $\mathbf{W}\mathbf{Y}^\top+\mathbf{Y}\mathbf{W}^\top=\mathbf{0}$
on the right by $\mathbf{Y}(\mathbf{Y}^\top\mathbf{Y})^{-1}$ to obtain
$\mathbf{W}=\mathbf{Y}\Omega$ with
$\Omega:=-\mathbf{W}^\top\mathbf{Y}(\mathbf{Y}^\top\mathbf{Y})^{-1}$; substituting
back gives $\mathbf{Y}(\Omega+\Omega^\top)\mathbf{Y}^\top=\mathbf{0}$, and since
$\mathbf{Y}$ has full column rank, $\Omega+\Omega^\top=\mathbf{0}$.
\end{proof}

\begin{proof}[Proof of Theorem~\ref{thm:identifiability}]
Recall $D\varphi(\theta)[\xi]=\xi_{\mathbf{Y}}\mathbf{Y}^\top+\mathbf{Y}\xi_{\mathbf{Y}}^\top+\xi_{\mathbf{D}}$,
whose first part lies in $T_{\mathbf{L}}$ and second part in $\Dsp$.

(i) By Lemma~\ref{lem:sym_kernel}, $\mathcal{V}_\theta\subseteq\ker D\varphi(\theta)$.
Suppose $T_{\mathbf{L}}\cap\Dsp=\{\mathbf{0}\}$ and $\xi\in\ker D\varphi(\theta)$.
Then $\xi_{\mathbf{Y}}\mathbf{Y}^\top+\mathbf{Y}\xi_{\mathbf{Y}}^\top=-\xi_{\mathbf{D}}\in T_{\mathbf{L}}\cap\Dsp$,
so $\xi_{\mathbf{D}}=\mathbf{0}$ and, by Lemma~\ref{lem:sym_kernel},
$\xi\in\mathcal{V}_\theta$. Conversely, if
$\mathbf{0}\neq\boldsymbol{\Delta}=\mathbf{Y}\mathbf{W}^\top+\mathbf{W}\mathbf{Y}^\top\in\Dsp$,
then $(\mathbf{W},-\boldsymbol{\Delta})\in\ker D\varphi(\theta)\setminus\mathcal{V}_\theta$.
Finally, for $\boldsymbol{\Delta}\in\Dsp$ with $\|\boldsymbol{\Delta}\|_F=1$,
$1=\|\PT_{T_{\mathbf{L}}}\boldsymbol{\Delta}\|_F^2+\|\PT^\perp_{T_{\mathbf{L}}}\boldsymbol{\Delta}\|_F^2$,
so $\|\PT_{T_{\mathbf{L}}}\boldsymbol{\Delta}\|_F=1$ if and only if
$\boldsymbol{\Delta}\in T_{\mathbf{L}}$; as the maximum in~\eqref{eq:alpha}
is attained on the compact unit sphere of $\Dsp$, $\alpha(\mathbf{Y})=1$ if
and only if $T_{\mathbf{L}}\cap\Dsp\neq\{\mathbf{0}\}$.

(ii) Let $\mathbf{T}:=\xi_{\mathbf{Y}}\mathbf{Y}^\top+\mathbf{Y}\xi_{\mathbf{Y}}^\top\in T_{\mathbf{L}}$
and $\boldsymbol{\Delta}:=\xi_{\mathbf{D}}\in\Dsp$. Since
$\mathbf{T}=\PT_{T_{\mathbf{L}}}\mathbf{T}$ and
$\boldsymbol{\Delta}=\PT_{\Dsp}\boldsymbol{\Delta}$,
$\langle\mathbf{T},\boldsymbol{\Delta}\rangle=\langle\mathbf{T},\PT_{T_{\mathbf{L}}}\PT_{\Dsp}\boldsymbol{\Delta}\rangle\ge-\alpha\|\mathbf{T}\|_F\|\boldsymbol{\Delta}\|_F$,
whence
\begin{align*}
\|\mathbf{T}+\boldsymbol{\Delta}\|_F^2&\ge\|\mathbf{T}\|_F^2+\|\boldsymbol{\Delta}\|_F^2-2\alpha\|\mathbf{T}\|_F\|\boldsymbol{\Delta}\|_F\\
&\ge(1-\alpha)\bigl(\|\mathbf{T}\|_F^2+\|\boldsymbol{\Delta}\|_F^2\bigr).
\end{align*}
For horizontal $\xi$, $\mathbf{S}:=\mathbf{Y}^\top\xi_{\mathbf{Y}}$ is
symmetric, and
\begin{align*}
\|\mathbf{T}\|_F^2&=2\|\xi_{\mathbf{Y}}\mathbf{Y}^\top\|_F^2+2\tr(\xi_{\mathbf{Y}}\mathbf{Y}^\top\xi_{\mathbf{Y}}\mathbf{Y}^\top)\\
&=2\tr(\xi_{\mathbf{Y}}^\top\xi_{\mathbf{Y}}\mathbf{Y}^\top\mathbf{Y})+2\|\mathbf{S}\|_F^2
\ge2\sigma_r^2(\mathbf{Y})\|\xi_{\mathbf{Y}}\|_F^2,
\end{align*}
using $\tr(\mathbf{P}\mathbf{R})\ge\lambda_{\min}(\mathbf{R})\tr(\mathbf{P})$
for $\mathbf{P},\mathbf{R}\succeq\mathbf{0}$.

(iii) The map $\mathbf{Y}\mapsto\mathbf{P}_{\mathbf{Y}}$ is continuous on
$\Rpr$, $\PT_{T_{\mathbf{L}}}$ is a polynomial in $\mathbf{P}_{\mathbf{Y}}$, and
the operator norm is continuous; hence $\alpha$ is continuous, and
$\alpha(\mathbf{Y}\mathbf{O})=\alpha(\mathbf{Y})$ because
$\mathbf{P}_{\mathbf{Y}\mathbf{O}}=\mathbf{P}_{\mathbf{Y}}$. Thus there is an
$\Or$-invariant open set $\mathcal{U}_0\ni\theta^\star$ on which $\alpha<1$.
For $\theta\in\mathcal{U}_0$, $D\pi(\theta)$ restricted to
$\mathcal{H}_\theta$ is an isomorphism onto $T_{[\theta]}(\Bpr/\Or)$
\cite[Sec.~3.5.8]{absil2008optimization}, and
$D\bar\varphi([\theta])\circ D\pi(\theta)=D\varphi(\theta)$; by (ii) and
$\sigma_r(\mathbf{Y})>0$, $D\varphi(\theta)|_{\mathcal{H}_\theta}$ is
injective. Hence $\bar\varphi$ is an immersion on $\pi(\mathcal{U}_0)$.
An immersion is locally injective (indeed, a local embedding)
\cite[Ch.~4]{lee2012introduction}, so there is an open
$\bar{\mathcal{U}}\ni[\theta^\star]$ on which $\bar\varphi$ is injective.
Let $\mathcal{U}:=\pi^{-1}(\bar{\mathcal{U}})$, which is open and
$\Or$-invariant. If $\theta_1,\theta_2\in\mathcal{U}$ and
$\varphi(\theta_1)=\varphi(\theta_2)$, then
$\bar\varphi([\theta_1])=\bar\varphi([\theta_2])$, hence
$[\theta_1]=[\theta_2]$. The converse is immediate.

(iv) For $\mathbf{X}\in\mathcal{S}^p$ and $\mathbf{P}:=\mathbf{P}_{\mathbf{Y}}$,
$\PT_{T_{\mathbf{L}}}\mathbf{X}=\mathbf{P}\mathbf{X}\mathbf{P}+\mathbf{P}\mathbf{X}\mathbf{P}^\perp+\mathbf{P}^\perp\mathbf{X}\mathbf{P}$,
a sum of three mutually Frobenius-orthogonal blocks. Hence
\[
\|\PT_{T_{\mathbf{L}}}\mathbf{X}\|_F^2\le\|\mathbf{P}\mathbf{X}\|_F^2+\|\mathbf{X}\mathbf{P}\|_F^2=2\|\mathbf{P}\mathbf{X}\|_F^2.
\]
For $\boldsymbol{\Delta}=\diag(d_1,\ldots,d_p)$, the $i$-th column of
$\mathbf{P}\boldsymbol{\Delta}$ is $d_i\mathbf{P}\mathbf{e}_i$, so
$\|\mathbf{P}\boldsymbol{\Delta}\|_F^2=\sum_id_i^2\|\mathbf{P}\mathbf{e}_i\|^2\le\mu_c(\mathbf{Y})\|\boldsymbol{\Delta}\|_F^2$.
Therefore $\alpha(\mathbf{Y})^2\le2\mu_c(\mathbf{Y})$.
\end{proof}

\section{Proof of Proposition \ref{prop1}}
\label{app1}
\begin{proof}
We establish the proposition in three stages: we identify the algebraic
condition that the projected vector must satisfy, derive the Sylvester
equation \eqref{eq:sylvester} as the necessary and sufficient condition
for that algebraic constraint, prove that \eqref{eq:sylvester} has a unique
skew-symmetric solution, and finally verify that the formula \eqref{eq:proj}
indeed defines the orthogonal projection onto $\mathcal{H}_\theta$ with
respect to the metric \eqref{eq:metric}.

Recall from \eqref{eq:vertical} and \eqref{eq:horizontal} that the tangent
space $T_\theta\mathcal{B}^{p,r}=\mathbb{R}^{p\times r}\times\mathcal{D}^p$
decomposes as a direct sum $T_\theta\mathcal{B}^{p,r}=\mathcal{V}_\theta\oplus\mathcal{H}_\theta$,
where the vertical space is
$\mathcal{V}_\theta=\{(\mathbf{Y}\Omega,\mathbf{0}):\Omega\in\mathfrak{o}_r\}$
and the horizontal space is
$\mathcal{H}_\theta=\{(\zeta_{\mathbf{Y}},\zeta_{\mathbf{D}})\in T_\theta\mathcal{B}^{p,r}:
\mathbf{Y}^\top\zeta_{\mathbf{Y}}=\zeta_{\mathbf{Y}}^\top \mathbf{Y},\;\zeta_{\mathbf{D}}\in\mathcal{D}^p\}$.
The orthogonal projection $P^{\mathcal{H}}_\theta(\mathbf{Z})$ of an ambient vector
$\mathbf{Z}=(\mathbf{Z}_1,\mathbf{Z}_2)\in\mathbb{R}^{p\times r}\times\mathcal{D}^p$ onto
$\mathcal{H}_\theta$ is the unique element $(\zeta_{\mathbf{Y}},\zeta_{\mathbf{D}})\in\mathcal{H}_\theta$
minimizing $\|\mathbf{Z}-(\zeta_{\mathbf{Y}},\zeta_{\mathbf{D}})\|^2_\theta$ with respect to the metric
\eqref{eq:metric}, or equivalently satisfying
$\mathbf{Z}-(\zeta_{\mathbf{Y}},\zeta_{\mathbf{D}})\in\mathcal{V}_\theta$ and $(\zeta_{\mathbf{Y}},\zeta_{\mathbf{D}})\in\mathcal{H}_\theta$.
Since $\mathcal{V}_\theta$ consists only of pairs of the form $(\mathbf{Y}\Omega,\mathbf{0})$
with $\Omega$ skew-symmetric, the condition $\mathbf{Z}-(\zeta_{\mathbf{Y}},\zeta_{\mathbf{D}})\in\mathcal{V}_\theta$
requires $\mathbf{Z}_2-\zeta_{\mathbf{D}}=\mathbf{0}$ and $\mathbf{Z}_1-\zeta_{\mathbf{Y}}=\mathbf{Y}\Omega$ for some
$\Omega\in\mathfrak{o}_r$.
The first of these conditions gives $\zeta_{\mathbf{D}}=\mathbf{Z}_2$; however, since
$\mathcal{H}_\theta$ requires $\zeta_{\mathbf{D}}$ to be a diagonal matrix and $\mathbf{Z}_2$ is
already assumed to lie in $\mathcal{D}^p$, we simply take $\zeta_{\mathbf{D}}=\ddiag(\mathbf{Z}_2)$.
For an ambient $\mathbf{Z}_2$ not necessarily diagonal, the diagonal component of the
projection retains only $\ddiag(\mathbf{Z}_2)$, since the off-diagonal part of $\mathbf{Z}_2$
belongs to the orthogonal complement of $\mathcal{D}^p$ in the space of
symmetric matrices under the affine-invariant metric, and the projection
onto $\mathcal{D}^p$ is precisely $\ddiag$.
The second condition requires $\zeta_{\mathbf{Y}}=\mathbf{Z}_1-\mathbf{Y}\Omega$ for some
$\Omega\in\mathfrak{o}_r$ to be chosen so that $(\zeta_{\mathbf{Y}},\zeta_{\mathbf{D}})$ lies in
$\mathcal{H}_\theta$, i.e., so that
\begin{equation}
  \mathbf{Y}^\top(\mathbf{Z}_1-\mathbf{Y}\Omega) = (\mathbf{Z}_1-\mathbf{Y}\Omega)^\top \mathbf{Y}.
  \label{eq:horiz_cond}
\end{equation}
Expanding the left-hand side and right-hand side of \eqref{eq:horiz_cond}
separately, and using the skew-symmetry of $\Omega$, which gives
$\Omega^\top=-\Omega$, we obtain
\begin{align*}
    \mathbf{Y}^\top\mathbf{Z}_1 - \mathbf{Y}^\top \mathbf{Y}\Omega
  &= \mathbf{Z}_1^\top \mathbf{Y} - (\mathbf{Y}\Omega)^\top \mathbf{Y}
  = \mathbf{Z}_1^\top \mathbf{Y} - \Omega^\top(\mathbf{Y}^\top \mathbf{Y})\nonumber\\
  &= \nrev{\mathbf{Z}_1^\top} \mathbf{Y} + \Omega(\mathbf{Y}^\top \mathbf{Y}).
\end{align*}
Rearranging, \eqref{eq:horiz_cond} is equivalent to
\begin{equation}
  \Omega(\mathbf{Y}^\top \mathbf{Y})+(\mathbf{Y}^\top \mathbf{Y})\Omega = \mathbf{Y}^\top \mathbf{Z}_1 - \mathbf{Z}_1^\top \mathbf{Y},
  \label{eq:sylvester_derived}
\end{equation}
which is precisely the Sylvester equation \eqref{eq:sylvester}.
We observe that the right-hand side of \eqref{eq:sylvester_derived} is
indeed skew-symmetric, since $(\mathbf{Y}^\top \mathbf{Z}_1-\mathbf{Z}_1^\top \mathbf{Y})^\top=\mathbf{Z}_1^\top \mathbf{Y}-\mathbf{Y}^\top \mathbf{Z}_1
=-(\mathbf{Y}^\top \mathbf{Z}_1-\mathbf{Z}_1^\top \mathbf{Y})$, so the equation is self-consistent with the
requirement $\Omega\in\mathfrak{o}_r$.

We now prove that \eqref{eq:sylvester} has a unique skew-symmetric solution.
Set $\mathbf{Q}:=\mathbf{Y}^\top \mathbf{Y}\in\mathbb{R}^{r\times r}$.
Since $\mathbf{Y}\in\mathcal{R}_{p,r}$ has full column rank by definition, $\mathbf{Q}$ is
symmetric positive definite, and in particular all its eigenvalues are
strictly positive.
Let $\mathbf{Q}=\mathbf{V}\Lambda \mathbf{V}^\top$ be the eigendecomposition of $\mathbf{Q}$, where
$\Lambda=\diag(\lambda_1,\ldots,\lambda_r)$ with $\lambda_i>0$ for all $i$,
and $\mathbf{V}$ is orthogonal.
Setting $\widetilde\Omega:=\mathbf{V}^\top\Omega \mathbf{V}$ and
$\widetilde R:=\mathbf{V}^\top(\mathbf{Y}^\top \mathbf{Z}_1-\mathbf{Z}_1^\top \mathbf{Y})\mathbf{V}$, equation
\eqref{eq:sylvester} transforms under conjugation by $\mathbf{V}$ to
\begin{equation}
  \widetilde\Omega\Lambda+\Lambda\widetilde\Omega=\widetilde R.
  \label{eq:sylvester_diag}
\end{equation}
Since $\mathbf{V}$ is orthogonal, $\widetilde\Omega$ is skew-symmetric if and only if
$\Omega$ is, and $\widetilde R$ is skew-symmetric since $\mathbf{Y}^\top \mathbf{Z}_1-\mathbf{Z}_1^\top \mathbf{Y}$
is.
Examining the $(i,j)$ entry of \eqref{eq:sylvester_diag} for $i\neq j$,
\[
  (\widetilde\Omega\Lambda+\Lambda\widetilde\Omega)_{ij}
  =\widetilde\Omega_{ij}\lambda_j+\lambda_i\widetilde\Omega_{ij}
  =(\lambda_i+\lambda_j)\widetilde\Omega_{ij}
  =\widetilde R_{ij},
\]
so that
\begin{equation}
  \widetilde\Omega_{ij}=\frac{\widetilde R_{ij}}{\lambda_i+\lambda_j},
  \quad i\neq j.
  \label{eq:omega_entries}
\end{equation}
Since $\lambda_i>0$ and $\lambda_j>0$, the denominator $\lambda_i+\lambda_j>0$
for all $i\neq j$, so \eqref{eq:omega_entries} uniquely determines the
off-diagonal entries of $\widetilde\Omega$.
For the diagonal entries, skew-symmetry requires $\widetilde\Omega_{ii}=0$,
and the $(i,i)$ entry of \eqref{eq:sylvester_diag} gives
$2\lambda_i\widetilde\Omega_{ii}=\widetilde R_{ii}$; since $\widetilde R$ is
skew-symmetric, $\widetilde R_{ii}=0$, so this is consistent with
$\widetilde\Omega_{ii}=0$ and imposes no further constraint.
Consequently, $\widetilde\Omega$ is uniquely determined as a skew-symmetric
matrix, and $\Omega=\mathbf{V}\widetilde\Omega \mathbf{V}^\top$ is the unique skew-symmetric
solution to \eqref{eq:sylvester}.

Having established existence and uniqueness of $\Omega$, we define
$\zeta_{\mathbf{Y}}:=\mathbf{Z}_1-\mathbf{Y}\Omega$ and $\zeta_{\mathbf{D}}:=\ddiag(\mathbf{Z}_2)$, and we verify that
$(\zeta_{\mathbf{Y}},\zeta_{\mathbf{D}})$ is indeed the orthogonal projection.
By construction, \eqref{eq:horiz_cond} holds, so $\zeta_{\mathbf{Y}}$ satisfies
$\mathbf{Y}^\top\zeta_{\mathbf{Y}}=\zeta_{\mathbf{Y}}^\top \mathbf{Y}$, and \nrev{$\zeta_{\mathbf{D}}\in\mathcal{D}^p$}, so
$(\zeta_{\mathbf{Y}},\zeta_{\mathbf{D}})\in\mathcal{H}_\theta$.
The residual $\mathbf{Z}-(\zeta_{\mathbf{Y}},\zeta_{\mathbf{D}})=(\mathbf{Y}\Omega,\mathbf{Z}_2-\ddiag(\mathbf{Z}_2))$.
The first component $\mathbf{Y}\Omega$ belongs to $\mathcal{V}_\theta$ by definition.
The second component $\mathbf{Z}_2-\ddiag(\mathbf{Z}_2)$ is a symmetric matrix with zero
diagonal, which is orthogonal to every element of $\mathcal{D}^p$ under the
affine-invariant metric on $\mathcal{D}^p_{++}$: for any $\xi_{\mathbf{D}}\in\mathcal{D}^p$,
$\operatorname{tr}(\mathbf{D}^{-1}(\mathbf{Z}_2-\ddiag(\mathbf{Z}_2))\mathbf{D}^{-1}\xi_{\mathbf{D}})=\nrev{0}$
since $\mathbf{D}^{-1}(\mathbf{Z}_2-\ddiag(\mathbf{Z}_2))\mathbf{D}^{-1}$ has zero diagonal and $\xi_{\mathbf{D}}$ is
diagonal, making their product traceless.
Hence the residual is orthogonal to $\mathcal{H}_\theta$ with respect to the
metric \eqref{eq:metric}, confirming that $P^{\mathcal{H}}_\theta(\mathbf{Z})=(\zeta_{\mathbf{Y}},\zeta_{\mathbf{D}})=(\mathbf{Z}_1-\mathbf{Y}\Omega,\ddiag(\mathbf{Z}_2))$,
as claimed in \eqref{eq:proj}.
\end{proof}

\section{Proof of Proposition \ref{prop:riem_grad}}
\label{app2}
\begin{proof}
We establish the three claims in sequence: the formula for the Euclidean
gradient of $\tilde{f}=f\circ\varphi$ on $\mathcal{B}^{p,r}$, the
identification of the Riemannian gradient through the metric, and the
automatic membership of the resulting vector in the horizontal space
$\mathcal{H}_\theta$.

Since $\varphi:\mathcal{B}^{p,r}\to\mathcal{S}_{++}^p$ is the smooth
surjection $\varphi(\mathbf{Y},\mathbf{D})=\mathbf{Y}\mathbf{Y}^\top+\mathbf{D}$ and $f$ is smooth on $\mathcal{S}_{++}^p$,
the composite $\tilde{f}=f\circ\varphi$ is smooth on $\mathcal{B}^{p,r}$.
To compute its directional derivative at $\theta=(\mathbf{Y},\mathbf{D})\in\mathcal{B}^{p,r}$
in the direction $\xi=(\xi_{\mathbf{Y}},\xi_{\mathbf{D}})\in T_\theta\mathcal{B}^{p,r}$, we first
linearize $\varphi$.
A direct calculation gives
\begin{align*}
      D\varphi(\theta)[\xi]
  &=\lim_{\varepsilon\to 0}\frac{\varphi(\mathbf{Y}+\varepsilon\xi_{\mathbf{Y}},\mathbf{D}+\varepsilon\xi_{\mathbf{D}})
                                 -\varphi(\mathbf{Y},\mathbf{D})}{\varepsilon}\nonumber\\
  &=\xi_{\mathbf{Y}} \mathbf{Y}^\top + \mathbf{Y}\xi_{\mathbf{Y}}^\top + \xi_{\mathbf{D}}
  =2\operatorname{sym}(\xi_{\mathbf{Y}} \mathbf{Y}^\top)+\xi_{\mathbf{D}},
\end{align*}
where we used the bi-linearity of the map $(\mathbf{Y},\mathbf{Y})\mapsto \mathbf{Y}\mathbf{Y}^\top$ and the
fact that $\xi_{\mathbf{D}}$ is diagonal, hence equal to its own symmetric part.
Applying the chain rule and the standard identification of the Euclidean
gradient of $f$ through the Frobenius inner product,
\begin{align}
  D\tilde{f}(\theta)[\xi]
  &= \langle\nabla_\Theta f,\, D\varphi(\theta)[\xi]\rangle_F \notag\\
  &= \operatorname{tr}\!\left(\nabla_\Theta f\cdot
     \bigl(2\operatorname{sym}(\xi_{\mathbf{Y}} \mathbf{Y}^\top)+\xi_{\mathbf{D}}\bigr)\right) \notag\\
  &= 2\operatorname{tr}\!\left(\nabla_\Theta f\cdot\operatorname{sym}(\xi_{\mathbf{Y}} \mathbf{Y}^\top)\right)
    +\operatorname{tr}\!\left(\nabla_\Theta f\cdot\xi_{\mathbf{D}}\right).
  \label{eq:dir_deriv_expand1}
\end{align}
For the first summand, since $\nabla_\Theta f$ is symmetric (as the Euclidean
gradient of a smooth function on $\mathcal{S}_{++}^p$, it inherits the
symmetry of the domain), we have
$\operatorname{tr}(\nabla_\Theta f\cdot\operatorname{sym}(\xi_{\mathbf{Y}} \mathbf{Y}^\top))
=\operatorname{tr}(\nabla_\Theta f\cdot\xi_\mathbf{Y} \mathbf{Y}^\top)$,
Since $\operatorname{tr}(\mathbf{A}\,\operatorname{sym}(\mathbf{B}))=\operatorname{tr}(\mathbf{A}\,\mathbf{B})$
whenever $\mathbf{A}$ is symmetric.
Applying the cyclic property of the trace,
$\operatorname{tr}(\nabla_\Theta f\cdot\xi_{\mathbf{Y}} \mathbf{Y}^\top)
=\operatorname{tr}(\mathbf{Y}^\top\nabla_\Theta f\cdot\xi_{\mathbf{Y}})
\nrev{=\tfrac12\operatorname{tr}(G_{\mathbf{Y}}^\top\xi_{\mathbf{Y}})}$, \nrev{so that the first summand of \eqref{eq:dir_deriv_expand1} equals $\operatorname{tr}(G_{\mathbf{Y}}^\top\xi_{\mathbf{Y}})$},
where we set $G_{\mathbf{Y}}:=2\nabla_\Theta f\cdot \mathbf{Y}\in\mathbb{R}^{p\times r}$.
For the second summand, since $\xi_{\mathbf{D}}$ is diagonal and $\nabla_\Theta f$ is
symmetric, only the diagonal part of $\nabla_\Theta f$ contributes:
$\operatorname{tr}(\nabla_\Theta f\cdot\xi_{\mathbf{D}})
=\operatorname{tr}(\operatorname{ddiag}(\nabla_\Theta f)\cdot\xi_{\mathbf{D}})
=\operatorname{tr}(G_{\mathbf{D}}\,\xi_{\mathbf{D}})$,
where $G_{\mathbf{D}}:=\operatorname{ddiag}(\nabla_\Theta f)\in\mathcal{D}^p$.
Substituting back into \eqref{eq:dir_deriv_expand1}, $\text{for all }\xi=(\xi_{\mathbf{Y}},\xi_{\mathbf{D}})\in T_\theta\mathcal{B}^{p,r}$,
\begin{equation}
  D\tilde{f}(\theta)[\xi]
  =\operatorname{tr}(G_{\mathbf{Y}}^\top\xi_{\mathbf{Y}})+\operatorname{tr}(G_{\mathbf{D}}\,\xi_{\mathbf{D}}).
  \label{eq:dir_deriv_final2}
\end{equation}
Since the Euclidean gradient of $\tilde{f}$ at $\theta$ is, by definition,
the unique element $\nabla_\theta\tilde{f}=(G_{\mathbf{Y}},G_{\mathbf{D}})$ of the ambient space
satisfying $D\tilde{f}(\theta)[\xi]=\operatorname{tr}(G_{\mathbf{Y}}^\top\xi_{\mathbf{Y}})+\operatorname{tr}(G_{\mathbf{D}}\xi_{\mathbf{D}})$
for all $(\xi_{\mathbf{Y}},\xi_{\mathbf{D}})$, the identification \eqref{eq:eucl_grad} follows
immediately from \eqref{eq:dir_deriv_final2}.

We now identify the Riemannian gradient $\grad\tilde{f}(\theta)=(\gamma_{\mathbf{Y}},\gamma_{\mathbf{D}})$,
defined as the unique element of $T_\theta\mathcal{B}^{p,r}$ satisfying
\begin{equation}
  \langle\grad\tilde{f}(\theta),\xi\rangle_\theta
  = D\tilde{f}(\theta)[\xi]
  \quad\text{for all }\xi=(\xi_{\mathbf{Y}},\xi_{\mathbf{D}})\in T_\theta\mathcal{B}^{p,r}.
  \label{eq:riem_grad_def1}
\end{equation}
Expanding the left-hand side using the metric \eqref{eq:metric} and the
right-hand side using \eqref{eq:dir_deriv_final2}, condition
\eqref{eq:riem_grad_def1} becomes
\begin{equation}
  \operatorname{tr}(\gamma_{\mathbf{Y}}^\top\xi_{\mathbf{Y}})
  +\operatorname{tr}(\mathbf{D}^{-1}\gamma_{\mathbf{D}} \mathbf{D}^{-1}\xi_{\mathbf{D}})
  =\operatorname{tr}(G_{\mathbf{Y}}^\top\xi_{\mathbf{Y}})
  +\operatorname{tr}(G_{\mathbf{D}}\,\xi_{\mathbf{D}})
  \label{eq:riem_id1}
\end{equation}
for all $(\xi_{\mathbf{Y}},\xi_{\mathbf{D}})\in T_\theta\mathcal{B}^{p,r}$.
The tangent space $T_\theta\mathcal{B}^{p,r}=\mathbb{R}^{p\times r}\times\mathcal{D}^p$
is a direct product, so $\xi_{\mathbf{Y}}$ and $\xi_{\mathbf{D}}$ range independently over
$\mathbb{R}^{p\times r}$ and $\mathcal{D}^p$ respectively.
Choosing first $\xi_{\mathbf{D}}=\mathbf{0}$ and letting $\xi_{\mathbf{Y}}$ vary freely over all of
$\mathbb{R}^{p\times r}$, equation \eqref{eq:riem_id1} reduces to
$\operatorname{tr}(\gamma_{\mathbf{Y}}^\top\xi_{\mathbf{Y}})=\operatorname{tr}(G_{\mathbf{Y}}^\top\xi_{\mathbf{Y}})$
for all $\xi_{\mathbf{Y}}\in\mathbb{R}^{p\times r}$,
from which the non-degeneracy of the Frobenius inner product gives
$\gamma_{\mathbf{Y}}=G_{\mathbf{Y}}=2\nabla_\Theta f\cdot \mathbf{Y}$.
Choosing next $\xi_{\mathbf{Y}}=\mathbf{0}$ and letting $\xi_{\mathbf{D}}$ vary over all diagonal matrices,
equation \eqref{eq:riem_id1} reduces to
$\operatorname{tr}(\mathbf{D}^{-1}\gamma_{\mathbf{D}} \mathbf{D}^{-1}\xi_{\mathbf{D}})=\operatorname{tr}(G_{\mathbf{D}}\,\xi_{\mathbf{D}})$
for all $\xi_{\mathbf{D}}\in\mathcal{D}^p$,
which, since both sides are linear functionals on $\mathcal{D}^p$ and a
diagonal matrix is characterised by its diagonal entries, is equivalent to
$(\mathbf{D}^{-1}\gamma_{\mathbf{D}} \mathbf{D}^{-1})_{ii}=(G_{\mathbf{D}})_{ii}$ for every $i=1,\ldots,p$,
i.e., $\gamma_{\mathbf{D}}^{ii}/\mathbf{D}_{ii}^2=(G_{\mathbf{D}})_{ii}$, giving
$\gamma_{\mathbf{D}}=\mathbf{D}\,G_{\mathbf{D}}\,\mathbf{D}=\mathbf{D}\,\operatorname{ddiag}(\nabla_\Theta f)\,\mathbf{D}$.
This establishes \eqref{eq:riem_grad}.

It remains to show that $\grad\tilde{f}(\theta)=(\gamma_{\mathbf{Y}},\gamma_{\mathbf{D}})$
belongs to the horizontal space $\mathcal{H}_\theta$.
Recall from \eqref{eq:horizontal} that
$\mathcal{H}_\theta=\{(\zeta_{\mathbf{Y}},\zeta_{\mathbf{D}})\in T_\theta\mathcal{B}^{p,r}:
\mathbf{Y}^\top\zeta_{\mathbf{Y}}=\zeta_{\mathbf{Y}}^\top \mathbf{Y}\}$,
so we must verify that $\mathbf{Y}^\top\gamma_{\mathbf{Y}}=\gamma_{\mathbf{Y}}^\top \mathbf{Y}$.
Substituting $\gamma_{\mathbf{Y}}=2\nabla_\Theta f\cdot \mathbf{Y}$,
\[
  \mathbf{Y}^\top\gamma_{\mathbf{Y}}
  = 2\mathbf{Y}^\top\nabla_\Theta f\cdot \mathbf{Y}.
\]
Since $\nabla_\Theta f$ is a symmetric matrix, the product $\mathbf{Y}^\top\nabla_\Theta f\cdot \mathbf{Y}$
is itself symmetric: for any vectors $u,v\in\mathbb{R}^r$,
$u^\top(\mathbf{Y}^\top\nabla_\Theta f \mathbf{Y})v
=(\mathbf{Y}u)^\top\nabla_\Theta f(\mathbf{Y}v)
=(\mathbf{Y}v)^\top\nabla_\Theta f(\mathbf{Y}u)
=v^\top(\mathbf{Y}^\top\nabla_\Theta f \mathbf{Y})u$,
so $(\mathbf{Y}^\top\nabla_\Theta f \mathbf{Y})^\top=\mathbf{Y}^\top\nabla_\Theta f \mathbf{Y}$.
Therefore $(\mathbf{Y}^\top\gamma_{\mathbf{Y}})^\top=2\mathbf{Y}^\top\nabla_\Theta f\cdot \mathbf{Y}=\mathbf{Y}^\top\gamma_{\mathbf{Y}}$,
which is precisely the condition $\mathbf{Y}^\top\gamma_{\mathbf{Y}}=\gamma_{\mathbf{Y}}^\top \mathbf{Y}$,
confirming $\grad\tilde{f}(\theta)\in\mathcal{H}_\theta$.

This last property also admits an equivalent and illuminating reformulation
in terms of the vertical space $\mathcal{V}_\theta$.
Recall that $\mathcal{V}_\theta=\{(\mathbf{Y}\Omega,\mathbf{0}):\Omega\in\mathfrak{o}_r\}$,
where $\mathfrak{o}_r$ denotes the space of $r\times r$ skew-symmetric matrices.
For any vertical tangent vector $(\mathbf{Y}\Omega,\mathbf{0})\in\mathcal{V}_\theta$,
the directional derivative of $\tilde{f}$ satisfies
\[
  D\tilde{f}(\theta)[(\mathbf{Y}\Omega,\mathbf{0})]
  =\operatorname{tr}(G_{\mathbf{Y}}^\top \mathbf{Y}\Omega)
  =\operatorname{tr}(2\mathbf{Y}^\top\nabla_\Theta f\cdot \mathbf{Y}\cdot\Omega).
\]
Since $\mathbf{B}:=\mathbf{Y}^\top\nabla_\Theta f\cdot \mathbf{Y}$ is symmetric and $\Omega$ is
skew-symmetric, the product $\mathbf{B}\Omega$ satisfies
$(\mathbf{B}\Omega)^\top=\Omega^\top \mathbf{B}^\top=-\Omega \mathbf{B}=-\mathbf{B}\Omega$
(using $\mathbf{B}^\top=\mathbf{B}$ and $\Omega^\top=-\Omega$), so $\mathbf{B}\Omega$ is skew-symmetric,
and therefore $\operatorname{tr}(\mathbf{B}\Omega)=0$.
Hence $D\tilde{f}(\theta)[(\mathbf{Y}\Omega,\mathbf{0})]=0$ for every $\Omega\in\mathfrak{o}_r$,
which by the defining property of the Riemannian gradient means
$\langle\grad\tilde{f}(\theta),(\mathbf{Y}\Omega,\mathbf{0})\rangle_\theta=0$
for all $(\mathbf{Y}\Omega,\mathbf{0})\in\mathcal{V}_\theta$,
i.e., $\grad\tilde{f}(\theta)\perp\mathcal{V}_\theta$.
Since $\mathcal{H}_\theta$ is defined as the orthogonal complement of
$\mathcal{V}_\theta$ within $T_\theta\mathcal{B}^{p,r}$ under the metric
\eqref{eq:metric}, this orthogonality is precisely the statement
$\grad\tilde{f}(\theta)\in\mathcal{H}_\theta$,
completing the proof.
\end{proof}

\begin{revised}
\section{Proof of Proposition~\ref{prop:geo_grad}}
\label{app3}
We use the following standard fact.

\begin{lemma}[Derivative of a trace function]
\label{lem:trace_deriv}
Let $\phi\in C^1((0,\infty))$, $\mathbf{X}\in\Spp$, and
$\mathbf{E}\in\mathcal{S}^p$. Then
$\frac{d}{ds}\tr\phi(\mathbf{X}+s\mathbf{E})\big|_{s=0}=\tr\bigl(\phi'(\mathbf{X})\mathbf{E}\bigr)$.
\end{lemma}
\begin{proof}
Let $\mathbf{X}=\mathbf{U}\diag(\lambda_i)\mathbf{U}^\top$. By the
Daleckii--Krein theorem~\cite[Thm.~V.3.3]{bhatia1997matrix},~\cite[Thm.~3.11]{higham2008functions},
$D\phi(\mathbf{X})[\mathbf{E}]=\mathbf{U}\bigl(\Phi^{[1]}\circ(\mathbf{U}^\top\mathbf{E}\mathbf{U})\bigr)\mathbf{U}^\top$,
where $\Phi^{[1]}$ is the matrix of first divided differences of $\phi$ at
$(\lambda_i)$, whose diagonal entries are $\phi'(\lambda_i)$. Taking the
trace, $\tr D\phi(\mathbf{X})[\mathbf{E}]=\sum_i\phi'(\lambda_i)[\mathbf{U}^\top\mathbf{E}\mathbf{U}]_{ii}=\tr(\phi'(\mathbf{X})\mathbf{E})$.
\end{proof}

\begin{proof}[Proof of Proposition~\ref{prop:geo_grad}]
\emph{Step 1 (second argument).} Fix $\mathbf{A}$ and let
$\phi(x)=\log^2x$, so $\phi'(x)=2\log(x)/x$ and
$\dS^2(\mathbf{A},\mathbf{B})=\tr\phi(\mathbf{M}(\mathbf{B}))$ with the linear
map $\mathbf{M}(\mathbf{B})=\mathbf{A}^{-1/2}\mathbf{B}\mathbf{A}^{-1/2}$. By
Lemma~\ref{lem:trace_deriv}, for $\mathbf{E}\in\mathcal{S}^p$,
\begin{align*}
D_{\mathbf{B}}\dS^2[\mathbf{E}]
&=\tr\bigl(2\log(\mathbf{M})\mathbf{M}^{-1}\mathbf{A}^{-1/2}\mathbf{E}\mathbf{A}^{-1/2}\bigr)\\
&=\bigl\langle 2\mathbf{A}^{-1/2}\log(\mathbf{M})\mathbf{M}^{-1}\mathbf{A}^{-1/2},\mathbf{E}\bigr\rangle,
\end{align*}
and the matrix on the left of the inner product is symmetric because
$\log(\mathbf{M})$ and $\mathbf{M}^{-1}$ commute. Hence
\begin{equation}
\nabla_{\mathbf{B}}\dS^2=2\mathbf{A}^{-1/2}\log(\mathbf{M})\mathbf{M}^{-1}\mathbf{A}^{-1/2}.
\label{eq:gradB_raw}
\end{equation}

\emph{Step 2 (symmetric form).} For an invertible $\mathbf{S}$ and a
diagonalizable $\mathbf{X}$ with positive spectrum,
$\mathbf{S}\log(\mathbf{X})\mathbf{S}^{-1}=\log(\mathbf{S}\mathbf{X}\mathbf{S}^{-1})$~\cite[Thm.~1.13]{higham2008functions}.
With $\mathbf{S}=\mathbf{A}^{-1/2}$ this gives
$\mathbf{A}^{-1/2}\log(\mathbf{M})\mathbf{A}^{1/2}=\log(\mathbf{A}^{-1}\mathbf{B})$,
and $\mathbf{A}^{-1/2}\mathbf{M}^{-1}\mathbf{A}^{-1/2}=\mathbf{B}^{-1}$; hence
\eqref{eq:gradB_raw} equals $2\log(\mathbf{A}^{-1}\mathbf{B})\mathbf{B}^{-1}$.
In the same way, with $\mathbf{N}:=\mathbf{B}^{-1/2}\mathbf{A}\mathbf{B}^{-1/2}$,
$\mathbf{B}^{-1/2}\log(\mathbf{N})\mathbf{B}^{-1/2}=\log(\mathbf{B}^{-1}\mathbf{A})\mathbf{B}^{-1}$,
and $\log(\mathbf{B}^{-1}\mathbf{A})=\log((\mathbf{A}^{-1}\mathbf{B})^{-1})=-\log(\mathbf{A}^{-1}\mathbf{B})$.
Therefore
$\nabla_{\mathbf{B}}\dS^2=-2\mathbf{B}^{-1/2}\log(\mathbf{N})\mathbf{B}^{-1/2}$,
which is~\eqref{eq:geo_grad_2}.

\emph{Step 3 (first argument).} By
Proposition~\ref{prop:airm}(a), $\dS^2(\mathbf{A},\mathbf{B})=\dS^2(\mathbf{B},\mathbf{A})$;
applying~\eqref{eq:geo_grad_2} to the right-hand side, in which
$\mathbf{A}$ is the second argument, yields~\eqref{eq:geo_grad_1}.

\emph{Step 4 (generalized eigendecomposition).} Let $\mathbf{V}$ satisfy
$\mathbf{V}^\top\mathbf{A}\mathbf{V}=\mathbf{I}$ and
$\mathbf{B}\mathbf{V}=\mathbf{A}\mathbf{V}\Lambda$. Then
$\mathbf{U}:=\mathbf{A}^{1/2}\mathbf{V}$ is orthogonal and
$\mathbf{M}\mathbf{U}=\mathbf{A}^{-1/2}\mathbf{B}\mathbf{V}=\mathbf{A}^{1/2}\mathbf{V}\Lambda=\mathbf{U}\Lambda$,
i.e., $\mathbf{M}=\mathbf{U}\Lambda\mathbf{U}^\top$ (conversely, any
eigendecomposition of $\mathbf{M}$ yields such a $\mathbf{V}=\mathbf{A}^{-1/2}\mathbf{U}$).
Hence $\dS^2=\sum_j\log^2\lambda_j$,
$\nabla_{\mathbf{A}}\dS^2=-2\mathbf{A}^{-1/2}\mathbf{U}\log(\Lambda)\mathbf{U}^\top\mathbf{A}^{-1/2}=-2\mathbf{V}\log(\Lambda)\mathbf{V}^\top$,
and, from~\eqref{eq:gradB_raw},
$\nabla_{\mathbf{B}}\dS^2=2\mathbf{A}^{-1/2}\mathbf{U}\log(\Lambda)\Lambda^{-1}\mathbf{U}^\top\mathbf{A}^{-1/2}=2\mathbf{V}\Lambda^{-1}\log(\Lambda)\mathbf{V}^\top$.

\emph{Step 5 (Riemannian form).} The affine-invariant metric
$g_{\mathbf{A}}(\mathbf{U}_1,\mathbf{U}_2)=\tr(\mathbf{A}^{-1}\mathbf{U}_1\mathbf{A}^{-1}\mathbf{U}_2)=\langle\mathbf{A}^{-1}\mathbf{U}_1\mathbf{A}^{-1},\mathbf{U}_2\rangle$
implies that the Riemannian gradient is $\mathbf{A}(\nabla_{\mathbf{A}}f)\mathbf{A}$.
With~\eqref{eq:geo_grad_1}, $\mathbf{A}\nabla_{\mathbf{A}}(\tfrac12\dS^2)\mathbf{A}=-\mathbf{A}^{1/2}\log(\mathbf{M})\mathbf{A}^{1/2}=-\Log_{\mathbf{A}}(\mathbf{B})$,
in agreement with the classical expression of the gradient of the
squared distance on a Hadamard manifold~\cite{pennec2006riemannian}.
\end{proof}

\section{Proof of Theorem~\ref{thm:convergence}}
\label{app4}
\begin{nrevised}
Throughout, $\mathbb{V}:=(\R^{p\times r}\times\mathcal{D}^p)^T$ is the
ambient space with the Frobenius norm $\|\cdot\|$, whose open subset
$\mathfrak{E}=(\R^{p\times r}\times\Dp)^T$ contains $\mathfrak{B}$, and
$\|\cdot\|_\theta$ is the norm of the product of the metrics~\eqref{eq:metric}.
For $\theta=(\mathbf{Y}_t,\mathbf{D}_t)_t$ and
$\eta=(\eta_{\mathbf{Y},t},\eta_{\mathbf{D},t})_t\in T_\theta\mathfrak{B}$, we
write $\mathbf{E}_t:=\mathbf{D}_t^{-1}\eta_{\mathbf{D},t}$ (diagonal), so that
\begin{equation}
  \|\eta\|_\theta^2=\textstyle\sum_t\bigl(\|\eta_{\mathbf{Y},t}\|_F^2+\|\mathbf{E}_t\|_F^2\bigr).
  \label{eq:metric_norm}
\end{equation}

\begin{proof}
\emph{Step 0 (smoothness on $\mathfrak{E}$ and the rank-deficient set).}
For $\theta\in\mathfrak{E}$, $\Theta_t=\mathbf{Y}_t\mathbf{Y}_t^\top+\mathbf{D}_t\succeq\mathbf{D}_t\succ\mathbf{0}$
regardless of the rank of $\mathbf{Y}_t$. Hence every term of $F$ is
$C^2$ on $\mathfrak{E}$: the quadratic forms
$\mathbf{x}^\top\Theta_t\mathbf{x}$ lie in $[0,\infty)$, where $\rho_t$ is
$C^2$ by Assumption~\ref{ass:smooth}(a); $-\log\det$ and $h$ are
$C^\infty$; and $\dS^2$ is $C^\infty$ on $\Spp\times\Spp$, being the sum
of squared logarithms of the eigenvalues of the analytic matrix function
$\Theta^{-1/2}\widetilde\Theta\Theta^{-1/2}$. In particular, $C_H$ and
$C_G$ are finite because $\mathcal{K}\subset\mathfrak{B}\subset\mathfrak{E}$
is compact. We next show that
\begin{equation}
  F(\bar\theta)\ge F(\theta^{(0)})\qquad\text{for every }\bar\theta\in\mathfrak{E}\setminus\mathfrak{B}.
  \label{eq:boundary_level}
\end{equation}
Suppose $F(\bar\theta)<F(\theta^{(0)})$ for some
$\bar\theta\in\mathfrak{E}\setminus\mathfrak{B}$. Full-column-rank
matrices are dense in $\R^{p\times r}$, so there are
$\theta^j\in\mathfrak{B}$ with $\theta^j\to\bar\theta$; by continuity of
$F$ on $\mathfrak{E}$, $\theta^j\in\mathcal{S}_0$ for all large $j$. By
compactness of $\mathcal{S}_0$, a subsequence converges to a point of
$\mathcal{S}_0\subset\mathfrak{B}$, which must equal
$\bar\theta\notin\mathfrak{B}$, a contradiction.

\emph{Step 1 (a uniform local model).} Let $\theta\in\mathcal{S}_0$ with
$g:=\grad F(\theta)\neq0$, let $\eta\in T_\theta\mathfrak{B}$ satisfy
$\langle g,\eta\rangle_\theta\le-c_0\|g\|_\theta^2$ and
$\|\eta\|_\theta\le C_\eta\|g\|_\theta$, and let
$\phi(\alpha):=F(\Retr_\theta(\alpha\eta))$. We claim that, for every
$\alpha\in(0,\bar\alpha]$:
(a) the rank test~\eqref{eq:rank_test} holds at $\alpha$;
(b) $\Retr_\theta(s\eta)\in\mathcal{K}$ for all $s\in[0,\alpha]$;
(c) $\phi(\alpha)\le\phi(0)+\alpha\phi'(0)+\frac{L_R}{2}\alpha^2\|\eta\|_\theta^2$;
(d) if, in addition, $\alpha\le2(1-c_1)c_0/(L_RC_\eta^2)$, the Armijo
condition~\eqref{eq:armijo} holds at $\alpha$.

Since $\|g\|_\theta\le G$ on $\mathcal{S}_0$, the definition of
$\bar\alpha$ in~\eqref{eq:conv_consts2} gives
\begin{equation}
  \alpha\|\eta\|_\theta\le\bar\alpha C_\eta G=\min\{(1-\gamma_\sigma)s_0,\,1\}.
  \label{eq:alpha_eta}
\end{equation}
(a) By Weyl's inequality, $\|\eta_{\mathbf{Y},t}\|_{\op}\le\|\eta\|_\theta$,
and $s_0\le\sigma_r(\mathbf{Y}_t)$,
$\sigma_r(\mathbf{Y}_t+\alpha\eta_{\mathbf{Y},t})\ge\sigma_r(\mathbf{Y}_t)-(1-\gamma_\sigma)s_0\ge\gamma_\sigma\sigma_r(\mathbf{Y}_t)$.

(b) Fix $s\in[0,\alpha]$. By (a) applied at $s$,
$\sigma_r(\mathbf{Y}_t+s\eta_{\mathbf{Y},t})\ge\gamma_\sigma s_0$, and
$\|\mathbf{Y}_t+s\eta_{\mathbf{Y},t}\|_F\le Y_{\max}+1$ by~\eqref{eq:alpha_eta}.
The $\mathbf{D}$-block of~\eqref{eq:retraction} equals
$\mathbf{D}_t(\mathbf{I}+s\mathbf{E}_t+\frac{s^2}{2}\mathbf{E}_t^2)$; its
$i$-th entry is $[\mathbf{D}_t]_{ii}(1+x+\frac{x^2}{2})$ with
$x=s[\mathbf{E}_t]_{ii}$, $|x|\le s\|\eta\|_\theta\le1$, and
$\frac12\le1+x+\frac{x^2}{2}\le\frac52$ for $|x|\le1$. Hence the entry lies
in $[k_d/2,5K_d/2]$, and $\Retr_\theta(s\eta)\in\mathcal{K}$.

(c) The curve $c(s):=\Retr_\theta(s\eta)$ is $C^\infty$ in $s$, with
\[
c'(s)=\bigl(\eta_{\mathbf{Y},t},\ \eta_{\mathbf{D},t}(\mathbf{I}+s\mathbf{E}_t)\bigr)_t,\qquad
c''(s)=\bigl(\mathbf{0},\ \eta_{\mathbf{D},t}\mathbf{E}_t\bigr)_t,
\]
where we used that diagonal matrices commute. For $s\le\alpha$,
$|1+s[\mathbf{E}_t]_{ii}|\le2$ and
$\|\eta_{\mathbf{D},t}\|_F=\|\mathbf{D}_t\mathbf{E}_t\|_F\le K_d\|\mathbf{E}_t\|_F$,
so, by~\eqref{eq:metric_norm},
\[
\begin{aligned}
\|c'(s)\|^2&\le\textstyle\sum_t\bigl(\|\eta_{\mathbf{Y},t}\|_F^2+4K_d^2\|\mathbf{E}_t\|_F^2\bigr)\\
&\le\max\{1,2K_d\}^2\|\eta\|_\theta^2,\\
\|c''(s)\|^2&\le\textstyle\sum_tK_d^2\|\mathbf{E}_t\|_{\op}^2\|\mathbf{E}_t\|_F^2\\
&\le K_d^2\|\eta\|_\theta^4 .
\end{aligned}
\]
Since $c(s)\in\mathcal{K}$ by (b),
$|\phi''(s)|=|\nabla^2F(c(s))[c'(s),c'(s)]+\langle\nabla F(c(s)),c''(s)\rangle|\le L_R\|\eta\|_\theta^2$.
Moreover, $\phi'(0)=DF(\theta)[D\Retr_\theta(0)[\eta]]=DF(\theta)[\eta]=\langle g,\eta\rangle_\theta$.
Taylor's formula
$\phi(\alpha)=\phi(0)+\alpha\phi'(0)+\int_0^\alpha(\alpha-s)\phi''(s)\,ds$
gives (c).

(d) By (c) and the angle and length conditions,
\[
\begin{aligned}
&\phi(\alpha)-\phi(0)-c_1\alpha\phi'(0)
\le(1-c_1)\alpha\phi'(0)+\tfrac{L_R}{2}\alpha^2\|\eta\|_\theta^2\\
&\qquad\le\alpha\|g\|_\theta^2\Bigl(-(1-c_1)c_0+\tfrac{L_R}{2}\alpha C_\eta^2\Bigr)\le0 .
\end{aligned}
\]

\emph{Step 2 (induction: well-posedness and sufficient decrease).}
Assume $\theta^{(k)}\in\mathcal{S}_0$ and $g^{(k)}\neq0$. Since
$\sigma_r(\mathbf{Y}_t^{(k)})\ge s_0>0$, the matrix
$\mathbf{Y}_t^{(k)\top}\mathbf{Y}_t^{(k)}$ is invertible, so~\eqref{eq:sylvester}
has a unique solution and the transport~\eqref{eq:transport} is well
defined at $\theta^{(k)}$. The safeguard~\eqref{eq:safeguard} returns a
direction that satisfies the angle and length conditions of Step~1: either
$\tilde\eta^{(k)}$ satisfies them by construction, or the restart
$-g^{(k)}$ does, because $c_0\le1\le C_\eta$. By Step~1(a),(d), every trial
$\alpha\in(0,\bar\alpha_A]$ is accepted. The trials $s_k\beta^j$ decrease
geometrically, and $s_k\beta^j\le\bar\alpha_A$ as soon as
$j\ge\log(s_k/\bar\alpha_A)/\log(1/\beta)$; hence $j_k\le\bar\jmath$. If
$j_k=0$, then $\alpha_k=s_k\ge s_{\min}$. If $j_k\ge1$, the trial
$\alpha_k/\beta$ was rejected, so $\alpha_k/\beta>\bar\alpha_A$. Thus
$\alpha_k\ge\alpha_{\rm low}$. The accepted point passes the rank test, so
$\mathbf{Y}_t^{(k+1)}\in\Rpr$, and its $\mathbf{D}$-blocks are positive
definite (Section~\ref{sec:geom_manifold}(f)); hence
$\theta^{(k+1)}\in\mathfrak{B}$. By~\eqref{eq:armijo} and the safeguard,
\[
\begin{aligned}
F(\theta^{(k+1)})&\le F(\theta^{(k)})+c_1\alpha_k\langle g^{(k)},\eta^{(k)}\rangle_{\theta^{(k)}}\\
&\le F(\theta^{(k)})-c_1c_0\,\alpha_{\rm low}\|g^{(k)}\|^2_{\theta^{(k)}},
\end{aligned}
\]
which is (ii), and $F(\theta^{(k+1)})<F(\theta^{(k)})\le F(\theta^{(0)})$,
so $\theta^{(k+1)}\in\mathcal{S}_0$. Since $\theta^{(0)}\in\mathcal{S}_0$,
induction proves (i) and (ii). (By~\eqref{eq:boundary_level}, the strict
decrease alone already excludes $\theta^{(k+1)}\in\mathfrak{E}\setminus\mathfrak{B}$.)

\emph{Step 3 (summation).} If the algorithm stops at a finite $k$, then
$g^{(k)}=0$ and there is nothing to prove. Otherwise, summing (ii) over
$k=0,\ldots,K-1$,
\[
\begin{aligned}
c_1c_0\,\alpha_{\rm low}\textstyle\sum_{k<K}\|g^{(k)}\|_{\theta^{(k)}}^2&\le F(\theta^{(0)})-F(\theta^{(K)})\\
&\le F(\theta^{(0)})-F_{\rm low},
\end{aligned}
\]
where $F_{\rm low}=\min_{\mathcal{S}_0}F$ is finite by compactness. Letting
$K\to\infty$ gives~\eqref{eq:sum_grad}, hence
$\|g^{(k)}\|_{\theta^{(k)}}\to0$, and $K\min_{k<K}\|g^{(k)}\|^2\le\sum_{k<K}\|g^{(k)}\|^2$
gives~\eqref{eq:conv_rate}. Each iteration uses $j_k+1\le1+\bar\jmath$
evaluations of $F$. Accumulation points exist because
$\{\theta^{(k)}\}\subset\mathcal{S}_0$ is compact; if
$\theta^{(k_j)}\to\bar\theta\in\mathcal{S}_0$, continuity of
$\theta\mapsto\|\grad F(\theta)\|_\theta$ on $\mathfrak{B}$ gives
$\grad F(\bar\theta)=0$.
\end{proof}
\end{nrevised}

\section{Proofs of the Statistical Results}
\label{app5}
Throughout this appendix, $\Theta^\circ=(\mathbf{L}_t^\circ+\mathbf{D}_t^\circ)_t$
is an LRaD sequence satisfying Assumption~\ref{ass:structure}
(in Theorem~\ref{thm:estimation}, $\Theta^\circ=\Theta^\star$), $\Nb$,
$\mathcal{T}_t$, $\PT_t$, $m,M,c_F$ are defined with respect to $\Theta^\circ$,
$\mathbf{U}_t$ is an orthonormal basis of $\operatorname{col}(\mathbf{L}_t^\circ)$,
$\mathbf{P}_t=\mathbf{U}_t\mathbf{U}_t^\top$, and
$C_\alpha:=\frac{4}{3(1-\alpha_\star)}$.

\subsection{Geometric lemmas}

\begin{lemma}[Curvature of the rank-$r$ cone]
\label{lem:curv}
Let $\mathbf{L}^\circ\succeq\mathbf{0}$ have rank $r$ with column-space
projector $\mathbf{P}$, and let $\mathbf{L}\succeq\mathbf{0}$ with
$\operatorname{rank}\mathbf{L}\le r$ and
$\|\mathbf{L}-\mathbf{L}^\circ\|_{\op}\le\lambda_r(\mathbf{L}^\circ)/2$. Then
$\mathbf{N}:=\PT^\perp_{T_{\mathbf{L}^\circ}}(\mathbf{L}-\mathbf{L}^\circ)=\mathbf{P}^\perp\mathbf{L}\mathbf{P}^\perp$
is positive semidefinite, $\operatorname{rank}\mathbf{N}\le r$, and
$\|\mathbf{N}\|_F\le\|\mathbf{L}-\mathbf{L}^\circ\|_F^2/\lambda_r(\mathbf{L}^\circ)$.
\end{lemma}
\begin{proof}
Since $\PT^\perp_{T_{\mathbf{L}^\circ}}\mathbf{X}=\mathbf{P}^\perp\mathbf{X}\mathbf{P}^\perp$
and $\mathbf{P}^\perp\mathbf{L}^\circ=\mathbf{0}$,
$\mathbf{N}=\mathbf{P}^\perp\mathbf{L}\mathbf{P}^\perp\succeq\mathbf{0}$ with rank
at most $\operatorname{rank}\mathbf{L}\le r$. Let $\mathbf{U}_\perp$ complete
$\mathbf{U}$ to an orthonormal basis, and write the blocks of $\mathbf{L}$
as $\mathbf{A}_{11}=\mathbf{U}^\top\mathbf{L}\mathbf{U}$,
$\mathbf{A}_{21}=\mathbf{U}_\perp^\top\mathbf{L}\mathbf{U}=\mathbf{A}_{12}^\top$,
$\mathbf{A}_{22}=\mathbf{U}_\perp^\top\mathbf{L}\mathbf{U}_\perp$. By Weyl's
inequality, $\lambda_{\min}(\mathbf{A}_{11})\ge\lambda_r(\mathbf{L}^\circ)-\|\mathbf{L}-\mathbf{L}^\circ\|_{\op}\ge\lambda_r(\mathbf{L}^\circ)/2>0$,
so $\operatorname{rank}\mathbf{L}\ge r$, i.e., $\operatorname{rank}\mathbf{L}=r$.
By the Schur-complement rank identity
$\operatorname{rank}\mathbf{L}=\operatorname{rank}\mathbf{A}_{11}+\operatorname{rank}(\mathbf{A}_{22}-\mathbf{A}_{21}\mathbf{A}_{11}^{-1}\mathbf{A}_{12})$,
the Schur complement vanishes:
$\mathbf{A}_{22}=\mathbf{A}_{21}\mathbf{A}_{11}^{-1}\mathbf{A}_{12}$. Since
$\mathbf{A}_{21}=\mathbf{U}_\perp^\top(\mathbf{L}-\mathbf{L}^\circ)\mathbf{U}$ is an
off-diagonal block of $\mathbf{L}-\mathbf{L}^\circ$ in this basis,
$2\|\mathbf{A}_{21}\|_F^2\le\|\mathbf{L}-\mathbf{L}^\circ\|_F^2$. Hence
$\|\mathbf{N}\|_F=\|\mathbf{A}_{22}\|_F\le\|\mathbf{A}_{21}\|_F^2\|\mathbf{A}_{11}^{-1}\|_{\op}\le\frac{\|\mathbf{L}-\mathbf{L}^\circ\|_F^2/2}{\lambda_r(\mathbf{L}^\circ)/2}$.
\end{proof}

\begin{lemma}[Transversal splitting]
\label{lem:split}
In the setting of Lemma~\ref{lem:curv}, let
$\boldsymbol{\Delta}_L:=\mathbf{L}-\mathbf{L}^\circ$ with
$\|\boldsymbol{\Delta}_L\|_F\le(1-\alpha)\lambda_r(\mathbf{L}^\circ)/4$, where
$\alpha:=\alpha(\mathbf{Y}^\circ)<1$, and let $\boldsymbol{\Delta}_D\in\Dsp$.
Then $\|\boldsymbol{\Delta}_L\|_F^2+\|\boldsymbol{\Delta}_D\|_F^2\le\frac{4}{3(1-\alpha)}\|\boldsymbol{\Delta}_L+\boldsymbol{\Delta}_D\|_F^2$.
\end{lemma}
\begin{proof}
Write $\boldsymbol{\Delta}_L=\PT_{T_{\mathbf{L}^\circ}}\boldsymbol{\Delta}_L+\mathbf{N}$.
As in the proof of Theorem~\ref{thm:identifiability}(ii),
$\langle\PT_{T_{\mathbf{L}^\circ}}\boldsymbol{\Delta}_L,\boldsymbol{\Delta}_D\rangle\ge-\alpha\|\boldsymbol{\Delta}_L\|_F\|\boldsymbol{\Delta}_D\|_F$,
and by Lemma~\ref{lem:curv},
$|\langle\mathbf{N},\boldsymbol{\Delta}_D\rangle|\le\frac{\|\boldsymbol{\Delta}_L\|_F^2}{\lambda_r(\mathbf{L}^\circ)}\|\boldsymbol{\Delta}_D\|_F\le\frac{1-\alpha}{4}\|\boldsymbol{\Delta}_L\|_F\|\boldsymbol{\Delta}_D\|_F$.
Therefore, with $x=\|\boldsymbol{\Delta}_L\|_F$ and $y=\|\boldsymbol{\Delta}_D\|_F$,
\begin{align*}
\|\boldsymbol{\Delta}_L+\boldsymbol{\Delta}_D\|_F^2&\ge x^2+y^2-2\Bigl(\alpha+\frac{1-\alpha}{4}\Bigr)xy\\
&\ge\Bigl(1-\alpha-\frac{1-\alpha}{4}\Bigr)(x^2+y^2),
\end{align*}
which is the claim.
\end{proof}

\begin{proof}[Proof of Lemma~\ref{lem:rsc_convex}]
Let $\Theta,\Theta^\circ$ have spectra in $[m,M]$ and
$\boldsymbol{\Delta}=\Theta-\Theta^\circ$. The segment
$\Theta^\circ+s\boldsymbol{\Delta}$, $s\in[0,1]$, has spectrum in $[m,M]$.
For $b(\Theta):=-\frac12\log\det\Theta$,
$\frac{d^2}{ds^2}b(\Theta^\circ+s\boldsymbol{\Delta})=\frac12\|\mathbf{X}_s^{-1/2}\boldsymbol{\Delta}\mathbf{X}_s^{-1/2}\|_F^2\ge\frac{\|\boldsymbol{\Delta}\|_F^2}{2M^2}$
with $\mathbf{X}_s=\Theta^\circ+s\boldsymbol{\Delta}$. By Taylor's formula with
integral remainder, the Bregman divergence of $b$ is at least
$\int_0^1(1-s)\frac{\|\boldsymbol{\Delta}\|_F^2}{2M^2}ds=\frac{\|\boldsymbol{\Delta}\|_F^2}{4M^2}$.
The data term $\Theta\mapsto\frac1{n_t}\sum_i\rho_t(\mathbf{x}_{t,i}^\top\Theta\mathbf{x}_{t,i})$
is a sum of convex functions of linear maps, so its Bregman divergence is
nonnegative. Hence Assumption~\ref{ass:rsc} holds with $\ell=1/(2M^2)$.
\end{proof}

\begin{proof}[Proof of Lemma~\ref{lem:dtime}]
The path Laplacian $\LT$ has eigenvalues
$4\sin^2(\pi k/(2T))$, $k=0,\ldots,T-1$, which gives~\eqref{eq:dtime}.
Each summand is nonincreasing in $\tau$, the summand with $k=0$ equals
$1$, and the others tend to $0$; this gives the first three claims. Let
$f(\vartheta):=(1+4\tau\sin^2(\vartheta/2))^{-1}$, which is nonincreasing on
$[0,\pi]$, and $\vartheta_k:=\pi k/T$. For $k\ge1$,
$\frac{\pi}{T}f(\vartheta_k)\le\int_{\vartheta_{k-1}}^{\vartheta_k}f$. Hence
\begin{align*}
\dtime(\tau)&\le\frac1T+\frac1\pi\int_0^\pi\frac{d\vartheta}{(1+2\tau)-2\tau\cos\vartheta}\\
&=\frac1T+\frac{1}{\sqrt{(1+2\tau)^2-4\tau^2}}=\frac1T+\frac{1}{\sqrt{1+4\tau}},
\end{align*}
using $\int_0^\pi\frac{d\vartheta}{a-b\cos\vartheta}=\frac{\pi}{\sqrt{a^2-b^2}}$ for $a>|b|$.
\end{proof}

\subsection{Probabilistic lemmas}

\begin{lemma}[Moments of the weighted scatter]
\label{lem:moments}
Under Assumption~\ref{ass:ellip}, let
$\mathbf{W}:=u_t(Q_t)\mathbf{x}\mathbf{x}^\top$ and $\Sigma:=\Sigma_t^\star$.
Then: (a) $\mathbf{W}=\psi_t(Q_t)\Sigma^{1/2}\mathbf{s}\mathbf{s}^\top\Sigma^{1/2}$
with $\psi_t(Q_t)$ independent of $\mathbf{s}$, $\E\mathbf{W}=\Sigma$, and
$\E\psi_t(Q_t)=p$;
(b) $\operatorname{Var}\langle\mathbf{W},\mathbf{E}\rangle\le(2+\sigma_\psi^2)\|\Sigma\|_{\op}^2\|\mathbf{E}\|_F^2$
for every $\mathbf{E}\in\mathcal{S}^p$; hence, for every subspace
$\mathcal{T}\subseteq\mathcal{S}^p$ of dimension $d$,
$\E\|\PT_{\mathcal{T}}(\mathbf{W}-\Sigma)\|_F^2\le d(2+\sigma_\psi^2)\|\Sigma\|_{\op}^2$;
(c) almost surely $\|\mathbf{W}-\Sigma\|_F\le(\kappa_\psi+1)p\|\Sigma\|_{\op}$ and
$\|\mathbf{W}-\Sigma\|_{\op}\le(\kappa_\psi p+1)\|\Sigma\|_{\op}$, and
$\|\E(\mathbf{W}-\Sigma)^2\|_{\op}\le(p+\sigma_\psi^2)\|\Sigma\|_{\op}^2$.
\end{lemma}
\begin{proof}
(a) With $\mathbf{x}=R\Sigma^{1/2}\mathbf{s}$, $Q_t=R^2\mathbf{s}^\top\Sigma^{1/2}\Theta_t^\star\Sigma^{1/2}\mathbf{s}=R^2$,
so $\mathbf{W}=u_t(R^2)R^2\Sigma^{1/2}\mathbf{s}\mathbf{s}^\top\Sigma^{1/2}=\psi_t(Q_t)\Sigma^{1/2}\mathbf{s}\mathbf{s}^\top\Sigma^{1/2}$.
Assumption~\ref{ass:ellip}(b) gives $\E\mathbf{W}=\Sigma$, and
$\E\psi_t(Q_t)=\E[u_t(Q_t)\mathbf{x}^\top\Theta_t^\star\mathbf{x}]=\tr(\Theta_t^\star\Sigma)=p$.

(b) Let $\mathbf{A}:=\Sigma^{1/2}\mathbf{E}\Sigma^{1/2}$, so
$\langle\mathbf{W},\mathbf{E}\rangle=\psi_t(Q_t)\,\mathbf{s}^\top\mathbf{A}\mathbf{s}$.
For $\mathbf{s}$ uniform on the sphere, $\E\mathbf{s}^\top\mathbf{A}\mathbf{s}=\tr(\mathbf{A})/p$
and $\E(\mathbf{s}^\top\mathbf{A}\mathbf{s})^2=\frac{\tr(\mathbf{A})^2+2\|\mathbf{A}\|_F^2}{p(p+2)}$
(write a standard Gaussian vector as $\mathbf{z}=\|\mathbf{z}\|\mathbf{s}$ with
$\|\mathbf{z}\|$ independent of $\mathbf{s}$ and use
$\E(\mathbf{z}^\top\mathbf{A}\mathbf{z})^2=\tr(\mathbf{A})^2+2\|\mathbf{A}\|_F^2$,
$\E\|\mathbf{z}\|^4=p(p+2)$). With $v:=\operatorname{Var}\psi_t(Q_t)\le\sigma_\psi^2p$
and $\E\psi_t^2=p^2+v$, independence gives
\[
\operatorname{Var}\langle\mathbf{W},\mathbf{E}\rangle
=\frac{(p^2+v)(\tr(\mathbf{A})^2+2\|\mathbf{A}\|_F^2)}{p(p+2)}-\tr(\mathbf{A})^2.
\]
The $p^2$-part equals $\frac{2p\|\mathbf{A}\|_F^2-2\tr(\mathbf{A})^2}{p+2}\le2\|\mathbf{A}\|_F^2$,
and, by $\tr(\mathbf{A})^2\le p\|\mathbf{A}\|_F^2$, the $v$-part is at most
$\frac{v(p+2)\|\mathbf{A}\|_F^2}{p(p+2)}\le\sigma_\psi^2\|\mathbf{A}\|_F^2$.
Finally, $\|\mathbf{A}\|_F\le\|\Sigma\|_{\op}\|\mathbf{E}\|_F$. For the second
claim, expand in an orthonormal basis $\{\mathbf{E}_k\}_{k=1}^d$ of
$\mathcal{T}$: $\E\|\PT_{\mathcal{T}}(\mathbf{W}-\Sigma)\|_F^2=\sum_k\operatorname{Var}\langle\mathbf{W},\mathbf{E}_k\rangle$.

(c) $\|\mathbf{W}\|_F=\|\mathbf{W}\|_{\op}=\psi_t(Q_t)\,\mathbf{s}^\top\Sigma\mathbf{s}\le\kappa_\psi p\|\Sigma\|_{\op}$,
$\|\Sigma\|_F\le\sqrt p\|\Sigma\|_{\op}\le p\|\Sigma\|_{\op}$, and the triangle
inequality give the first two bounds. Next,
$\mathbf{W}^2=\psi_t^2(\mathbf{s}^\top\Sigma\mathbf{s})\Sigma^{1/2}\mathbf{s}\mathbf{s}^\top\Sigma^{1/2}\preceq\psi_t^2\|\Sigma\|_{\op}\Sigma^{1/2}\mathbf{s}\mathbf{s}^\top\Sigma^{1/2}$,
so $\E(\mathbf{W}-\Sigma)^2\preceq\E\mathbf{W}^2\preceq(p^2+v)\|\Sigma\|_{\op}\Sigma/p$,
whose norm is at most $(p+\sigma_\psi^2)\|\Sigma\|_{\op}^2$.
\end{proof}

\begin{lemma}[Concentration]
\label{lem:conc}
Let Assumptions~\ref{ass:ellip}--\ref{ass:structure} hold, let
$\mathbf{Z}_t:=\frac1{n_t}\sum_i u_t(\mathbf{x}_{t,i}^\top\Theta_t^\star\mathbf{x}_{t,i})\mathbf{x}_{t,i}\mathbf{x}_{t,i}^\top-\Sigma_t^\star$,
$\mathbf{a}:=(\PT_t\mathbf{Z}_t)_{t=1}^T\in(\mathcal{S}^p)^T$, and
$\mathbf{Q}:=\frac\ell2\mathbf{I}_T+\mu c_F\LT$. If
$n_{\min}\ge2(\kappa_\psi+1)^2pL_\eta$, then with probability at least
$1-\eta$, simultaneously:
\begin{enumerate}
\item[(a)] $\|(\mathbf{Q}^{-1/2}\otimes\mathcal{I})\mathbf{a}\|^2\le9L_\eta\frac{\dr\sigma_W^2}{n_{\min}}\cdot\frac{2T}{\ell}\dtime\bigl(\frac{2\mu c_F}{\ell}\bigr)$;
\item[(b)] $\max_t\|\PT_t\mathbf{Z}_t\|_F\le3\sigma_W\sqrt{\dr L_\eta/n_{\min}}$;
\item[(c)] $\max_t\|\mathbf{Z}_t\|_{\op}\le\frac1m\bigl(\sqrt{2(1+\sigma_\psi^2)}+1\bigr)\sqrt{pL_\eta/n_{\min}}$;
\end{enumerate}
and consequently
$\omega:=\max_t(\|\mathbf{Z}_t\|_{\op}+\|\PT_t\mathbf{Z}_t\|_F)\le\frac{C_\omega}{m}\sqrt{\dr L_\eta/n_{\min}}$.
Here $\|\cdot\|$ on $(\mathcal{S}^p)^T$ is
$\|(\mathbf{B}_t)_t\|^2=\sum_t\|\mathbf{B}_t\|_F^2$ and
$(\mathbf{Q}^{-1/2}\otimes\mathcal{I})(\mathbf{B}_t)_t:=\bigl(\sum_s[\mathbf{Q}^{-1/2}]_{ts}\mathbf{B}_s\bigr)_t$.
\end{lemma}
\begin{proof}
Recall $\|\Sigma_t^\star\|_{\op}\le1/\lambda_{\min}(\mathbf{D}_t^\star)\le1/m$, so
Lemma~\ref{lem:moments} applies with $\|\Sigma\|_{\op}\le1/m$ and
$(2+\sigma_\psi^2)\|\Sigma\|_{\op}^2\le\sigma_W^2$. We use the vector Bernstein
inequality~\cite[Thm.~12]{gross2011recovering}: if $\boldsymbol{\xi}_j$ are
independent zero-mean random vectors in a Hilbert space with
$\|\boldsymbol{\xi}_j\|\le B$ a.s.\ and $\sum_j\E\|\boldsymbol{\xi}_j\|^2\le V$, then
for $0\le s\le V/B$,
$\mathbb{P}(\|\sum_j\boldsymbol{\xi}_j\|\ge\sqrt V+s)\le\exp(-s^2/(4V))$.

(a) The vectors
$\boldsymbol{\xi}_{t,i}:=(\mathbf{Q}^{-1/2}\mathbf{e}_t)\otimes\PT_t(\mathbf{W}_{t,i}-\Sigma_t^\star)/n_t$
are independent and zero mean, and
$\sum_{t,i}\boldsymbol{\xi}_{t,i}=(\mathbf{Q}^{-1/2}\otimes\mathcal{I})\mathbf{a}$. Since
$\|\mathbf{Q}^{-1/2}\mathbf{e}_t\|^2=[\mathbf{Q}^{-1}]_{tt}\le2/\ell$,
Lemma~\ref{lem:moments}(c) gives
$\|\boldsymbol{\xi}_{t,i}\|\le B:=\sqrt{2/\ell}\,(\kappa_\psi+1)p/(mn_{\min})$, and
Lemma~\ref{lem:moments}(b) with $\dim\mathcal{T}_t=\dr$ gives
$\sum_{t,i}\E\|\boldsymbol{\xi}_{t,i}\|^2\le\sum_t[\mathbf{Q}^{-1}]_{tt}\dr\sigma_W^2/n_t\le\frac{\dr\sigma_W^2}{n_{\min}}\tr(\mathbf{Q}^{-1})=:V$.
Since $\mathbf{Q}^{-1}=\frac2\ell(\mathbf{I}+\tau\LT)^{-1}$ with $\tau=2\mu c_F/\ell$,
$\tr(\mathbf{Q}^{-1})=\frac{2T}{\ell}\dtime(\tau)\ge\frac2\ell$. Take
$s=2\sqrt{V\log(3/\eta)}$; the requirement $s\le V/B$, i.e.,
$4B^2\log(3/\eta)\le V$, follows from
$n_{\min}\ge2(\kappa_\psi+1)^2p\log(3/\eta)\ge\frac{4(\kappa_\psi+1)^2p^2\log(3/\eta)}{\dr(2+\sigma_\psi^2)}$
(because $\dr\ge p$ and $2+\sigma_\psi^2\ge2$). Hence, with probability at
least $1-\eta/3$,
$\|(\mathbf{Q}^{-1/2}\otimes\mathcal{I})\mathbf{a}\|^2\le V(1+2\sqrt{\log(3/\eta)})^2\le9V\log(3e/\eta)\le9VL_\eta$.

(b) For a fixed $t$, apply the same inequality to
$\boldsymbol{\xi}_i=\PT_t(\mathbf{W}_{t,i}-\Sigma_t^\star)/n_t$ with
$B=(\kappa_\psi+1)p/(mn_t)$, $V=\dr\sigma_W^2/n_t$, and
$s=2\sqrt{V\log(3T/\eta)}$; the condition $s\le V/B$ holds as above. With
probability at least $1-\eta/(3T)$,
$\|\PT_t\mathbf{Z}_t\|_F\le3\sqrt{V\log(3eT/\eta)}\le3\sigma_W\sqrt{\dr L_\eta/n_{\min}}$.
A union bound over $t$ gives (b) with probability at least $1-\eta/3$.

(c) Apply the matrix Bernstein inequality~\cite[Thm.~1.4]{tropp2012} to
$\mathbf{X}_i=(\mathbf{W}_{t,i}-\Sigma_t^\star)/n_t$, for which, by
Lemma~\ref{lem:moments}(c), $\|\mathbf{X}_i\|_{\op}\le R:=(\kappa_\psi p+1)/(mn_t)$
and $\|\sum_i\E\mathbf{X}_i^2\|_{\op}\le\sigma^2:=(p+\sigma_\psi^2)/(m^2n_t)$. With
$L_c:=\log(6pT/\eta)\le L_\eta$, the two-sided bound
$\mathbb{P}(\|\sum_i\mathbf{X}_i\|_{\op}\ge u)\le2p\exp(-\frac{u^2/2}{\sigma^2+Ru/3})$
at $u=\sqrt{2\sigma^2L_c}+\frac23RL_c$ is at most $\eta/(3T)$. Since
$n_t\ge2(\kappa_\psi+1)^2pL_\eta$,
$(\kappa_\psi p+1)L_\eta/n_t\le(\kappa_\psi+1)pL_\eta/n_t\le\sqrt{pL_\eta/n_t}/\sqrt2$, and
$\sqrt{2(p+\sigma_\psi^2)}\le\sqrt{2(1+\sigma_\psi^2)p}$; hence
$u\le\frac1m(\sqrt{2(1+\sigma_\psi^2)}+1)\sqrt{pL_\eta/n_{\min}}$. A union bound
over $t$ gives (c) with probability at least $1-\eta/3$.

Finally, combining (b) and (c) with $p\le\dr$ and
$3\sigma_W=3\sqrt{2+\sigma_\psi^2}/m$ gives the bound on $\omega$.
\end{proof}

\subsection{A deterministic master inequality}

\begin{proposition}[Master inequality]
\label{prop:master}
Let $\Theta^\circ$ be LRaD and satisfy Assumption~\ref{ass:structure}, let
$\hat\Theta\in\Nb$ satisfy $F(\hat\Theta)\le F(\Theta^\circ)$, and suppose that
the inequality in Assumption~\ref{ass:rsc} holds at $\Theta^\circ$. Let
$\mathbf{g}_t\in\mathcal{S}^p$ be arbitrary deterministic matrices, define
$\mathbf{Z}_t:=2(\nabla\mathcal{L}_t(\Theta_t^\circ)-\mathbf{g}_t)$,
$\mathbf{a}:=(\PT_t\mathbf{Z}_t)_t$, $\mathbf{Q}:=\frac\ell2\mathbf{I}_T+\mu c_F\LT$, and
assume
\begin{equation}
\omega:=\max_t\bigl(\|\mathbf{Z}_t\|_{\op}+\|\PT_t\mathbf{Z}_t\|_F\bigr)\le\frac{3(1-\alpha_\star)\ell\sigma_\star}{16\sqrt r}.
\label{eq:omega_cond}
\end{equation}
Then, with $\boldsymbol{\Delta}_t:=\hat\Theta_t-\Theta_t^\circ$,
$\delta_t^\circ:=\dS(\Theta_t^\circ,\Theta_{t+1}^\circ)$ and
$a_t^\circ:=\|\Theta_t^\circ\|_{0,\mathrm{off}}$,
\begin{equation}
\begin{aligned}
\sum_{t=1}^T\|\boldsymbol{\Delta}_t\|_F^2\le\ &\frac4\ell\|(\mathbf{Q}^{-1/2}\otimes\mathcal{I})\mathbf{a}\|^2+\frac{32\lambda^2}{\ell^2}\sum_ta_t^\circ\\
&+\frac{32\mu}{\ell}\sum_t(\delta_t^\circ)^2+\frac{64}{\ell^2}\sum_t\|\mathbf{g}_t\|_F^2.
\end{aligned}
\label{eq:master}
\end{equation}
\end{proposition}
\begin{proof}
Write $\boldsymbol{\Delta}_t=\boldsymbol{\Delta}_{L,t}+\boldsymbol{\Delta}_{D,t}$
with $\boldsymbol{\Delta}_{L,t}:=\hat{\mathbf{L}}_t-\mathbf{L}_t^\circ$ and
$\boldsymbol{\Delta}_{D,t}:=\hat{\mathbf{D}}_t-\mathbf{D}_t^\circ\in\Dsp$, and let
$\mathcal{V}(\boldsymbol{\Delta}):=\sum_{t=1}^{T-1}\|\boldsymbol{\Delta}_{t+1}-\boldsymbol{\Delta}_t\|_F^2$.
Since $\|\boldsymbol{\Delta}_{L,t}\|_F\le\varrho\le(1-\alpha_\star)\sigma_\star/4$,
Lemmas~\ref{lem:curv}--\ref{lem:split} apply to each $t$.

\emph{Step 1 (basic inequality).} $F(\hat\Theta)\le F(\Theta^\circ)$ reads
\begin{equation}
\begin{aligned}
&\textstyle\sum_t[\mathcal{L}_t(\hat\Theta_t)-\mathcal{L}_t(\Theta_t^\circ)]+\lambda\sum_t[h(\hat\Theta_t)-h(\Theta_t^\circ)]\\
&\textstyle\quad+\mu\sum_t[\dS^2(\hat\Theta_t,\hat\Theta_{t+1})-(\delta_t^\circ)^2]\le0.
\end{aligned}
\label{eq:basic}
\end{equation}

\emph{Step 2 (likelihood).} By the RSC inequality and
$\nabla\mathcal{L}_t(\Theta_t^\circ)=\mathbf{g}_t+\frac12\mathbf{Z}_t$,
$\mathcal{L}_t(\hat\Theta_t)-\mathcal{L}_t(\Theta_t^\circ)\ge\langle\mathbf{g}_t,\boldsymbol{\Delta}_t\rangle+\frac12\langle\mathbf{Z}_t,\boldsymbol{\Delta}_t\rangle+\frac\ell2\|\boldsymbol{\Delta}_t\|_F^2$.

\emph{Step 3 (sparsity penalty).} By convexity of $h$ and
Lemma~\ref{lem:logcosh},
$h(\hat\Theta_t)-h(\Theta_t^\circ)\ge\langle\nabla h(\Theta_t^\circ),\boldsymbol{\Delta}_t\rangle\ge-\sqrt{a_t^\circ}\|\boldsymbol{\Delta}_t\|_F$.

\emph{Step 4 (temporal penalty).} All $\hat\Theta_t$ and $\Theta_t^\circ$
have spectra in $[m,M]$. By Proposition~\ref{prop:airm}(e),
$\dS^2(\hat\Theta_t,\hat\Theta_{t+1})\ge c_F\|\hat\Theta_t-\hat\Theta_{t+1}\|_F^2$
and $c_F\|\Theta_t^\circ-\Theta_{t+1}^\circ\|_F^2\le c_F\frac{M^4}{m^2}(\delta_t^\circ)^2=(\delta_t^\circ)^2$.
Using $\|\mathbf{u}+\mathbf{v}\|^2\ge\frac12\|\mathbf{u}\|^2-\|\mathbf{v}\|^2$
with $\mathbf{u}=\boldsymbol{\Delta}_t-\boldsymbol{\Delta}_{t+1}$ and
$\mathbf{v}=\Theta_t^\circ-\Theta_{t+1}^\circ$,
\[
\sum_t[\dS^2(\hat\Theta_t,\hat\Theta_{t+1})-(\delta_t^\circ)^2]\ge\frac{c_F}{2}\mathcal{V}(\boldsymbol{\Delta})-2\sum_t(\delta_t^\circ)^2.
\]

\emph{Step 5 (splitting of the noise).} Since $\boldsymbol{\Delta}_t-\mathbf{N}_t=\PT_{T_{\mathbf{L}_t^\circ}}\boldsymbol{\Delta}_{L,t}+\boldsymbol{\Delta}_{D,t}\in\mathcal{T}_t$,
where $\mathbf{N}_t:=\PT^\perp_{T_{\mathbf{L}_t^\circ}}\boldsymbol{\Delta}_{L,t}$, we
have $\PT_t^\perp\boldsymbol{\Delta}_t=\PT_t^\perp\mathbf{N}_t$ and
\[
\langle\mathbf{Z}_t,\boldsymbol{\Delta}_t\rangle=\langle\PT_t\mathbf{Z}_t,\boldsymbol{\Delta}_t\rangle+\langle\PT_t^\perp\mathbf{Z}_t,\mathbf{N}_t\rangle.
\]
By Lemma~\ref{lem:curv}, $\mathbf{N}_t\succeq\mathbf{0}$ has rank at most $r$, so
$\|\mathbf{N}_t\|_*\le\sqrt r\|\mathbf{N}_t\|_F\le\sqrt r\|\boldsymbol{\Delta}_{L,t}\|_F^2/\sigma_\star$;
moreover $\|\PT_t^\perp\mathbf{Z}_t\|_{\op}\le\|\mathbf{Z}_t\|_{\op}+\|\PT_t\mathbf{Z}_t\|_F\le\omega$.
With Lemma~\ref{lem:split} and~\eqref{eq:omega_cond},
\[
\frac12\sum_t|\langle\PT_t^\perp\mathbf{Z}_t,\mathbf{N}_t\rangle|\le\frac{\omega\sqrt rC_\alpha}{2\sigma_\star}\sum_t\|\boldsymbol{\Delta}_t\|_F^2\le\frac\ell8\sum_t\|\boldsymbol{\Delta}_t\|_F^2.
\]
For the leading part, $\sum_t\langle\PT_t\mathbf{Z}_t,\boldsymbol{\Delta}_t\rangle=\langle(\mathbf{Q}^{-1/2}\otimes\mathcal{I})\mathbf{a},(\mathbf{Q}^{1/2}\otimes\mathcal{I})\boldsymbol{\Delta}\rangle$,
and $\|(\mathbf{Q}^{1/2}\otimes\mathcal{I})\boldsymbol{\Delta}\|^2=\frac\ell2\sum_t\|\boldsymbol{\Delta}_t\|_F^2+\mu c_F\mathcal{V}(\boldsymbol{\Delta})$.
By the Cauchy--Schwarz and Young inequalities,
\begin{align*}
\frac12\Bigl|\sum_t\langle\PT_t\mathbf{Z}_t,\boldsymbol{\Delta}_t\rangle\Bigr|\le\ &\frac14\|(\mathbf{Q}^{-1/2}\otimes\mathcal{I})\mathbf{a}\|^2\\
&+\frac\ell8\sum_t\|\boldsymbol{\Delta}_t\|_F^2+\frac{\mu c_F}{4}\mathcal{V}(\boldsymbol{\Delta}).
\end{align*}

\emph{Step 6 (assembly).} By Young's inequality,
$\lambda\sqrt{a_t^\circ}\|\boldsymbol{\Delta}_t\|_F\le\frac{2\lambda^2a_t^\circ}{\ell}+\frac\ell8\|\boldsymbol{\Delta}_t\|_F^2$
and $\|\mathbf{g}_t\|_F\|\boldsymbol{\Delta}_t\|_F\le\frac4\ell\|\mathbf{g}_t\|_F^2+\frac{\ell}{16}\|\boldsymbol{\Delta}_t\|_F^2$.
Inserting Steps 2--5 into~\eqref{eq:basic},
\begin{align*}
&\frac\ell2\sum_t\|\boldsymbol{\Delta}_t\|_F^2+\frac{\mu c_F}{2}\mathcal{V}(\boldsymbol{\Delta})\\
&\le\frac14\|(\mathbf{Q}^{-1/2}\otimes\mathcal{I})\mathbf{a}\|^2+\frac{7\ell}{16}\sum_t\|\boldsymbol{\Delta}_t\|_F^2+\frac{\mu c_F}{4}\mathcal{V}(\boldsymbol{\Delta})\\
&\quad+\frac{2\lambda^2}{\ell}\sum_ta_t^\circ+2\mu\sum_t(\delta_t^\circ)^2+\frac4\ell\sum_t\|\mathbf{g}_t\|_F^2.
\end{align*}
Dropping the nonnegative term $\frac{\mu c_F}{4}\mathcal{V}(\boldsymbol{\Delta})$ and
multiplying by $16/\ell$ yields~\eqref{eq:master}.
\end{proof}

\subsection{Proofs of Theorem~\ref{thm:estimation}, Corollary~\ref{cor:temporal}, and Theorem~\ref{thm:misspec}}

\begin{proof}[Proof of Theorem~\ref{thm:estimation}]
Take $\Theta^\circ=\Theta^\star$ and
$\mathbf{g}_t:=\E\nabla\mathcal{L}_t(\Theta_t^\star)=\frac12(\E[u_t(Q_t)\mathbf{x}\mathbf{x}^\top]-\Sigma_t^\star)=\mathbf{0}$
by Assumption~\ref{ass:ellip}(b); then $\mathbf{Z}_t$ coincides with the
matrix of Lemma~\ref{lem:conc}. Let $\mathcal{A}$ be the intersection of
the RSC event and the event of Lemma~\ref{lem:conc};
$\mathbb{P}(\mathcal{A})\ge1-\eta-\eta_0$, since~\eqref{eq:sample_size}
implies $n_{\min}\ge2(\kappa_\psi+1)^2pL_\eta$. On $\mathcal{A}$, the
second part of~\eqref{eq:sample_size} gives
$\omega\le\frac{C_\omega}{m}\sqrt{\dr L_\eta/n_{\min}}\le\frac{3(1-\alpha_\star)\ell\sigma_\star}{16\sqrt r}$,
i.e.,~\eqref{eq:omega_cond}. Proposition~\ref{prop:master} and
Lemma~\ref{lem:conc}(a) give
\begin{align*}
\frac1T\sum_t\|\boldsymbol{\Delta}_t\|_F^2\le\ &\frac{4}{\ell T}\cdot9L_\eta\frac{\dr\sigma_W^2}{n_{\min}}\frac{2T}{\ell}\dtime(\tau)\\
&+\frac{32\lambda^2}{\ell^2}\bar a+\frac{32\mu}{\ell}\bar\delta^2,
\end{align*}
which is~\eqref{eq:error_bound}. The affine-invariant bound follows from
Proposition~\ref{prop:airm}(e).

For the interior claim, Lemma~\ref{lem:split} gives, for every $t$,
$\|\boldsymbol{\Delta}_{L,t}\|_F^2+\|\boldsymbol{\Delta}_{D,t}\|_F^2\le C_\alpha\|\boldsymbol{\Delta}_t\|_F^2\le C_\alpha T\bar{\mathcal{E}}<r_0^2$.
Hence $\|\hat{\mathbf{L}}_t-\mathbf{L}_t^\star\|_F<\varrho$ and
$\|\hat{\mathbf{D}}_t-\mathbf{D}_t^\star\|_{\op}\le\|\boldsymbol{\Delta}_{D,t}\|_F<r_0$, so
$d_{\rm lo}\mathbf{I}\prec\hat{\mathbf{D}}_t\prec d_{\rm hi}\mathbf{I}$\nrev{,
and $\|\hat\Theta_t-\Theta_t^\star\|_F^2\le\sum_s\|\boldsymbol{\Delta}_s\|_F^2\le T\bar{\mathcal{E}}<r_0^2\le r_c^2$}: the point lies
in the interior of $\Nb$ (relative to rank-$r$ positive semidefinite
$\mathbf{L}_t$). If $\hat\Theta$ minimizes $F$ over $\Nb$, then $F\ge F(\hat\Theta)$
on an open neighborhood of $(\hat{\mathbf{L}}_t,\hat{\mathbf{D}}_t)_t$; its preimage
under the continuous map
$(\mathbf{Y}_t,\mathbf{D}_t)_t\mapsto(\mathbf{Y}_t\mathbf{Y}_t^\top,\mathbf{D}_t)_t$ is
an open neighborhood of any representative $\hat\theta$, on which
$F\ge F(\hat\theta)$. Thus $\hat\theta$ is a local minimizer of~\eqref{eq:DEPFM}.
\end{proof}

\begin{proof}[Proof of Corollary~\ref{cor:temporal}]
(i) $\dtime(0)=1$. (ii) If $\mu\ge\ell(T^2-1)/(8c_F)$, then
$\tau=2\mu c_F/\ell\ge(T^2-1)/4$, so $(1+4\tau)^{-1/2}\le1/T$ and, by
Lemma~\ref{lem:dtime}, $\dtime(\tau)\le2/T$; the temporal bias vanishes
because $\bar\delta=0$. (iii) By Lemma~\ref{lem:dtime} and
$(1+4\tau)^{-1/2}\le(4\tau)^{-1/2}$, the stochastic-plus-temporal part of
$\bar{\mathcal{E}}$ is at most $A/T+\frac{A}{2\sqrt\tau}+B\tau$, where we used
$\frac{32\mu}{\ell}\bar\delta^2=\frac{16\tau}{c_F}\bar\delta^2=B\tau$. At
$\tau^\star=(A/(4B))^{2/3}$, i.e., $\mu=\mu^\star$, we have
$\frac{A}{2\sqrt{\tau^\star}}=2B\tau^\star$, so the sum equals
$\frac AT+3B\tau^\star=\frac AT+3\cdot4^{-2/3}A^{2/3}B^{1/3}\le\frac AT+1.2A^{2/3}B^{1/3}$.
If $B\le A/8$ and $T\ge4$, this is at most $0.25A+1.2\cdot\frac12A=0.85A$.
\end{proof}

\begin{proof}[Proof of Theorem~\ref{thm:misspec}]
Apply Proposition~\ref{prop:master} with comparator $\Theta^\circ$ and
$\mathbf{g}_t=\mathbf{g}_t^\circ$, on the intersection of $\Omega^\circ$ and the
RSC event. The random vectors $\PT_t^\circ\mathbf{Z}_t^\circ$ are independent
across $t$ and zero mean, so cross terms vanish in expectation and
\begin{align*}
\E\|(\mathbf{Q}^{-1/2}\otimes\mathcal{I})\mathbf{a}^\circ\|^2&=\sum_t[\mathbf{Q}^{-1}]_{tt}\E\|\PT_t^\circ\mathbf{Z}_t^\circ\|_F^2\\
&\le\frac{\dr\sigma_\circ^2}{n_{\min}}\tr(\mathbf{Q}^{-1})=\frac{2T\dr\sigma_\circ^2}{\ell n_{\min}}\dtime(\tau),
\end{align*}
where $\E\|\PT_t^\circ\mathbf{Z}_t^\circ\|_F^2\le\dr\sigma_\circ^2/n_t$ follows
by expanding in an orthonormal basis of $\mathcal{T}_t^\circ$ as in
Lemma~\ref{lem:moments}(b). By Markov's inequality, with probability at
least $1-\eta$,
$\frac{4}{\ell T}\|(\mathbf{Q}^{-1/2}\otimes\mathcal{I})\mathbf{a}^\circ\|^2\le\frac{8\sigma_\circ^2\dr}{\eta\ell^2n_{\min}}\dtime(\tau)$.
Then~\eqref{eq:master} bounds
$\frac1T\sum_t\|\hat\Theta_t-\Theta_t^\circ\|_F^2$, and
$\|\hat\Theta_t-\Theta_t^\star\|_F^2\le2\|\hat\Theta_t-\Theta_t^\circ\|_F^2+2\|\Theta_t^\circ-\Theta_t^\star\|_F^2$
yields~\eqref{eq:misspec}.
\end{proof}

\section{Proof of Proposition~\ref{prop:t_rsc}}
\label{app:t_rsc}
\begin{nrevised}
Fix $t$ and drop it from the notation: $\Sigma:=\Sigma_t^\star$,
$\Theta^\star:=\Theta_t^\star$, $n:=n_t$, and $c:=(\nu+p)/2$, so that
$\rho(s)=c\log(1+s/\nu)$, $\rho'(s)=c/(\nu+s)$, $\rho''(s)=-c/(\nu+s)^2$,
$u(s)=2\rho'(s)$, and $\psi(s)=su(s)$. Write
$\mathbf{x}_i=R_i\Sigma^{1/2}\mathbf{s}_i$ as in
Section~\ref{sec:background}, so that
$Q_i=\mathbf{x}_i^\top\Theta^\star\mathbf{x}_i=R_i^2$, and set
\[
B_i:=\frac{R_i^2}{\nu+R_i^2},\qquad \psi_i:=\psi(Q_i)=(\nu+p)B_i=2cB_i .
\]
For $t_\nu$ data, $B_i\sim\mathrm{Beta}(p/2,\nu/2)$ is independent of
$\mathbf{s}_i$ (Example~\ref{ex:t}), so that
\begin{equation}
\begin{aligned}
&\E B=\tfrac{p}{p+\nu},\qquad \E B^2=\tfrac{p(p+2)}{(p+\nu)(p+\nu+2)},\\
&\E[B(1-B)]=\tfrac{p\nu}{(p+\nu)(p+\nu+2)}=\tfrac{p\,\kappa_\nu}{p+\nu}.
\end{aligned}
\label{eq:beta_moments}
\end{equation}
For $\mathbf{s}$ uniform on the sphere and $\mathbf{A}\in\mathcal{S}^p$, we
also use $\E\mathbf{s}\mathbf{s}^\top=\mathbf{I}/p$ and
$\E(\mathbf{s}^\top\mathbf{A}\mathbf{s})^2=\frac{(\tr\mathbf{A})^2+2\|\mathbf{A}\|_F^2}{p(p+2)}$
(proof of Lemma~\ref{lem:moments}).

\begin{proof}
\emph{(i) Population curvature.} For $\Theta\succ\mathbf{0}$ and
$\boldsymbol{\Delta}\in\mathcal{S}^p$,
\begin{equation}
\begin{aligned}
D^2\mathcal{L}(\Theta)[\boldsymbol{\Delta},\boldsymbol{\Delta}]
=\ &\tfrac12\|\Theta^{-1/2}\boldsymbol{\Delta}\Theta^{-1/2}\|_F^2\\
&+\tfrac1n\textstyle\sum_i\rho''(\mathbf{x}_i^\top\Theta\mathbf{x}_i)(\mathbf{x}_i^\top\boldsymbol{\Delta}\mathbf{x}_i)^2 .
\end{aligned}
\label{eq:t_hess}
\end{equation}
The integrands of $\mathcal{L}$ and of its first two derivatives are
dominated, locally uniformly in $\Theta$, by integrable functions:
$0\le\rho(\mathbf{x}^\top\Theta\mathbf{x})\le c\log(1+\|\Theta\|_{\op}\|\mathbf{x}\|^2/\nu)$,
which is integrable because the $t_\nu$ law has finite logarithmic
moments, while
$|\rho'(\mathbf{x}^\top\Theta\mathbf{x})\,\mathbf{x}^\top\boldsymbol{\Delta}\mathbf{x}|\le c\|\boldsymbol{\Delta}\|_{\op}/\lambda_{\min}(\Theta)$
and
$|\rho''(\mathbf{x}^\top\Theta\mathbf{x})|(\mathbf{x}^\top\boldsymbol{\Delta}\mathbf{x})^2\le c\|\boldsymbol{\Delta}\|_{\op}^2/\lambda_{\min}(\Theta)^2$
are bounded. By dominated convergence, $\bar{\mathcal{L}}$ is $C^2$ on
$\Spp$ and $D^2\bar{\mathcal{L}}=\E D^2\mathcal{L}$. At $\Theta^\star$,
$\Theta^{\star-1/2}=\Sigma^{1/2}$ and
$\mathbf{x}^\top\boldsymbol{\Delta}\mathbf{x}=R^2\mathbf{s}^\top\mathbf{A}\mathbf{s}$,
so $\rho''(Q)(\mathbf{x}^\top\boldsymbol{\Delta}\mathbf{x})^2=-cB^2(\mathbf{s}^\top\mathbf{A}\mathbf{s})^2$.
By independence and~\eqref{eq:beta_moments},
\begin{align*}
D^2\bar{\mathcal{L}}(\Theta^\star)[\boldsymbol{\Delta},\boldsymbol{\Delta}]
&=\tfrac12\|\mathbf{A}\|_F^2-c\,\E B^2\,\E(\mathbf{s}^\top\mathbf{A}\mathbf{s})^2\\
&=\tfrac12\|\mathbf{A}\|_F^2-\frac{(\tr\mathbf{A})^2+2\|\mathbf{A}\|_F^2}{2(p+\nu+2)}\\
&=\frac{(p+\nu)\|\mathbf{A}\|_F^2-(\tr\mathbf{A})^2}{2(p+\nu+2)}.
\end{align*}
By the Cauchy--Schwarz inequality,
$(\tr\mathbf{A})^2=\langle\mathbf{A},\mathbf{I}\rangle^2\le p\|\mathbf{A}\|_F^2$,
with equality if and only if $\mathbf{A}\in\operatorname{span}\{\mathbf{I}\}$,
i.e., $\boldsymbol{\Delta}\in\operatorname{span}\{\Theta^\star\}$; this gives
the first inequality. The second follows from
$\|\boldsymbol{\Delta}\|_F=\|\Theta^{\star1/2}\mathbf{A}\Theta^{\star1/2}\|_F\le M\|\mathbf{A}\|_F$.

\emph{(ii) Sharpness.} For $|s|$ small, $(1+s)\Theta^\star=(1+s)\mathbf{L}^\star+(1+s)\mathbf{D}^\star$
lies in the $t$-th slice of $\Nb$: $(1+s)\mathbf{L}^\star\succeq\mathbf{0}$
has rank $r$, $\|s\mathbf{L}^\star\|_F\le\varrho$,
$\|s\Theta^\star\|_F\le r_c$, and $(1+s)\mathbf{D}^\star$ stays strictly
inside $[d_{\rm lo},d_{\rm hi}]$ by assumption. Let
$\bar b(s):=\bar{\mathcal{L}}((1+s)\Theta^\star)-\bar{\mathcal{L}}(\Theta^\star)-s\langle\nabla\bar{\mathcal{L}}(\Theta^\star),\Theta^\star\rangle$.
By Taylor's formula and the continuity of $D^2\bar{\mathcal{L}}$,
\[
\frac{\bar b(s)}{s^2}=\int_0^1(1-\tau)D^2\bar{\mathcal{L}}((1+\tau s)\Theta^\star)[\Theta^\star,\Theta^\star]\,d\tau
\ \xrightarrow[s\to0]{}\ \frac{p\kappa_\nu}{4},
\]
since (i) with $\mathbf{A}=\mathbf{I}$ gives
$D^2\bar{\mathcal{L}}(\Theta^\star)[\Theta^\star,\Theta^\star]=\frac{(p+\nu)p-p^2}{2(p+\nu+2)}=\frac{p\kappa_\nu}{2}$.
The assumed inequality $\bar b(s)\ge\frac{\ell}{2}s^2\|\Theta^\star\|_F^2$
yields, as $s\to0$, $\ell\|\Theta^\star\|_F^2\le p\kappa_\nu/2$; finally,
$\|\Theta^\star\|_F^2\ge p\,m^2$.

\emph{(iii) Empirical local RSC.}

\emph{Step 1 (exact decomposition of the Bregman divergence).} Let
$\|\Theta-\Theta^\star\|_F\le m/5$, $\boldsymbol{\Delta}:=\Theta-\Theta^\star$,
and $\mathbf{A}:=\Theta^{\star-1/2}\boldsymbol{\Delta}\Theta^{\star-1/2}$.
Then $\|\mathbf{A}\|_{\op}\le\|\boldsymbol{\Delta}\|_{\op}/m\le1/5$, so the
spectrum of $\mathbf{I}+\mathbf{A}$ lies in $[4/5,6/5]$, and
$\mathbf{H}:=\log(\mathbf{I}+\mathbf{A})$ satisfies
$\|\mathbf{H}\|_{\op}\le h:=1/4$ (since $|\log(4/5)|<1/4$ and $\log(6/5)<1/4$),
$\Theta=\Theta^{\star1/2}e^{\mathbf{H}}\Theta^{\star1/2}$, and
$\boldsymbol{\Delta}=\Theta^{\star1/2}(e^{\mathbf{H}}-\mathbf{I})\Theta^{\star1/2}$.
Note that $\Theta^{\star1/2}\mathbf{H}\Theta^{\star1/2}=\Log_{\Theta^\star}\Theta$
and $\|\mathbf{H}\|_F=\dS(\Theta^\star,\Theta)$. Since $|e^x-1|\le|x|e^{|x|}$,
\begin{equation}
  \|\boldsymbol{\Delta}\|_F\le M\|e^{\mathbf{H}}-\mathbf{I}\|_F\le Me^{h}\|\mathbf{H}\|_F .
  \label{eq:Delta_H}
\end{equation}
Let $\sigma_i:=\mathbf{s}_i^\top e^{\mathbf{H}}\mathbf{s}_i$. Then
$\mathbf{x}_i^\top\Theta\mathbf{x}_i=R_i^2\sigma_i$,
$\mathbf{x}_i^\top\boldsymbol{\Delta}\mathbf{x}_i=R_i^2(\sigma_i-1)$,
$\log\det\Theta=\log\det\Theta^\star+\tr\mathbf{H}$,
$\langle\Sigma,\boldsymbol{\Delta}\rangle=\tr(e^{\mathbf{H}}-\mathbf{I})$, and
$(\nu+R_i^2\sigma_i)/(\nu+R_i^2)=1-B_i+B_i\sigma_i$. Using
$\nabla\mathcal{L}(\Theta^\star)=-\frac12\Sigma+\frac1{2n}\sum_iu(Q_i)\mathbf{x}_i\mathbf{x}_i^\top$
and $u(Q_i)Q_i=\psi_i=2cB_i$, the Bregman divergence
$b(\Theta):=\mathcal{L}(\Theta)-\mathcal{L}(\Theta^\star)-\langle\nabla\mathcal{L}(\Theta^\star),\boldsymbol{\Delta}\rangle$
equals
\[
\begin{aligned}
b(\Theta)=\ &\tfrac12\tr(e^{\mathbf{H}}-\mathbf{I}-\mathbf{H})\\
&+\tfrac{c}{n}\textstyle\sum_i\bigl[\log(1-B_i+B_i\sigma_i)-B_i(\sigma_i-1)\bigr].
\end{aligned}
\]
Writing $\sigma_i-1=\mathbf{s}_i^\top\mathbf{H}\mathbf{s}_i+\mathbf{s}_i^\top\mathbf{X}\mathbf{s}_i$
with $\mathbf{X}:=e^{\mathbf{H}}-\mathbf{I}-\mathbf{H}$, we obtain the identity
\begin{equation}
\begin{aligned}
&b(\Theta)=\mathcal{G}(\mathbf{H})-\tfrac12\langle\widetilde{\mathbf{Z}},\mathbf{X}\rangle,\\
&\mathcal{G}(\mathbf{H}):=\tfrac{c}{n}\textstyle\sum_i\bigl[\log(1-B_i+B_i\sigma_i)-B_i\,\mathbf{s}_i^\top\mathbf{H}\mathbf{s}_i\bigr],\\
&\widetilde{\mathbf{Z}}:=\tfrac1n\textstyle\sum_i\psi_i\,\mathbf{s}_i\mathbf{s}_i^\top-\mathbf{I}
=\Theta^{\star1/2}\bigl(2\nabla\mathcal{L}(\Theta^\star)\bigr)\Theta^{\star1/2}.
\end{aligned}
\label{eq:breg_identity}
\end{equation}
Here $\mathcal{G}(\mathbf{H})=\mathcal{L}(\Theta)-\mathcal{L}(\Theta^\star)-\langle\nabla\mathcal{L}(\Theta^\star),\Log_{\Theta^\star}\Theta\rangle$
is the Bregman divergence along the affine-invariant geodesic
$\tau\mapsto\Theta^{\star1/2}e^{\tau\mathbf{H}}\Theta^{\star1/2}$, and the
second term accounts for the discrepancy
$\boldsymbol{\Delta}-\Log_{\Theta^\star}\Theta=\Theta^{\star1/2}\mathbf{X}\Theta^{\star1/2}$
between the chord and the geodesic, paired with the whitened score
$\widetilde{\mathbf{Z}}$.

\emph{Step 2 (geodesic curvature).} For fixed $i$, let
$\sigma_i(\tau):=\mathbf{s}_i^\top e^{\tau\mathbf{H}}\mathbf{s}_i$,
$d_i(\tau):=1-B_i+B_i\sigma_i(\tau)$, and
$\varphi_i(\tau):=\log d_i(\tau)-\tau B_i\mathbf{s}_i^\top\mathbf{H}\mathbf{s}_i$,
$\tau\in[0,1]$, so that $\mathcal{G}(\mathbf{H})=\frac{c}{n}\sum_i\varphi_i(1)$.
Since $\sigma_i(0)=1$ and $\sigma_i'(0)=\mathbf{s}_i^\top\mathbf{H}\mathbf{s}_i$,
we have $\varphi_i(0)=\varphi_i'(0)=0$, and
\[
\begin{aligned}
\varphi_i''&=\frac{B_i(1-B_i)\sigma_i''+B_i^2\bigl(\sigma_i\sigma_i''-(\sigma_i')^2\bigr)}{d_i^2}\\
&\ge\frac{B_i(1-B_i)\,\sigma_i''}{d_i^2},
\end{aligned}
\]
because, by the Cauchy--Schwarz inequality applied to
$e^{\tau\mathbf{H}/2}\mathbf{s}_i$ and $\mathbf{H}e^{\tau\mathbf{H}/2}\mathbf{s}_i$,
$(\sigma_i')^2=(\mathbf{s}_i^\top\mathbf{H}e^{\tau\mathbf{H}}\mathbf{s}_i)^2\le\sigma_i\,\mathbf{s}_i^\top\mathbf{H}^2e^{\tau\mathbf{H}}\mathbf{s}_i=\sigma_i\sigma_i''$.
(This termwise convexity is the geodesic convexity of the Student-$t$
likelihood~\cite{wiesel2012geodesic}.) Since $\mathbf{H}$ and
$e^{\tau\mathbf{H}}$ commute and $\|\mathbf{H}\|_{\op}\le h$,
$\sigma_i''(\tau)=\mathbf{s}_i^\top\mathbf{H}^2e^{\tau\mathbf{H}}\mathbf{s}_i\ge e^{-h}\mathbf{s}_i^\top\mathbf{H}^2\mathbf{s}_i$
and $d_i(\tau)\le\max\{1,\sigma_i(\tau)\}\le e^{h}$ for $\tau\in[0,1]$.
Hence $\varphi_i''(\tau)\ge e^{-3h}B_i(1-B_i)\,\mathbf{s}_i^\top\mathbf{H}^2\mathbf{s}_i$
and, by Taylor's formula
$\varphi_i(1)=\int_0^1(1-\tau)\varphi_i''(\tau)\,d\tau$,
\begin{equation}
\begin{aligned}
&\mathcal{G}(\mathbf{H})\ge\frac{c\,e^{-3h}}{2}\langle\widehat{\mathbf{S}},\mathbf{H}^2\rangle
\ge\frac{c\,e^{-3h}}{2}\lambda_{\min}(\widehat{\mathbf{S}})\|\mathbf{H}\|_F^2,\\
&\widehat{\mathbf{S}}:=\tfrac1n\textstyle\sum_iB_i(1-B_i)\,\mathbf{s}_i\mathbf{s}_i^\top .
\end{aligned}
\label{eq:G_lower}
\end{equation}

\emph{Step 3 (chord--geodesic discrepancy).} Since
$0\le e^x-1-x\le\frac{x^2}{2}e^{|x|}$ for all $x\in\R$, $\mathbf{X}\succeq\mathbf{0}$
and $\tr\mathbf{X}\le\frac{e^{h}}{2}\|\mathbf{H}\|_F^2$. Hence
\begin{equation}
  \bigl|\tfrac12\langle\widetilde{\mathbf{Z}},\mathbf{X}\rangle\bigr|\le\tfrac12\|\widetilde{\mathbf{Z}}\|_{\op}\tr\mathbf{X}\le\tfrac{e^{h}}{4}\|\widetilde{\mathbf{Z}}\|_{\op}\|\mathbf{H}\|_F^2 .
  \label{eq:disc_bound}
\end{equation}

\emph{Step 4 (concentration).} Let $L':=\log(4pT/\eta)$; restore the
index $t$. We show that, with probability at least $1-\eta$, for all $t$,
\begin{equation}
  \lambda_{\min}(\widehat{\mathbf{S}}_t)\ge\frac{\kappa_\nu}{2(p+\nu)}
  \quad\text{and}\quad
  \|\widetilde{\mathbf{Z}}_t\|_{\op}\le\frac{\kappa_\nu}{4e}.
  \label{eq:t_events}
\end{equation}
Condition~\eqref{eq:t_sample} reads $n_t\ge950(p+\nu+2)L'/\kappa_\nu^2$, and
since $\kappa_\nu\le1$ and $128e^2<950$ it implies
\begin{equation}
  n_t\ \ge\ \max\Bigl\{\frac{2(p+\nu)L'}{\kappa_\nu},\ \frac{128e^2(p+1)L'}{\kappa_\nu^2},\ \frac{16e(p+\nu)L'}{3\kappa_\nu}\Bigr\}.
  \label{eq:t_three}
\end{equation}
(a) The summands $B_i(1-B_i)\mathbf{s}_i\mathbf{s}_i^\top/n_t$ of
$\widehat{\mathbf{S}}_t$ are independent and positive semidefinite, with
$\lambda_{\max}\le1/(4n_t)$ because $B(1-B)\le1/4$. By~\eqref{eq:beta_moments}
and independence, $\E\widehat{\mathbf{S}}_t=\frac{\kappa_\nu}{p+\nu}\mathbf{I}$.
The matrix Chernoff inequality~\cite[Thm.~1.1]{tropp2012} with $\delta=1/2$ gives
\[
\begin{aligned}
\mathbb{P}\Bigl(\lambda_{\min}(\widehat{\mathbf{S}}_t)\le\frac{\kappa_\nu}{2(p+\nu)}\Bigr)
&\le p\Bigl(\frac{2}{e}\Bigr)^{\frac{2n_t\kappa_\nu}{p+\nu}}\\
&\le p\,e^{-\frac{n_t\kappa_\nu}{2(p+\nu)}}\le\frac{\eta}{2T},
\end{aligned}
\]
where we used $2(1-\log2)\ge1/2$ and the first bound in~\eqref{eq:t_three}.

(b) The summands $\mathbf{X}_i:=(\psi_i\mathbf{s}_i\mathbf{s}_i^\top-\mathbf{I})/n_t$
of $\widetilde{\mathbf{Z}}_t$ are independent and have zero mean, since
$\E\psi=2c\E B=p$ and $\E\mathbf{s}\mathbf{s}^\top=\mathbf{I}/p$. As
$\psi_i\in[0,p+\nu)$, $\|\mathbf{X}_i\|_{\op}\le(p+\nu)/n_t$. Moreover,
$(\mathbf{s}\mathbf{s}^\top)^2=\mathbf{s}\mathbf{s}^\top$ gives
$\E(\psi\mathbf{s}\mathbf{s}^\top-\mathbf{I})^2=(\E\psi^2/p-1)\mathbf{I}$, and
$\E\psi^2=p^2+\operatorname{Var}\psi=p^2+2p\kappa_\nu$ (Example~\ref{ex:t}),
so $\|\sum_i\E\mathbf{X}_i^2\|_{\op}\le(p-1+2\kappa_\nu)/n_t\le(p+1)/n_t$.
The matrix Bernstein inequality, applied as in the proof of
Lemma~\ref{lem:conc}(c), gives, with probability at least $1-\eta/(2T)$,
\[
\|\widetilde{\mathbf{Z}}_t\|_{\op}\le\sqrt{\frac{2(p+1)L'}{n_t}}+\frac{2(p+\nu)L'}{3n_t}
\le\frac{\kappa_\nu}{8e}+\frac{\kappa_\nu}{8e},
\]
by the second and third bounds in~\eqref{eq:t_three}. A union bound over
$t$ and over (a)--(b) yields~\eqref{eq:t_events}.

\emph{Step 5 (conclusion).} On~\eqref{eq:t_events},
$c\,\lambda_{\min}(\widehat{\mathbf{S}}_t)\ge\frac{\nu+p}{2}\cdot\frac{\kappa_\nu}{2(p+\nu)}=\frac{\kappa_\nu}{4}$.
Combining~\eqref{eq:breg_identity}, \eqref{eq:G_lower}, and~\eqref{eq:disc_bound}
with $h=1/4$,
\[
b(\Theta)\ge\Bigl(\frac{e^{-3/4}\kappa_\nu}{8}-\frac{e^{1/4}}{4}\cdot\frac{\kappa_\nu}{4e}\Bigr)\|\mathbf{H}\|_F^2
=\frac{e^{-3/4}\kappa_\nu}{16}\|\mathbf{H}\|_F^2 .
\]
By~\eqref{eq:Delta_H}, $\|\mathbf{H}\|_F^2\ge e^{-1/2}\|\boldsymbol{\Delta}\|_F^2/M^2$, and therefore
$b(\Theta)\ge\frac{e^{-5/4}\kappa_\nu}{16M^2}\|\boldsymbol{\Delta}\|_F^2\ge\frac{\kappa_\nu}{56M^2}\|\boldsymbol{\Delta}\|_F^2=\frac{\ell_\nu}{2}\|\boldsymbol{\Delta}\|_F^2$,
since $e^{-5/4}>2/7$. This holds simultaneously for all $t$ and all
$\Theta$ in the ball, because~\eqref{eq:t_events} does not depend on
$\Theta$.
\end{proof}
\end{nrevised}

\section{Proof of Theorem~\ref{thm:support}}
\label{app6}
\begin{lemma}[Lipschitz continuity of partial correlations]
\label{lem:pcorr_lip}
If $m\mathbf{I}\preceq\Theta,\Theta'\preceq M\mathbf{I}$, then for all $i\neq j$,
$|\varrho_{ij}(\Theta)-\varrho_{ij}(\Theta')|\le\frac{2M}{m^2}\|\Theta-\Theta'\|_{\max}$.
\end{lemma}
\begin{proof}
Write $\varrho_{ij}=-c\,(ab)^{-1/2}$ with $a=\Theta_{ii}$, $b=\Theta_{jj}$,
$c=\Theta_{ij}$, and primes for $\Theta'$. Then $a,b,a',b'\ge m$ and
$|c|\le\sqrt{ab}\le M$. Since
$|\partial_a(ab)^{-1/2}|=\frac12a^{-3/2}b^{-1/2}\le\frac1{2m^2}$ (and likewise
in $b$) on the segment joining $(a,b)$ and $(a',b')$,
\begin{align*}
|\varrho_{ij}-\varrho_{ij}'|&\le\frac{|c-c'|}{\sqrt{a'b'}}+|c|\,\bigl|(ab)^{-1/2}-(a'b')^{-1/2}\bigr|\\
&\le\frac{|c-c'|}{m}+\frac{M(|a-a'|+|b-b'|)}{2m^2},
\end{align*}
which is at most $(\frac1m+\frac M{m^2})\|\Theta-\Theta'\|_{\max}\le\frac{2M}{m^2}\|\Theta-\Theta'\|_{\max}$.
\end{proof}

\begin{proof}[Proof of Theorem~\ref{thm:support}]
On the event of Theorem~\ref{thm:estimation}, for every $t$,
$\|\hat\Theta_t-\Theta_t^\star\|_{\max}\le\|\hat\Theta_t-\Theta_t^\star\|_F\le\bigl(\sum_s\|\hat\Theta_s-\Theta_s^\star\|_F^2\bigr)^{1/2}\le\sqrt{T\bar{\mathcal{E}}}$.
Both $\hat\Theta_t$ and $\Theta_t^\star$ have spectra in $[m,M]$, so by
Lemma~\ref{lem:pcorr_lip} and~\eqref{eq:support_cond},
$|\hat\varrho_{t,ij}-\varrho_{t,ij}^\star|\le\frac{2M}{m^2}\sqrt{T\bar{\mathcal{E}}}<\frac{\tau_{\min}}{2}$
for all $t$ and $i\neq j$. If $(i,j)\in E_t^\star$, then
$|\hat\varrho_{t,ij}|>\tau_{\min}-\tau_{\min}/2=\tau_{\min}/2$, so the edge is
retained; if $(i,j)\notin E_t^\star$, then $\varrho_{t,ij}^\star=0$ and
$|\hat\varrho_{t,ij}|<\tau_{\min}/2$, so the edge is discarded. Hence
$\hat E_t=E_t^\star$ for all $t$.
\end{proof}
\end{revised}

\bibliographystyle{unsrt}
\bibliography{References}

\end{document}